\documentclass[lettersize,journal]{IEEEtran}
\usepackage{amsmath,amsfonts}
\IEEEoverridecommandlockouts
\normalsize

\usepackage{cite}
\usepackage{url}
\usepackage{hyperref}
\hypersetup{colorlinks=true, linkcolor=red, citecolor=purple, urlcolor=blue} 
\usepackage{amssymb}
\usepackage{amsmath}
\usepackage{array}
\usepackage{amsthm}
\allowdisplaybreaks 
\usepackage{autobreak}
\usepackage{mathptmx}
\usepackage{mathtools, cuted}
\usepackage{algpseudocode}
\usepackage[ruled, vlined, linesnumbered]{algorithm2e}
\usepackage{graphicx}
\usepackage{subcaption} 
\usepackage{bbm}
\usepackage{booktabs}
\usepackage{multirow}
\usepackage{makecell}
\usepackage[table,xcdraw,dvipsnames]{xcolor}
\usepackage{colortbl}
\usepackage[most]{tcolorbox}

\newtcolorbox{theorem-box}[1][]{%
    colback=blue!10!white, 
    colframe=blue!60!black, 
    fonttitle=\bfseries\color{yellow}, 
    title=#1, 
}

\allowdisplaybreaks

\newtheorem{Theorem}{Theorem}
\newtheorem{Lemma}{Lemma}

\newtheorem{Assumption}{Assumption}
\newtheorem{Remark}{Remark}

\newcommand{\rs}{\!\!}
\newcolumntype{C}[1]{>{\centering \arraybackslash}p{#1}}
\newcommand{\bblue}{\textcolor{blue}}

\newcommand{\bviol}{\textcolor{violet}}

\newcommand{\bquad}{\qquad\qquad\qquad\qquad\qquad}\newcommand{\mquad}{\qquad\qquad\qquad}
\newcommand{\squad}{\qquad\qquad}

 \newcommand{\sqz}{{\tt SqueezeNet1}}

\usepackage{acronym}
\acrodef{ml}[ML]{machine learning}
\acrodef{fl}[FL]{federated learning}
\acrodef{hfl}[HFL]{hierarchical federated learning}
\acrodef{rl}[RL]{reinforcement learning}
\acrodef{drl}[DRL]{deep reinforcement learning}
\acrodef{cs}[CS]{central server}

\acrodef{bs}[BS]{base station}
\acrodef{isp}[ISP]{(wireless) internet service provider}
\acrodef{ue}[UE]{user equipment}
\acrodef{es}[ES]{edge server}
\acrodef{csp}[CSP]{content service provider}
\acrodef{fedavg}[{\tt FedAvg}]{federated averaging}
\acrodef{fednova}[{\tt FedNova}]{federated normalized averaging}
\acrodef{afa}[{\tt AFA}]{anarchic federated averaging}
\acrodef{afacd}[{\tt AFA-CD}]{AFA-cross-device}
\acrodef{feddisco}[{\tt FedDisco}]{federated learning with discrepancy-aware collaboration}
\acrodef{scaffold}[{\tt SCAFFOLD}]{stochastic controlled averaging algorithm}
\acrodef{osafl}[{\tt OSAFL}]{\underline{\textbf{o}}nline-\underline{\textbf{s}}core-\underline{\textbf{a}}ided \underline{\textbf{f}}ederated \underline{\textbf{l}}earning}
\acrodef{iot}[IoT]{Internet of Things}
\acrodef{os}[OS]{operating system}
\acrodef{iid}[IID]{independent and identically distributed}
\acrodef{sca}[SCA]{successive convex approximation}

\acrodef{sgd}[{\tt SGD}]{stochastic gradient descent}
\acrodef{cpu}[CPU]{central processing unit}
\acrodef{gpu}[GPU]{graphics processing unit}
\acrodef{prb}[pRB]{physical resource block}
\acrodef{snr}[SNR]{signal-to-noise-ratio}
\acrodef{lp}[LP]{linear programming}
\acrodef{fpp}[FPP]{floating point precision}
\acrodef{cdf}[CDF]{cumulative distribution function}
\acrodef{uav}[UAV]{unmanned aerial vehicles}
\acrodef{ap}[AP]{access point}
\acrodef{hz}[Hz]{hertz}
\acrodef{csi}[CSI]{channel state information}
\acrodef{kkt}[KKT]{Karush–Kuhn–Tucker}
\acrodef{fifo}[{\tt FIFO}]{first-in-first-out}
\acrodef{trimtoplabel}[{\tt TrimTopLabel}]{trim top label}
\acrodef{fcn}[{\tt FCN}]{fully connected neural network}
\acrodef{lstm}[{\tt LSTM}]{long short-term memory}
\acrodef{cnn}[{\tt CNN}]{convolutional neural network}
\acrodef{fc}[{\tt FC}]{fully connected}
\acrodef{resnet18}[{\tt ResNet-$18$}]{residual network}
\acrodef{isac}[ISAC]{integrated sensing and communication}
\acrodef{mfedavg}[{\tt M-FedAvg}]{modified-FedAvg}
\acrodef{mfedprox}[{\tt M-FedProx}]{modified-FedProx}
\acrodef{mfednova}[{\tt M-FedNova}]{modified-FedNova}
\acrodef{mafacd}[{\tt M-AFA-CD}]{modified-AFA-CD}
\acrodef{mfeddisco}[{\tt M-FedDisco}]{modified-FedDisco}
\acrodef{cfl}[CFL]{continual federated learning}
\acrodef{mpc}[MPC]{multi-path component}
\acrodef{ccdf}[CCDF]{complementary cumulative distribution function}

\title{Online-Score-Aided Federated Learning for Resource-Constrained Wireless Clients \\with Continual Data Arrival} 
\author{Ferdous Pervej, \IEEEmembership{Member, IEEE}, Minseok Choi, \IEEEmembership{Member, IEEE}, and Andreas F. Molisch, \IEEEmembership{Fellow, IEEE}
\thanks{This work was supported by NSF-IITP Project $2152646$.}
\thanks{Ferdous Pervej is with the Department of Electrical and Computer Engineering, Utah State University, Logan, UT 84322 USA, and also with the Ming Hsieh Department of Electrical and Computer Engineering, University of Southern California, Los Angeles, CA 90089 USA (e-mail: ferdous.pervej@usu.edu).}
\thanks{Andreas F. Molisch is with the Ming Hsieh Department of Electrical and Computer Engineering, University of Southern California, Los Angeles, CA 90089 USA (e-mail: molisch@usc.edu).}
\thanks{Minseok Choi is with the Department of Electronic Engineering, Kyung Hee University, Republic of South Korea (email: choims@khu.ac.kr)}
\vspace{-0.3in}
}

\begin{document}

\maketitle
\IEEEpeerreviewmaketitle

\begin{abstract}
Heterogeneous system configurations of distributed clients connected to the \ac{cs} via a time-varying wireless network pose significant challenges for popular distributed \ac{ml} algorithms such as \ac{fl}. 
Although the limited (radio and computational) resources are widely acknowledged, two critical yet often ignored aspects are (a) client devices can only dedicate a small chunk of their limited storage for the \ac{fl} task and (b) new training samples may arrive continually in many practical wireless applications.
Therefore, we propose a new \ac{fl} algorithm, \ac{osafl}, specifically designed for tasks with continual data arrival in resource-constrained environments.  
We first theoretically show how the convergence bound is affected by continual data distribution shifts, uncertain client participation, gradient quantization errors, and noise from stochastic gradients and statistical data heterogeneity across clients.  
We then show how to (sub-optimally) minimize these errors by choosing appropriate aggregation weights at the \ac{cs} during global update. 
Our extensive simulation results across three popular image classification datasets and three \ac{ml} models with different numbers of trainable parameters validate the effectiveness of the proposed \ac{osafl} algorithm compared to (modified) state-of-the-art \ac{fl} baselines.
\end{abstract}


\vspace{-0.15in}

\section{Introduction}

\acresetall 
\acused{fl}

\noindent 
\IEEEPARstart{F}{ederated} learning (FL) \cite{mcmahan17Communication} is a privacy-preserving distributed \ac{ml} approach that enables training an \ac{ml} model, parameterized by trainable parameters $\mathbf{w} \in \mathbb{R}^N$, where $N$ is the number of parameters, across distributed clients' local datasets.
A \ac{cs} typically broadcasts the model to a set of clients $\mathcal{U}=\left\{u\right\}_{u=0}^{U-1}$, who then train the received model on their respective local dataset $\mathcal{D}_u$ to minimize 
\begin{align}
\label{localObj_Static}
    f_u(\mathbf{w}|\mathcal{D}_u) \coloneqq \left( 1 / |\mathcal{D}_u| \right) \sum\nolimits_{(\mathbf{x}, y) \in \mathcal{D}_u} l(\mathbf{w}|(\mathbf{x}, y)), 
\end{align}
where $l(\mathbf{w}|(\mathbf{x},y)$ is the loss associated to training sample $(\mathbf{x},y)$. 
Moreover, $\mathbf{x}$ and $y$ are the training feature and label, respectively. 
As such, \ac{fl} is designed to train $\mathbf{w}$ distributively to optimize \cite{mcmahan17Communication}
\begin{align}
\label{globalObj_Static}
    \underset{\mathbf{w}}{\mathrm{minimize}} ~~ f(\mathbf{w}|\mathcal{D}) \coloneqq \sum\nolimits_{u=0}^{U-1} \alpha_u f_u(\mathbf{w}|\mathcal{D}_u),
\end{align}
where $0 \leq \alpha_u \leq 1$ with $\sum_{u=0}^{U-1}\alpha_u=1$ and $\mathcal{D} = \bigcup_{u=0}^{U-1} \mathcal{D}_u$.
Since a client's loss function $f_u(\mathbf{w}|\mathcal{D}_u)$ depends on its local dataset $\mathcal{D}_u$ and the global loss function $f(\mathbf{w}|\mathcal{D})$ is a weighted combination of the clients' loss functions, clients' datasets directly impact the global loss function: non-IID (independent and identically distributed) data distributions cause statistical data heterogeneity, which negatively affects this loss function.

While \ac{fl} in wireless networks is particularly attractive because clients do not need to offload their privacy-sensitive data to the \ac{cs}, {\em resource constraints} on both the client and network sides need to be acknowledged \cite{niknam2020federated}. 
One fundamental assumption in typical \ac{fl} is that the client has a static dataset $\mathcal{D}_u$ available before training begins. 
However, $\mathcal{D}_u$ may not be static and/or readily available in many practical applications \cite{pervej2025resource, hosseinalipour2023parallel}.
Moreover, the statistical distributions of client data may not remain static.
The training samples can {\em continually arrive} \cite{pervej2022mobility,pervej2025resource,hosseinalipour2023parallel} in many applications.
This is particularly true for many wireless applications, such as \ac{isac} \cite{pulkkinen2024model}, temporal/online \ac{csi} prediction \cite{luo2018channel}, demand predictions in video caching network \cite{pervej2025resource}, simultaneous localization and mapping \cite{taniguchi2017online}, wearable sensors \cite{zhang2022online}, etc.
For example, new \ac{csi} are available in each radio frame, which is typically $10$ milliseconds. 
The clients can use their historical \ac{csi} to predict the future \ac{csi}, which may take only a few milliseconds, and the actual \ac{csi} information can be known at the end of that radio frame. 
Therefore, the clients can use their prediction results and the actual label for model training. 
In wireless localization \cite[Chapter $29$]{molisch2023wireless}, many users share real-time location information (some users also do not reveal such information), i.e., the actual label is known instantaneously, which can help design online \ac{ml} algorithms for efficient indoor/outdoor localization.

The {\em limited storage} capacity of clients poses another major concern in devising an efficient \ac{fl} solution in wireless networks, as clients can only store a limited number of training samples. 
Client devices are not expected to solely store these training samples and the \ac{ml} model; they also store user-generated and operational files for the \ac{os}.
Therefore, only a small chunk of the limited storage can be used for the \ac{fl} task. 
This can exacerbate the learning process, especially when the clients have very limited storage, e.g., \ac{iot} devices. 
As such, clients may need to {\em remove} old samples to make space for newly arrived samples. 
Therefore, the deleted samples will be lost forever.

This problem is {\em unique} to wireless applications and differs from traditional \ac{cfl} \cite{ma2022continual,qi2023better} where new tasks continually arrive. 
More specifically, unlike new tasks or classes arriving {\em incrementally}, in many wireless applications, the task (or the number of labels) may remain the same, but the data distributions change over time.
Take beam management for example: given that the \ac{bs} and \ac{ue} have fixed numbers of antennas, the distribution of the strongest beam can change over time and frequency due to time-varying frequency-selective channels and the distribution of various \acp{mpc} in the environment.
As such, regardless of the {\em distribution} of the initial labels in a client's dataset, it is bound to {\em change} due to the removal of old samples to make space for continually arriving new samples from {\em unknown classes}.

Time-varying datasets and the ephemeral use of some training samples can lead to severe performance degradation compared to general model training on a stationary dataset, which seeks to minimize the loss function in (\ref{localObj_Static}).  
Due to storage limitations and changes in clients' label distributions in subsequent training rounds, their optimal local models may no longer remain optimal in the new datasets, a phenomenon known as {\em concept drift} \cite{hosseinalipour2023parallel}, leading to instability in prediction performance.
Therefore, when $\mathcal{D}_u$ changes over time, we need to pay special attention and customize the learning algorithms.
Besides, since existing \ac{fl} algorithms primarily focus on static datasets, they may perform poorly in resource-constrained settings.
As such, any change in the training datasets due to new sample arrivals and old sample departures to accommodate the fixed memory storage of the clients necessitates new theoretical analysis and new \ac{fl} learning algorithms.
Moreover, the \ac{ml} model size can vary by applications, leading some clients to store only a few training samples in their remaining storage.
As such, keeping the dataset static by ignoring newly arrived samples and training the model on only a few samples may not sufficiently capture changes in the data distribution over time, which can degrade test performance.

\subsection{State-of-the-Art FL Algorithms}

\noindent
The so-called \ac{fedavg} algorithm \cite{mcmahan17Communication} paved the way for privacy-preserving \ac{fl}.
This seminal work assumes that distributed clients train a global model on their local datasets for an equal number of local rounds and then offload the trained model parameters to the central server.
However, this is not practical in many cases where system heterogeneity can introduce stragglers. 
Besides, under non-IID data distribution, client $u$'s $\mathcal{D}_u$ statistically differ from client $u'$'s $\mathcal{D}_{u'}$ which introduces client-drift in \ac{fedavg} \cite{karimireddy20scaffold}.


Following \ac{fedavg}, many works \cite{li2020federated,karimireddy20scaffold,wang2020Tackling,yang22Anarchic,ye23fedDisco} proposed ways to address the issues of system and data heterogeneity. 
FedProx \cite{li2020federated} mitigates the straggler effect by allowing partially locally trained model aggregation and adding a proximal term to the local objective function in order to keep the local model parameter close to the initially received global model.
\cite{karimireddy20scaffold} mitigates client drifts using local control variates. 
The problem of system heterogeneity was further explored in \ac{fednova} \cite{wang2020Tackling}.
This pioneering work advocates aggregating clients' normalized trained gradients rather than their raw trained gradients to find the global model.    
The normalized gradients are also used in the global aggregation policy in \ac{afa} to handle the data and system heterogeneity \cite{yang22Anarchic}.
Recently, \ac{feddisco} also considered normalized gradients with label/class distribution discrepancies to find global aggregation weights \cite{ye23fedDisco}.

\subsection{State-of-the-Art FL in Wireless Networks}

\noindent
Many recent studies \cite{chen2020joint,amiri2020federated,tran2019federated,pervej2023resource,zhou2022joint} customize the above general \ac{fl} algorithms and train them in wireless networks. 
Most notably, the intertwined wireless networking and \ac{fl} parameters are often optimized jointly when the \ac{ml} model is exchanged between the clients and \ac{cs} using the time-varying wireless fading channels between the clients and \ac{bs}\footnote{These studies assume the \ac{bs} acts as the \ac{cs} or the \ac{bs} works as the medium to transport the model parameters between the clients and \ac{cs}.}     \cite{chen2020joint,amiri2020federated,tran2019federated,pervej2023resource,zhou2022joint}.
While these studies assume the resources are optimized for performing \ac{fl} training in a coordinated fashion, where the models are exchanged in every \ac{fl} training round, variations of these assumptions also exist \cite{zhao2024ensemble,li2024learning, liu2022resource}.  

Due to practical resource constraints in wireless networks and on the client side, not all clients may participate in model training in every global round, motivating partial client selection for model training \cite{pervej2025resource,zhang2022joint,saha2022data,yao2023gomore,chen2024efficient}. 
In our earlier work \cite{pervej2025resource}, we proposed a resource-aware \ac{hfl} with partial client participation, in which clients were selected to optimize a weighted objective function that balances the number of local \ac{sgd} steps and the associated energy cost.  
\cite{zhang2022joint} used a heuristic method to jointly configure client scheduling, local training rounds, and radio resource allocations to train a \ac{fedavg}-based algorithm in a wireless network. 
Data quality was quantified by data volume and label heterogeneity in \cite{saha2022data}, which was then leveraged to select a subset of clients for model training. 
Partial client participation was also optimized in \cite{yao2023gomore}, where clients' trained local models were used to update the global model when client uplink transmissions were error-free, whereas the previous round's global model was reused when they were erroneous.
Recently, \cite{chen2024efficient} not only considered partial client selection for \ac{fl} in resource-constrained wireless networks but also proposed partial model aggregations: only a few layers of the entire model are exchanged and aggregated.


Model pruning \cite{liu2018rethinking} and quantization \cite{polino2018model} are the other two prominent avenues to alleviate clients' limited computational power and communication overheads. 
While model pruning and model quantization are intended to mitigate resource constraints and leverage similar ideas, they differ largely in theory. 
In model pruning, some of the neurons of the \ac{ml} models are pruned to reduce the training overhead. 
In contrast, all parameters are quantized in model quantization, and the training happens with the quantized parameters. 
Naturally, some recent works \cite{chen2024adaptive,pervej2024hierarchical,wang2023performance,kim2023green} utilized these concepts for performing \ac{fl} in wireless networks.
In particular, \cite{chen2024adaptive,pervej2024hierarchical} propose leveraging model pruning and optimizing the pruning ratio based on available resources at the client and network sides. 
Besides, \cite{wang2023performance} and \cite{kim2023green} advocate choosing the quantization levels according to the available wireless and computation resources.

Lastly, gradient quantization \cite{alistarh2017qsgd} is also widely used to reduce communication overheads. 
Clients offload quantized gradients (or model differences) to the \ac{cs}, which then use these quantized updates to aggregate the global model \cite{reisizadeh2020fedpaq}. 
This novel gradient (or model differences) quantization concept, along with joint wireless and federated learning parameter optimization, has also been well explored in the literature \cite{liu2023Communication,hou2025lightweight,mao2022communication}.

\subsection{Research Gaps and Our Contributions}
\noindent
The above studies established how system heterogeneity and constrained wireless network resources necessitate jointly optimizing resources on both the client and network sides. 
However, the above studies mostly consider static datasets, which may certainly not be the case in many practical wireless applications with resource-constrained clients. 
With the continual rise of data-sensing applications, proper data management and new algorithms to address {\em shifts} in data distribution are necessary.
In this paper, we, therefore, propose an \ac{osafl} algorithm to address these shortcomings.
Our key contributions are 
\begin{itemize}
    \item We design a new algorithm---called \ac{osafl}---that specifically accounts for the time-varying and resource-constrained characteristics of the underlying wireless network and the clients, assuming that the clients perform their local training steps based on their available resources and remove old training samples to make space for new ones.
    \item To handle system and data heterogeneity, we use quantized normalized gradients in the global model aggregation policy. Based on our theoretical convergence analysis, we optimize clients' {\em aggregation scores} that facilitate the convergence of the proposed algorithm. Our analysis reveals that a client's contribution should be proportional to key parameters, such as their local rounds, participation probability, changes in loss functions due to data drift, and other parameters that fall under common assumptions for theoretical analysis.            
    \item Although, to the best of our knowledge, there are no exact existing baselines, we use existing popular algorithms and modify them for apples-to-apples performance comparison of \ac{osafl} on the image classification task with $3$ popular datasets with $3$ different \ac{ml} models, namely a shallow \sqz, a moderate \ac{cnn}, and a bulky \ac{resnet18}. 
    Our extensive simulation analysis shows that the proposed \ac{osafl} algorithm outperforms these baselines in almost all scenarios.
\end{itemize}

The rest of the paper is organized as follows.
We first discuss the preliminaries of \ac{fl} in resource-constrained wireless networks with time-varying datasets in Section \ref{prelim}.
Then, we introduce our \ac{osafl} algorithm in Section \ref{sec_OSAFL_Alg}, followed by an extensive theoretical analysis of it in Section \ref{sec_conv_osafl}.
Section \ref{sec_results} presents our empirical results and discussions. 
Finally, Section \ref{sec_conclusion} concludes the paper.

\section{FL in Constrained Wireless Networks with Time-Varying Datasets: Preliminaries}
\label{prelim}
\noindent
In this work, we consider that wireless devices act as clients, and the \ac{cs} is embedded into the clients' serving \ac{bs}.

\subsection{Dataset Acquisition for Model Training}
\noindent
Since wireless devices have limited resources, we assume that clients can only dedicate a small chunk of their storage to store the \ac{ml} model and training samples. 
Without any loss of generality, denote the maximum number of training data samples client $u$ can store by $\mathrm{D}_u$. 
Besides, each client has its own initial dataset $\mathcal{D}_u^{t=0}$, where $|\mathcal{D}_u^{t}| = \mathrm{D}_u$, $|\cdot|$ represents the cardinality of a set and $t$ represents the $t^{\mathrm{th}}$ \ac{fl} round. 
The clients have non-IID label distributions; i.e., the label distributions across clients differ.
Furthermore, we assume that each client can have at most $E_u$ new training samples between two consecutive \ac{fl} rounds. 
More specifically, each client has a probability of obtaining new samples. 
In many wireless applications, these new samples may come from historical observations and/or from continual sensing using the client's onboard sensor(s). 
While we assume that the arrival probabilities of new data are fixed, the instantaneous label distributions are {\em stochastic} and can change during each training round. 
Moreover, the distributions within a client itself can change from one round to another.
For simplicity, we assume that there are $E_u$ number of possible sample arrival slots, each with probability $p_{u,\mathrm{ac}}$, between two \ac{fl} rounds.
In other words, new sample arrival per slot is modeled as an independent \emph{Bernoulli} distribution with \emph{success probability} $p_{u,\mathrm{ac}}$. 
As such, the total number of new training samples between two \ac{fl} rounds can be modeled as a \emph{Binomial} distribution with parameters $(E_u, p_{u,\mathrm{ac}})$. 


As each client can only store $\mathrm{D}_u$ training samples, when a new training sample arrives, the client must remove an old training sample to make space for the newly arrived sample\footnote{One may model this data removal process using predefined rules and policies, which are beyond the scope of this paper.}.  
We consider $E_u < \mathrm{D}_u$, and the dataset is only updated before the start of a new global round. 
In practice, the arrived sample can be held in a temporary buffer, and the dataset can be updated only once before a new global round starts. 
Therefore, the training dataset $\mathcal{D}_u^t$ remains unchanged until the next, i.e., $(t+1)^{\mathrm{th}}$ \ac{fl} round begins.
Moreover, since model training only happens periodically in each global round, and each such round has a fixed duration, which is discussed in the sequel, we assume that each client has sufficient time to process these newly arrived samples to prepare as labeled data, which they use for the model training.

\subsection{FL with Time Varying Datasets}
\noindent
Let us denote the global model during global round $t$ by $\mathbf{w}^t$.
Since dataset $\mathcal{D}_u^t$ is time-varying, with the \ac{fedavg} algorithm \cite{mcmahan17Communication}, the \ac{cs} aims to minimize the following objective function in each \ac{fl} round.
\begin{align}
\label{globalObjective_FedAvg}
    f^t(\mathbf{w}^t) \coloneqq f(\mathbf{w}^t| \mathcal{D}^t) \coloneqq \sum\nolimits_{u=0}^{U-1} \alpha_u f_u^t (\mathbf{w}^t),
\end{align}
where $\mathcal{D}^t \coloneqq \bigcup_{u=0}^{U-1} \mathcal{D}_u^t$ and $f_u^t(\mathbf{w}^t)$ is the local objective function of client $u$, which is defined as
\begin{align}
\label{localObj}
    f_u^t (\mathbf{w}^t) \coloneqq f_u (\mathbf{w}^t|\mathcal{D}_u^t) = (1/|\mathcal{D}_u^t|) \sum\nolimits_{(\mathbf{x}, y) \in \mathcal{D}_u^t} l(\mathbf{w}^t|(\mathbf{x}, y)). 
\end{align}
Notice that the local and global objective functions with the dynamic datasets differ from the objective function in the general static case, as shown in (\ref{localObj_Static}) and (\ref{globalObj_Static}), respectively.

As we can see in (\ref{globalObjective_FedAvg}), clients' local loss functions affect the global loss function.
Besides, the updated global model also depends on the clients' updated model.    
However, since both the clients and the network have many resource constraints, and the wireless links are also time-varying, it is essential to optimize clients' (a) local training rounds, (b) \ac{cpu} clock cycles, and (c) transmission power for training any FL algorithm in resource-constrained wireless networks, which are discussed in the sequel.

\subsection{Joint Resource Optimization Under Resource Constraints}
\label{resourceOptim}
\noindent
We assume that the clients have a fixed deadline, denoted by $\mathrm{t_{th}}$, and a limited energy budget, denoted by $\mathrm{e}_{u,\mathrm{bd}}$, to perform local model computation and trained model offloading. 
As such, the clients need to explicitly consider the overheads for local model training and offloading in order to determine the number of local \ac{sgd} rounds, denoted by $\kappa_u^t$, that they can perform during global round $t$ to minimize (\ref{localObj}).
Each client calculates the associated computation time overhead as \cite{pervej2025resource} 
\begin{align}
    \mathrm{t}_{u,\mathrm{cp}}^t &\coloneqq (n\bar{n} c_u s_u \times \kappa_u^t)/\bar{f}_u^t,
\end{align} 
where $n$ is the number of mini-batches, $\bar{n}$ is the mini-batch size, $c_u$ is the number of \ac{cpu} cycles to compute $1$-bit data, $s_u$ is the data sample size in bits, and $\bar{f}_u^t$ is the \ac{cpu} cycle. 
Similarly, they calculate the energy overhead as \cite{pervej2025resource} 
\begin{align}
\label{eq:compute_energy_cost}
    \mathrm{e}_{u,\mathrm{cp}}^t &\coloneqq 0.5 \rho n \bar{n} c_u s_u \left(\bar{f}_u^t\right)^2 \times \kappa_u^t,
\end{align}
where $\rho$ is the effective capacitance of the \ac{cpu} chip.
Besides, they calculate the offloading time as \cite{pervej2025resource} 
\begin{align}
    \mathrm{t}_{u,\mathrm{up}}^t \coloneqq \Upsilon/ \left[\omega \log_2 \left( 1 + [\Xi_{u}^t \Gamma_u^t p_u^t]/[\omega \xi^2]\right)\right],
\end{align}
where $\Upsilon$ is the resultant uplink payload size from the model/gradient update, $\omega$ is the bandwidth, $\Xi_u^t$ is the large-scale path loss, $\Gamma_u^t$ is the log-Normal shadowing, $p_u^t$ is the  transmission power, and $\xi^2$ is the noise power spectral density. 

It is worth noting that the above data rate calculation makes two assumptions: (a) small-scale fading is {\em averaged out} due to diversity in modern wireless networks, and (b) $\Gamma_u^t$ is known at the time we need to make the optimization decisions (see below), either because it stays constant for the offloading duration, or because one may extrapolate. 
While transmission of model/gradient updates in the uplink may require multiple transmission slots, i.e., larger than channel coherence time, the above simplifications are widely used \cite{pervej2024hierarchical,pervej2025resource, hosseinalipour2023parallel}, and our key contributions are in the \ac{fl} algorithm presented in the sequel. 
Nevertheless, interested readers may refer to \cite{pervej2023resource} for such an in-depth per-transmission-slot-based analysis.

Thus, the energy overhead for payload offloading is
\begin{align}
    \mathrm{e}_{u,\mathrm{up}}^t \coloneqq \mathrm{t}_{u,\mathrm{up}}^t \cdot p_u^t.
\end{align} 

\subsubsection{Problem Formulation}
Given the fixed deadline and energy constraints, each client aims to jointly find the maximum possible local rounds and the corresponding optimal \ac{cpu} frequency and uplink transmission power in an energy-efficient way by solving the following optimization problem.
\begin{subequations}
\label{localIterOptim_Orig}
\begin{align}
    \underset{ \kappa_u^t, f_u^t, p_u^t }  {\tt{max} ~ } & ~~  \frac{\epsilon \cdot \kappa_u^t}{0.5 \rho n \bar{n} c_u s_u \left(\bar{f}_u^t\right)^2} + \frac {(1-\epsilon) \cdot \omega \log_2 \left(\rs 1 + \frac{\Xi_{u}^t \Gamma_u^t p_u^t}{\omega \xi^2} \rs \right) } {p_u^t} \tag{\ref{localIterOptim_Orig}} \\
    {\tt{s.t.}} & \quad 0 \leq \kappa_u^t \leq \kappa, \qquad \kappa_u^t \in \mathbb{Z}^{+}, \\
    &\quad 0 \leq p_u^t \leq p_{u,\mathrm{max}}, \\
    &\quad 0 \leq \bar{f}_u^t \leq \bar{f}_{u,\mathrm{max}}, \\
    &\quad \mathrm{e}_{u,\mathrm{cp}}^t + \mathrm{e}_{u,\mathrm{up}}^t \leq \mathrm{e_{bd}}, \label{eq:energy_cons} \\
    &\quad \mathrm{t}_{u,\mathrm{cp}}^t + \mathrm{t}_{u,\mathrm{up}}^t \leq \mathrm{t_{th}}, \label{eq:deadline_cons}
\end{align}
\end{subequations}
where $\epsilon \in [0,1]$ is a weighting parameter that balances the local training energy utility, defined as the fraction of the total local training rounds to the energy expense per local round\footnote{While it is possible to define this metric as the fraction of local training rounds to the corresponding energy expense, as calculated in (\ref{eq:compute_energy_cost}), such a definition will cancel out the $\kappa_u^t$ terms appearing both in the numerator and denominator. 
Thus, the objective function may not maximize the number of local training rounds but only minimize energy consumption for training and offloading, which does not serve the purpose of performing many rounds of local training before sending updates to the \ac{cs}.}, with the corresponding energy efficiency of offloading the trained model.
Besides, the constraints ensure local rounds, transmission power, \ac{cpu} frequency, total energy overhead, and total time overhead are within the allowable upper limits. 

\begin{Remark} 
By solving (\ref{localIterOptim_Orig}), each client intuitively aims to maximize the number of local training rounds, $\kappa_u^t$, thereby generally improving their corresponding model's predictive performance, while minimizing energy costs by jointly optimizing $\kappa_u^t$, $\bar{f}_u^t$, and $p_u^t$.
This problem, however, is a mixed-integer nonlinear problem and thus non-convex. 
In this work, we assume that each client has perfect \ac{csi} available and can iteratively solve (\ref{localIterOptim_Orig}). 
We stress that this problem can be solved in different ways, which is not particularly the key contribution of this work. 
\end{Remark}

\subsubsection{Problem Transformations and Iterative Solution}
The original problem is difficult to solve due to its non-convex nature:
joint solutions often require relaxation of the integer $\kappa_u^t$ variable to a continuous one and/or linearization of the non-convex/non-linear constraints.
However, one can observe that, when the other two optimization variables are known, the objective function is linear on $\kappa_u^t$, while monotonically decreasing on CPU frequency $\bar{f}_u^t$ and transmit power $p_u^t$.
As such, we seek a simple iterative solution to find the maximum possible local iteration without violating the constraints.
More specifically, we start with $\kappa_u^t=1$ and find the (sub-)optimal CPU frequency and transmit power as follows.

\noindent
\textbf{Optimize CPU Frequency given Local Rounds and Transmission Power:} 
Given the $\kappa_u^{t}$ and an initial feasible transmission power $p_u^{t,i}$, we optimize the \ac{cpu} frequency by transforming  (\ref{localIterOptim_Orig}) as
\begin{subequations}
\label{cpuFreqOptim_Sub_Prob}
\begin{align}
    &\underset{ \bar{f}_u^t }  { \tt{max} ~~ } \frac{\epsilon \cdot \kappa_u^{t}}{0.5 \rho n \bar{n} c_u s_u \left(\bar{f}_u^t\right)^2} + \frac {\omega\left(1-\epsilon\right) \log_2 \Big( 1 + \frac{\Xi_{u}^t \Gamma_u^t p_u^{t,i}} {\omega \xi^2}\Big) } {p_u^{t,i}}  \tag{\ref{cpuFreqOptim_Sub_Prob}} \\
    & {\tt{subject ~ to}} \quad 0 \leq \bar{f}_u^t \leq \bar{f}_{u,\mathrm{max}},\\ 
    &\quad 0.5 \rho n \bar{n} c_u s_u \kappa_u^{t} \left(\bar{f}_u^t\right)^2 + \frac{\Upsilon \cdot p_u^{t,i}} {\omega \log_2 \Big( 1 + \frac{\Xi_{u}^t \Gamma_u^t p_u^{t,i}}{\omega \xi^2} \Big) } \leq \mathrm{e_{bd}}, \\
    &\quad \frac{n\bar{n} c_u s_u \kappa_u^{t}}{\bar{f}_u^t} + \frac{\Upsilon} {\omega \log_2 \left( 1 + [\Xi_{u}^t \Gamma_u^t p_u^{t,i} ] / [\omega \xi^2] \right)} \leq \mathrm{t_{th}}.
\end{align}
\end{subequations}

\begin{Lemma}
\label{lemma_CPUfreqGiven_kappa_p}
Given $\kappa_u^{t}$ and $p_u^{t,i}$, the optimal solution of (\ref{cpuFreqOptim_Sub_Prob}) is
\begin{align}
\label{optimal_CPU_freq}
    \bar{f}_u^{t^*} = \frac {n\bar{n} c_u s_u \kappa_u^{t} \times \omega \log_2 \left( 1 + [\Xi_{u}^t \Gamma_u^t p_u^{t,i} ] / [\omega \xi^2]\right)} { \mathrm{t_{th}} \times \omega \log_2 \left( 1 + [\Xi_{u}^t \Gamma_u^t p_u^{t,i}]/[\omega \xi^2]\right) - \Upsilon}.
\end{align}
\end{Lemma}
The detailed proof of this lemma and the subsequent ones are left in the supplementary materials.
We stress that if $\sqrt{ \frac{ \mathrm{e_{bd}} \times \omega \log_2 \left( 1 + \frac{\Xi_{u}^t \Gamma_u^t p_u^{t,i}}{\omega \xi^2}\right)  - \Upsilon \cdot p_u^{t,i} } { 0.5 \rho n \bar{n} c_u s_u \kappa_u^{t} \times \omega \log_2 \left( 1 + \frac{\Xi_{u}^t \Gamma_u^t p_u^{t,i}}{\omega \xi^2}\right)} } < \frac {n\bar{n} c_u s_u \kappa_u^{t} \times \omega \log_2 \left( 1 + \frac{\Xi_{u}^t \Gamma_u^t p_u^{t,i}}{\omega \xi^2}\right) } { \mathrm{t_{th}} \times \omega \log_2 \left( 1 + \frac{\Xi_{u}^t \Gamma_u^t p_u^{t,i}}{\omega \xi^2}\right) - \Upsilon}$, \eqref{cpuFreqOptim_Sub_Prob} is infeasible.
Besides, if $\frac {n\bar{n} c_u s_u \kappa_u^{t} \times \omega \log_2 \left( 1 + \frac{\Xi_{u}^t \Gamma_u^t p_u^{t,i}}{\omega \xi^2}\right) } { \mathrm{t_{th}} \times \omega \log_2 \left( 1 + \frac{\Xi_{u}^t \Gamma_u^t p_u^{t,i}}{\omega \xi^2}\right) - \Upsilon} > \bar{f}_{u,\mathrm{max}}$ or $\sqrt{ \frac{ \mathrm{e_{bd}} \times \omega \log_2 \left( 1 + \frac{\Xi_{u}^t \Gamma_u^t p_u^{t,i}}{\omega \xi^2}\right)  - \Upsilon \cdot p_u^{t,i} } { 0.5 \rho n \bar{n} c_u s_u \kappa_u^{t} \times \omega \log_2 \left( 1 + \frac{\Xi_{u}^t \Gamma_u^t p_u^{t,i}}{\omega \xi^2}\right)} } < 0$, the problem is infeasible.

\noindent
\textbf{Optimize Transmission Power given Local Rounds and CPU Frequency:}
Given the local rounds $\kappa_u^{t}$ and \ac{cpu} frequency $\bar{f}_u^{t^*}$, we optimize the transmission power by transforming (\ref{localIterOptim_Orig}) as
\begin{subequations}
\label{txPowerOptim_Sub_Prob}
\begin{align}
    &\underset{ p_u^t }  { \tt{max} ~ } \frac{\epsilon \cdot \kappa_u^{t}}{0.5 \rho n \bar{n} c_u s_u \left(\bar{f}_u^{t^*}\right)^2} + \frac {\omega \left(1-\epsilon\right) \log_2 \left( 1 + \frac{\Xi_{u}^t \Gamma_u^t p_u^t}{\omega \xi^2}\right) } {p_u^t}  \tag{\ref{txPowerOptim_Sub_Prob}} \\
    &{\tt{subject ~ to}} \quad 0 \leq p_u^t \leq p_{u,\mathrm{max}}, \\
    &\quad 0.5 \rho n \bar{n} c_u s_u \kappa_u^{t} \left(\bar{f}_u^{t^*}\right)^2 + \frac{\Upsilon \cdot p_u^t} {\omega \log_2 \left( 1 + \frac{\Xi_{u}^t \Gamma_u^t p_u^t}{\omega \xi^2}\right) } \leq \mathrm{e_{bd}}, \\
    &\quad p_u^t \geq \Bigg[ \omega \xi^2 \Bigg( 2^{\Big[\frac{\Upsilon\bar{f}_u^{t^*} } { \omega \left(\mathrm{t_{th}} \bar{f}_u^{t^*} - n\bar{n} c_u s_u \kappa_u^{t} \right) } \Big]} - 1 \Bigg) \Bigg] / \left(\Xi_{u}^t \Gamma_u^t \right).
\end{align}
\end{subequations}

\begin{Lemma}
\label{lemmaTxPowerGiven_Kappa_CPUFreq}
Given $\kappa_u^t$ and $\bar{f}_u^{t^*}$, the optimal solution (\ref{txPowerOptim_Sub_Prob}) is
\begin{align}
\label{eq:optPower}
    p_u^{t^*} = \Bigg[ \omega \xi^2 \Bigg( 2^{\Big[\frac{\Upsilon\bar{f}_u^{t^*} } { \omega \left(\mathrm{t_{th}} \bar{f}_u^{t^*} - n\bar{n} c_u s_u \kappa_u^{t} \right) } \Big]} - 1 \Bigg) \Bigg] / \left(\Xi_{u}^t \Gamma_u^t \right).
\end{align}
\end{Lemma}

It should be noted that if $p_u^{t^*}> p_{u,\mathrm{max}}$ or forces $0.5 \rho n \bar{n} c_u s_u \kappa_u^{t} \left(\bar{f}_u^{t^*}\right)^2 + \frac{\Upsilon \cdot p_u^{t^*}} {\omega \log_2 \left( 1 + \frac{\Xi_{u}^t \Gamma_u^t p_u^t}{\omega \xi^2}\right) }$ to be larger than the energy budget $\mathrm{e_{bd}}$, then \eqref{txPowerOptim_Sub_Prob} is infeasible.

Thus, given the closed-form solutions, we iteratively obtain solutions to \eqref{cpuFreqOptim_Sub_Prob} and \eqref{txPowerOptim_Sub_Prob}.
Given the chosen $\kappa_u^t$, if problems \eqref{cpuFreqOptim_Sub_Prob} and \eqref{txPowerOptim_Sub_Prob} are feasible, we increase the value of $\kappa_u^t$ and then solve \eqref{cpuFreqOptim_Sub_Prob} and \eqref{txPowerOptim_Sub_Prob} iteratively again. 
We keep repeating this until the problem becomes infeasible or $\kappa_u^t=\kappa$. 
If the problem is infeasible for $\kappa_u^t$, the feasible solution is to use the previous value of $\kappa_u^t$ and the corresponding CPU frequency and transmit power.  
The detailed steps are summarized in Algorithm \ref{iterativeAlgorithm_LRO}.

\begin{Remark}
Resource optimization is necessary to train any \ac{fl} algorithms in resource-constrained wireless networks to ensure proper resource utilization and reduce the number of stragglers.
Although such resource optimization incurs additional computational overheads, it is indeed necessary to protect users' data privacy.
Therefore, joint resource optimization, as shown in optimization problem (\ref{localIterOptim_Orig}), is essentially a prerequisite for training privacy-preserving \ac{fl} algorithms in resource-constrained wireless networks.
\end{Remark}

\begin{algorithm}[!ht]
\small
\fontsize{8}{8}\selectfont 
\SetAlgoLined 
\DontPrintSemicolon
\KwIn{Initial feasible $p_u^{t,i=0}$; total iteration $I$}
Set $\kappa_u^{t^*}=1$, $\bar{f}_u^{t^*}=\bar{f}_\mathrm{max}$, $p_u^t=p_{u,\mathrm{max}}$ \;
\For{$\kappa_u^t=1$ to $\kappa$}{
    Get $\bar{f}_u^{t}$ and $p_u^{t}$ using {\tt GetFP$\big(\kappa_u^t, p_u^{t,i}, I\big)$} \;
    \If{$\bar{f}_u^{t} \neq {\tt nan}$ and $p_u^{t} \neq {\tt nan}$}{
        $\kappa_u^{t^*} \gets \kappa_u^t, \bar{f}_u^{t^*} \gets \bar{f}_u^t$, $p_u^{t^*} \gets p_u^t$ \; 
    }
}
\KwOut{Optimized local round $\kappa_u^{t^*}$, \ac{cpu} frequency $f_u^{t^*}$ and transmission power $p_u^{t^*}$}
\caption{\tt Iterative Solution for Local Resource Optimization}
\label{iterativeAlgorithm_LRO}
\end{algorithm}
\begin{algorithm}[!ht]
\small \fontsize{8}{8}\selectfont 
\For{$i=0$ to $I-1$}{
    Use $p_u^{t,i}$ to get $\bar{f}_u^{t^*}$ using (\ref{optimal_CPU_freq}) \;
    Use $\bar{f}_u^{t^*}$ to get $p_u^{t^*}$ using (\ref{eq:optPower}) \;
    \If{(\ref{optimal_CPU_freq}) or (\ref{eq:optPower}) is infeasible}{
    Return $p_u^{t}={\tt nan}$ and $\bar{f}_u^{t}={\tt nan}$ \;   
    }
}
\KwOut{Optimized local \ac{cpu} frequency $f_u^{t^*}$ and transmission power $p_u^{t^*}$}
\caption{\tt GetFP$\big(\kappa_u^t, p_u^{t,i}, I\big)$}
\label{alg:iterCPUpower}    
\end{algorithm}
\begin{algorithm}[!ht]
\small
\SetAlgoLined 
\DontPrintSemicolon
\KwIn{Initial global model $\mathbf{w}^0$, client set $\mathcal{U}$, total global round $T$, local learning rate $\eta_\mathrm{lo}$, and global learning rate $\eta_\mathrm{gl}$}
\For{$t=0$ to $T-1$}{
    \For{$u$ in $\mathcal{U}$ in parallel}{
        Client receives the latest global model from the CS \;
        Client synchronizes the local model: $\mathbf{w}_{u}^{t,0} \gets \mathbf{w}^{t}$ \;
        Client solves resource optimization problem  (\ref{localIterOptim_Orig}) using Algorithm \ref{iterativeAlgorithm_LRO} to determine $\mathrm{1}_u^t$ and local round $\kappa_u^t$  \label{alg_optim_line} \;
        \uIf{$\mathrm{1}_u^t=1$}{
        Client performs $\kappa_u^t$ SGD steps: $\mathbf{w}_u^{t,\kappa_u^t} = \mathbf{w}_u^{t,0} - \eta_\mathrm{lo} \sum\nolimits_{\tau=0}^{\kappa_u^t-1} g_u^t \left(\mathbf{w}_u^{t,\tau} \right)$ \;
        Client calculates normalized accumulated gradients $\mathbf{d}_u^t \coloneqq (\eta_\mathrm{lo} / \kappa_u^t) \sum\nolimits_{\tau=0}^{\kappa_u^t-1} g_u^t \left(\mathbf{w}_u^{t,\tau} \right)$ \;
        Client quantizes $\mathbf{d}_u^t$ using the stochastic quantizer $Q$ defined in (\ref{eq:stoQuantizer}) and sends $Q(\mathbf{d}_u^t)$ and $\kappa_u^t$ to the central server\;
        }
    }
    \Ac{cs} calculates scores $\left\{\delta_u^t\right\}_{u=0}^{U-1}$ based on convergence analysis using (\ref{eq:sub_optim_score}) and the global step size scheduler $\eta_\mathrm{ef}^t$ using (\ref{eq:eta_ef}) \tcp*{discussed in the sequel} 
    \Ac{cs} performs global aggregation: $\mathbf{w}^{t+1} = \mathbf{w}^t - \eta_\mathrm{gl} \eta_\mathrm{ef}^t \sum\nolimits_{u=0}^{U-1} \alpha_u^t \cdot \mathrm{1}_u^t \cdot Q \left( \mathbf{d}_u^t \right)$ \;
    }
\KwOut{Trained global model $\mathbf{w}^T$}
\caption{Proposed Online-Score-Aided FL}
\label{osafl_Algorithm}
\end{algorithm}

\section{Proposed Online-Score-Aided FL Algorithm for Wireless Applications}
\label{sec_OSAFL_Alg}

\noindent
When the client's dataset $\mathcal{D}_u^t$ is time-varying, a mere modification of the global objective function (\ref{globalObj_Static}) may not be sufficient to capture the intricate impact of the ephemeral training samples in the global model's performance. 
As such, we present our proposed \ac{osafl} algorithm in what follows.

At the beginning of a global round, each client receives the updated global model $\mathbf{w}^t$ from the \ac{cs} and synchronizes the local model as 
\begin{align}
\label{localModelSync}
    \mathbf{w}_u^{t,0} \gets \mathbf{w}^{t}, 
\end{align}
The clients then solve problem (\ref{localIterOptim_Orig}) to determine local training rounds and other optimization parameters, followed by performing $\kappa_u^t$ local mini-batch \ac{sgd} steps on their latest available local datasets $\mathcal{D}_u^t$ to minimize (\ref{localObj}).
As such, they update their local models as 
\begin{align}
\label{UEsLocalSGDrounds}
    \mathbf{w}_u^{t,\kappa_u^t} = \mathbf{w}_u^{t,0} - \eta_\mathrm{lo} \sum\nolimits_{\tau=0}^{\kappa_u^t-1} g_u^t \left(\mathbf{w}_u^{t,\tau} \right),
\end{align}
where $g_u^t(\mathbf{w}_u^{t,\tau}) \coloneqq g_u \left(\mathbf{w}_u^{t,\tau}|\mathcal{D}_u^t\right)$ and is the unbiased stochastic gradient of client $u$ and $\eta_\mathrm{lo}$ is the local learning rate.

Once the clients finish their local training, they calculate the normalized accumulated gradients as 
\begin{equation}
    \mathbf{d}_u^t \coloneqq (\eta_\mathrm{lo} / \kappa_u^t) \sum\nolimits_{\tau=0}^{\kappa_u^t-1} g_u^t \left( \mathbf{w}_u^{t,\tau} \right) = \big[\mathbf{w}_u^{t,0} - \mathbf{w}_u^{t,\kappa_u^t} \big]/ \kappa_u^t.
\end{equation}
The clients then send their quantized normalized gradient $Q\left(\mathbf{d}_u^t\right)$, where $Q(\cdot)$ is a stochastic quantizer. 
For $\mathbf{d} \in \mathbb{R}^N$ and $\mathbf{d} \neq \mathbf{0}$, this is defined as \cite{alistarh2017qsgd}  
\begin{align}
\label{eq:stoQuantizer}
    Q(\mathbf{d}) \coloneqq \Vert \mathbf{d}\Vert_2 \cdot \mathrm{sign}(d_i) \cdot \zeta_i (\mathbf{d},s), ~\forall i \in [N],
\end{align}
where $\zeta_i(\mathbf{d},s)$ is a random variable and satisfies $\zeta_i(\mathbf{d},s)= 
\begin{cases}
    [l+1]/s, & \mathrm{w.p.~} \big(\vert d_i\vert / \Vert \mathbf{d}\Vert_2 \big)s - l, \\
    l/s, & \mathrm{otherwise},
\end{cases}$,
where $d_i$ is the $i^\mathrm{th}$ entry of the vector $\mathbf{d}$, $s$ is the tuning parameter defining the total number of quantization levels, and $l \in [0,s)$ is an integer such that $\big(\vert d_i\vert / \Vert \mathbf{d}\Vert_2\big) \in \big[l/s, (l+1)/s \big]$.
We note that the wireless payload size for this quantized normalized gradient is  
\begin{align} 
\label{eq:model_payload}
    \Upsilon=N(1 + \lceil \log_2(s) \rceil)+ 32,
\end{align}
where $32$ bits are needed to offload the $\Vert \mathbf{d}_u^t\Vert_2$ with single precision, $N$ bits for the signs, and $N\left \lceil \log_2(s) \right \rceil$ for the quantization levels. 


The \ac{cs} takes a global \ac{sgd} step with step size $\eta_\mathrm{gl}$ using the normalized accumulated gradients as 
\begin{align}
\label{globalUpdateRule}
    \mathbf{w}^{t+1} 
    &= \mathbf{w}^t - \eta_{\mathrm{gl}} \eta_{\mathrm{ef}}^t \sum\nolimits_{u=0}^{U-1} \alpha_u^t \cdot \mathrm{1}_u^t \cdot Q\left(\mathbf{d}_u^t\right),
\end{align}
where $\eta_\mathrm{ef}^t >0$ is a scalar to control the update step size, $\alpha_u^t \coloneqq \alpha_u \delta_u^t$, $\alpha_u \coloneqq \mathrm{D}_u /\sum_{u=0}^{U-1} \mathrm{D}_u$, and $\delta_u^t \geq 0$ is the \emph{score} 
of client $u$ during time $t$.
Besides, $\mathrm{1}_u^t$ is an binary indicator function 
\begin{align}
    \mathrm{1}_u^t=
    \begin{cases}
        1, & \text{with probability } v_u^t,\\
        0, & \text{otherwise},
    \end{cases},
\end{align}
which is tied to the resource optimization problem (\ref{localIterOptim_Orig}) and time-varying wireless links. 
More specifically, the solution to (\ref{localIterOptim_Orig}) determines whether a client can participate in model training during a \ac{fl} round $t$.

In the update rule, \ac{osafl} uses two key parameters, the score function $\delta_u^t$ and the global step size controller $\eta_\mathrm{ef}^t$.
Intuitively, since data distribution shifts and other errors under different resource constraints affects local training and the offloading of trained normalized gradients, we want to jointly optimize the \emph{score} function and the step size controller to facilitate the algorithm's convergence.
Due to this update policy, \ac{osafl} minimizes the following surrogate global objective function instead of (\ref{globalObjective_FedAvg}).
\begin{align}
    f^t (\mathbf{w}^t) \coloneqq \sum\nolimits_{u=0}^{U-1} \alpha_u^t f_u^t (\mathbf{w}^t).
\end{align}
Therefore, \ac{osafl} seeks a {\em sequence of global models} $\mathbf{W}^{*}=\left\{{\mathbf{w}^0}^{*}, {\mathbf{w}^1}^{*}, \dots, {\mathbf{w}^{T-1}}^{*}\right\}$ so that each ${\mathbf{w}^t}^*$ in this sequence minimizes the above global loss function in that round $t$.

Algorithm \ref{osafl_Algorithm} summarizes the key steps of the proposed \ac{osafl} algorithm.

\section{Theoretical Analysis of OSAFL}
\label{sec_conv_osafl}
\noindent

\subsection{Assumptions}
\noindent
We make the following standard assumptions \cite{wang2020Tackling, yang22Anarchic, ye23fedDisco, reisizadeh2020fedpaq, pervej2023resource, pervej2024hierarchical} that are needed for the theoretical analysis. 


\begin{Assumption}[Smoothness]
    The local loss functions are $\beta$-Lipschitz smooth. That is, for some $\beta > 0$, $\Vert \nabla f_u^t \left(\mathbf{w} \right) - \nabla f_u^t \left(\mathbf{w}' \right) \Vert \leq \beta \Vert \mathbf{w} - \mathbf{w}' \Vert$, for all $\mathbf{w}$, $\mathbf{w}' \in \mathbb{R}^N$ and $u \in \mathcal{U}$.
\end{Assumption}
\begin{Assumption}[Unbiased gradient with bounded variance]
    The stochastic gradient at each client is an unbiased estimate of the client's true gradient, i.e., $\mathbb{E}_{\zeta \sim \mathcal{D}_u^t} [g_u^t \left(\mathbf{w}  \right)] = \nabla f_u^t \left(\mathbf{w} \right)$, where $\mathbb{E}[\cdot]$ is the expectation operator. 
    Besides, the stochastic gradient has a bounded variance, i.e., $\mathbb{E}_{\zeta \sim \mathcal{D}_u^t} \left[\Vert g_u^t \left(\mathbf{w} \right) - \nabla f_u^t \left(\mathbf{w} \right) \Vert^2\right] \leq \sigma^2$, for some $\sigma \geq 0$ and for all $u \in \mathcal{U}$. 
\end{Assumption}

\begin{Assumption}[Bounded gradient dissimilarity]
The divergence between the local and global gradients is bounded, i.e., $\left\Vert \nabla f_u^t \left(\mathbf{w}  \right) - \nabla f^t \left(\mathbf{w} \right) \right \Vert^2 \leq \varpi^2$, for some $\varpi \geq 0$, $\forall u \in \mathcal{U}$.
\end{Assumption}

\begin{Assumption}[Unbiased quantizer and bounded variance]
The stochastic quantizer is unbiased, i.e., $\mathbb{E}_Q[Q(\mathbf{d})] = \mathbf{d}$ and has bounded variance $\mathbb{E}_Q \left[ \left\Vert Q(\mathbf{d}) - \mathbf{d} \right\Vert^2 \right] \leq q \left\Vert \mathbf{d}\right\Vert^2$, for some positive real constant $q$.    
\end{Assumption}

\subsection{Convergence Analysis}
\noindent
We consider the expected average global gradient norm as an indicator of convergence of the proposed \ac{osafl} algorithm.

\begin{Theorem}
\label{convgRate}
Suppose the above assumptions hold. 
When the learning rates satisfy $\eta_\mathrm{gl}\eta_\mathrm{lo} < \frac{1}{\beta \eta_\mathrm{ef}^t}$, and $\eta_\mathrm{lo} < \text{min} \left\{ \frac{1}{8\beta\kappa}, \sqrt{\frac{4 - 16 \sum_{u=0}^{U-1} \alpha_u ( \delta_u^t (1 - v_u^t) )^2 - 33 \kappa \sum_{u=0}^{U-1} v_u^t (1+q-v_u^t) \left(\alpha_u \delta_u^t\right)^2}{192 \kappa^2 \beta^2 \sum_{u=0}^{U-1} \alpha_u ( v_u^t \delta_u^t)^2}} \right\}$, and $4 - 16 \sum_{u=0}^{U-1} \alpha_u ( \delta_u^t (1 - v_u^t) )^2 - 33 \kappa \sum_{u=0}^{U-1} v_u^t (1+q-v_u^t) \left(\alpha_u \delta_u^t\right)^2 \geq 0$, we have
\begin{align}
\label{eq:convRate_Per_Iter}
   &\eta_\mathrm{ef}^t \left\Vert \nabla f^t \left(\mathbf{w}^{t} \right) \right\Vert^2  
    \leq \frac{2 \big[ f^t(\mathbf{w}^t) - \mathbb{E} \left[ f^{t+1}(\mathbf{w}^{t+1}) \right]\big]}{\eta_\mathrm{gl} \eta_\mathrm{lo}} + \frac{2\Phi^t}{\eta_\mathrm{gl} \eta_\mathrm{lo}} + \nonumber\\
    &\mquad \Omega_\sigma (\sigma, \pmb{v,\kappa^t,\delta^t}) + \Omega_\varpi (\sigma, \pmb{v,\kappa^t,\delta^t}),    
\end{align}
where $\Phi^t \triangleq  \mathbb{E} \left[f^{t+1}(\mathbf{w}^{t+1}) \right] - \mathbb{E} \left[f^{t} (\mathbf{w}^{t+1})\right]$, $\Omega_\sigma (\sigma, \pmb{v,\kappa^t,\delta^t}) \coloneqq 2(2+q) \beta \eta_\mathrm{gl} \eta_\mathrm{lo} \left(\sigma \eta_\mathrm{ef}^t \right)^2 \sum_{u=0}^{U-1} v_u^t \alpha_u^2 \frac{ (\delta_u^t)^2}{\kappa_u^t} + 8 \eta_\mathrm{ef}^t \beta^2 \sigma^2 \eta_\mathrm{lo}^2 \sum_{u=0}^{U-1} \alpha_u \kappa_u^t (v_u^t\delta_u^t)^2 + 24 \sigma^2 \left(\eta_\mathrm{ef}^t \right)^2 \beta^3 \eta_\mathrm{lo}^3 \sum_{u=0}^{U-1} v_u^t(1+q-v_u^t) \left(\alpha_u \delta_u^t \kappa_u^t \right)^2$, and $\Omega_\varpi (\sigma, \pmb{v,\kappa^t,\delta^t}) \coloneqq 4 \eta_\mathrm{ef}^t \varpi^2 \sum_{u=0}^{U-1} \alpha_u ( \delta_u^t (1 - v_u^t) )^2 + 48 \eta_\mathrm{ef}^t \left(\beta \varpi \right)^2 \eta_\mathrm{lo}^2  \sum_{u=0}^{U-1} \alpha_u (v_u^t\delta_u^t)^2 \left(\kappa_u^t\right)^2 + 6 \beta \eta_\mathrm{gl} \eta_\mathrm{lo} \left(\eta_\mathrm{ef}^t \varpi\right)^2 \sum_{u=0}^{U-1} v_u^t(1+q-v_u^t) \left(\alpha_u \delta_u^t\right)^2 \kappa_u^t + 144 \eta_\mathrm{gl} \left(\varpi \eta_\mathrm{ef}^t \right)^2 \beta^3 \eta_\mathrm{lo}^3 \sum_{u=0}^{U-1} v_u^t(1+q-v_u^t) \left(\alpha_u \delta_u^t\right)^2 \left(\kappa_u^t \right)^3$.
An immediate consequence of the above equation is 
\begin{align}
\label{convRate_Eqn}
    &\tilde{A} \sum\nolimits_{t=0}^{T-1} \eta_{\mathrm{ef}}^t \mathbb{E} \left[\left\Vert \nabla f^t (\mathbf{w}^t) \right\Vert^2 \right] 
    \leq \frac{2 \tilde{A} \left(f^0(\mathbf{w}^0) - \mathbb{E} \left[ f^{T}(\mathbf{w}^{T}) \right] \right)}{\eta_\mathrm{gl} \eta_\mathrm{lo} } + \nonumber\\
    &\squad \left(2 \tilde{A} /[\eta_\mathrm{gl} \eta_\mathrm{lo}] \right) \sum\nolimits_{t=0}^{T-1} \Phi^t + \tilde{A} \sum\nolimits_{t=0}^{T-1} \Omega_\sigma (\pmb{\sigma, v, \kappa^t, \delta^t}) + \nonumber\\
    &\squad \tilde{A} \sum\nolimits_{t=0}^{T-1} \Omega_\varpi (\pmb{\varpi, v, \kappa^t, \delta^t}),
\end{align}
where $\tilde{A} \coloneqq \left(\sum_{t=0}^{T-1} \eta_{\mathrm{ef}}^t\right)^{-1}$.
\end{Theorem}

The convergence bound in (\ref{convRate_Eqn}) reveals some key insights. 
The first term captures the changes in the global loss function when the algorithm runs for $T$ global rounds. 
The second term results from data {\em distribution shifts} across clients.
Besides, the third term arises from the bounded-variance assumption for stochastic gradients and also shows how noise from quantization and unknown client participation probabilities affect the convergence bound.
Finally, the last term is the consequence of heterogeneous statistical data distributions across clients in each global round, leading to bounded gradient divergence between the local and global loss functions.

When there is no data distribution shift, i.e., $\Phi^t = \mathbb{E}\left[f^{t+1}(\mathbf{w}^{t+1}) \right] - \mathbb{E}\left[f^{t} (\mathbf{w}^{t+1})\right]=0$, all clients participate, i.e., $v_u^t=1$, and $\kappa_u^t=\kappa$ for all clients and global rounds, we can immediately write
\begin{align} 
    &\tilde{A} \sum\nolimits_{t=0}^{T-1} \eta_{\mathrm{ef}}^t \mathbb{E} \left[\left\Vert \nabla f^t (\mathbf{w}^t) \right\Vert^2 \right] 
    \leq \frac{2 \tilde{A} \left(f^0(\mathbf{w}^0) - \mathbb{E} \left[ f^{T}(\mathbf{w}^{T}) \right] \right)}{\eta_\mathrm{gl} \eta_\mathrm{lo}} + \nonumber\\
    & \tilde{A} \sum\nolimits_{t=0}^{T-1} \tilde{\Omega}_\sigma (\pmb{\sigma, v, \kappa^t, \delta^t}) + \tilde{A} \sum\nolimits_{t=0}^{T-1} \tilde{\Omega}_\varpi (\pmb{\varpi, v, \kappa^t, \delta^t}),
\end{align}
where $\tilde{\Omega}_\sigma (\pmb{\sigma, v, \kappa^t, \delta^t}) \coloneqq  8 \kappa \eta_\mathrm{ef}^t \beta^2 \sigma^2 \eta_\mathrm{lo}^2 \sum_{u=0}^{U-1} \alpha_u (\delta_u^t)^2 + \frac{2(2+q) \beta \eta_\mathrm{gl} \eta_\mathrm{lo} \left(\sigma \eta_\mathrm{ef}^t \right)^2}{\kappa} \sum_{u=0}^{U-1} \alpha_u^2 (\delta_u^t)^2 + 24 q \left(\kappa \sigma \eta_\mathrm{ef}^t \right)^2 \left(\beta \eta_\mathrm{lo}\right)^3 \sum_{u=0}^{U-1} \left(\alpha_u \delta_u^t \right)^2$ and $\tilde{\Omega}_\varpi (\pmb{\varpi, v, \kappa^t, \delta^t}) \coloneqq 48 \eta_\mathrm{ef}^t \left(\kappa \beta \varpi \eta_\mathrm{lo} \right)^2 \sum_{u=0}^{U-1} \alpha_u (\delta_u^t)^2 + 6 q \kappa \beta \eta_\mathrm{gl} \eta_\mathrm{lo} \left(\eta_\mathrm{ef}^t \varpi\right)^2 \sum_{u=0}^{U-1} \left(\alpha_u \delta_u^t\right)^2 + 144 q \eta_\mathrm{gl} \left(\varpi \eta_\mathrm{ef}^t \right)^2 \left(\kappa \beta \eta_\mathrm{lo}\right)^3 \sum_{u=0}^{U-1} \left(\alpha_u \delta_u^t\right)^2$.

\begin{Remark}(Special Case)
When $\eta_{\mathrm{ef}}^t=1$, we have 
\begin{align}
    &\frac{1}{T} \sum_{t=0}^{T-1} \mathbb{E} \left[\left\Vert \nabla f^t (\mathbf{w}^t) \right\Vert^2 \right] 
    \leq \frac{2\left(f^0(\mathbf{w}^0) - \mathbb{E} \left[ f^{T}(\mathbf{w}^{T}) \right] \right)}{\eta_\mathrm{gl} \eta_\mathrm{lo} T} + \nonumber\\
    &\quad [2/(\eta_\mathrm{gl} \eta_\mathrm{lo} T)] \sum\nolimits_{t=0}^{T-1} \Phi^t + [1/T] \sum\nolimits_{t=0}^{T-1} \hat{\Omega}_\sigma (\pmb{\sigma, v, \kappa^t, \delta^t}) + \nonumber\\
    &\quad [1/T] \sum\nolimits_{t=0}^{T-1} \hat{\Omega}_\varpi (\pmb{\varpi, v, \kappa^t, \delta^t}),
\end{align}
where $\hat{\Omega}_\sigma (\pmb{\sigma, v, \kappa^t, \delta^t}) \coloneqq 2(2+q) \beta \eta_\mathrm{gl} \eta_\mathrm{lo} \sigma^2 \sum_{u=0}^{U-1} v_u^t \alpha_u^2 \frac{ (\delta_u^t)^2} {\kappa_u^t} + 8 \beta^2 \sigma^2 \eta_\mathrm{lo}^2 \sum_{u=0}^{U-1} \alpha_u \kappa_u^t (v_u^t\delta_u^t)^2 + 24 \sigma^2 \beta^3 \eta_\mathrm{lo}^3 \sum_{u=0}^{U-1} v_u^t (1 + q - v_u^t) \left(\alpha_u \delta_u^t \kappa_u^t \right)^2$ and $\hat{\Omega}_\varpi (\pmb{\varpi, v, \kappa^t, \delta^t}) \coloneqq 4 \varpi^2 \sum_{u=0}^{U-1} \alpha_u ( \delta_u^t (1 - v_u^t) )^2 + 48 \left(\beta \varpi \right)^2 \eta_\mathrm{lo}^2  \sum_{u=0}^{U-1} \alpha_u (v_u^t\delta_u^t)^2 \left(\kappa_u^t\right)^2 + 6 \beta \eta_\mathrm{gl} \eta_\mathrm{lo} \varpi^2 \sum_{u=0}^{U-1} v_u^t (1 + q - v_u^t) \left(\alpha_u \delta_u^t\right)^2 \kappa_u^t + 144 \eta_\mathrm{gl} \varpi^2 \beta^3 \eta_\mathrm{lo}^3 \sum_{u=0}^{U-1} v_u^t (1 + q - v_u^t) \left(\alpha_u \delta_u^t\right)^2 \left(\kappa_u^t \right)^3$.

For sufficiently large $T$, if the learning rates satisfy $\eta_\mathrm{gl}=U$ and $\eta_\mathrm{lo}=1/\sqrt{UT}$, we have
\begin{align}
    &\frac{1}{T} \sum_{t=0}^{T-1} \mathbb{E} \left[ \left\Vert \nabla f^t (\mathbf{w}^t) \right\Vert^2 \right] 
    \leq \mathcal{O} \bigg(\rs \frac{1}{\sqrt{UT}} \rs \bigg) + \mathcal{O} \bigg(\rs \frac{1} {\sqrt{UT}} \rs \times \rs \sum_{t=0}^{T-1} \Phi^t \rs \bigg) + \nonumber\\
    &\qquad \mathcal{O} \left( \frac{\sqrt{U}}{T^{1.5}} \times q \sigma^2 \sum_{t=0}^{T-1} \sum_{u=0}^{U-1} \frac{v_u^t}{\kappa_u^t} \left( \alpha_u \delta_u^t \right)^2 \right) + \nonumber\\
    &\qquad \mathcal{O} \left(\frac{1}{T} \times \varpi^2 \sum_{t=0}^{T-1}\sum_{u=0}^{U-1} \alpha_u \left[ \delta_u^t(1-v_u^t)\right]^2 \right) + \nonumber\\
    &\qquad \mathcal{O} \left( \frac{\sqrt{U}}{T^{1.5}} \times \varpi^2 \sum_{t=0}^{T-1} \sum_{u=0}^{U-1} v_u^t(1+q-v_u^t) \kappa_u^t \left( \alpha_u \delta_u^t \right)^2  \right), 
\end{align}
which shows how data distribution drift, client participation probability, noise from quantization and stochastic gradients, and divergence between the (local and global) loss functions affect the rate.
Under ideal condition, i.e., all $v_u^t=1$ and $\kappa_u^t=\kappa$, this bound boils down to
\begin{align}
    &\frac{1}{T} \sum_{t=0}^{T-1} \mathbb{E} \left[ \left\Vert \nabla f^t (\mathbf{w}^t) \right\Vert^2 \right] 
    \leq \mathcal{O} \bigg(\rs \frac{1}{\sqrt{UT}} \rs \bigg) + \mathcal{O} \bigg(\rs \frac{1} {\sqrt{UT}} \rs \times \rs \sum_{t=0}^{T-1} \Phi^t \rs \bigg) + \nonumber\\
    &\qquad \mathcal{O} \left( \frac{\sqrt{U}}{T^{1.5}} \times q\left[\frac{\sigma^2}{\kappa} + \kappa \varpi^2\right] \sum_{t=0}^{T-1} \sum_{u=0}^{U-1} \left( \alpha_u \delta_u^t \right)^2 \right).
\end{align}
\end{Remark}

The key takeaway message from the above Remark is that the accumulated drift (due to data distribution shift) shall not grow faster than $\sqrt{UT}$.

\subsection{Choice of Online Score and Global Step Size Controller}
\noindent
From (\ref{eq:convRate_Per_Iter}), if we pick the dominant term assuming the local learning rate $\eta_\mathrm{lo}$ is small, we have 
\begin{align}
\label{eq:conv_bound_dom_terms}
    &\left\Vert \nabla f^t \left(\mathbf{w}^{t} \right) \right\Vert^2  
    \leq \frac{2 \big[ f^t(\mathbf{w}^t) - \mathbb{E} \left[ f^{t+1}(\mathbf{w}^{t+1}) \right]\big]}{\eta_\mathrm{ef}^t \eta_\mathrm{gl} \eta_\mathrm{lo}} + \frac{2\Phi^t}{\eta_\mathrm{ef}^t \eta_\mathrm{gl} \eta_\mathrm{lo}} + \nonumber\\
    &\qquad 2(2+q) \beta \eta_\mathrm{ef}^t \eta_\mathrm{gl} \eta_\mathrm{lo} \sigma^2 \sum\nolimits_{u=0}^{U-1} v_u^t \alpha_u^2  (\delta_u^t)^2/\kappa_u^t + \nonumber\\
    &\qquad 4 \varpi^2 \sum\nolimits_{u=0}^{U-1} \alpha_u ( \delta_u^t (1 - v_u^t) )^2 + \nonumber\\
    &\qquad 6 \beta \eta_\mathrm{ef}^t \eta_\mathrm{gl} \eta_\mathrm{lo} \varpi^2 \sum\nolimits_{u=0}^{U-1} v_u^t(1+q-v_u^t) \kappa_u^t \left(\alpha_u \delta_u^t\right)^2.
\end{align}
Thus, in order to facilitate the convergence the right-hand side of (\ref{eq:conv_bound_dom_terms}) needs to be minimized.
Thus, we seek proper $\delta_u^t$ and $\eta_\mathrm{ef}^t$ to minimize this right-hand side.

\subsubsection{Optimization of the Scores}
Since $\eta_\mathrm{gl} \eta_\mathrm{lo} \leq [1 /(\beta \eta_{\mathrm{ef}}^t)]$, plugging $\eta_\mathrm{gl} \eta_\mathrm{lo} = [1 /(\beta \eta_{\mathrm{ef}}^t)]$ into (\ref{eq:conv_bound_dom_terms}), we get
\begin{align}
    &\left\Vert \nabla f^t (\mathbf{w}^t) \right\Vert^2 
    \leq 2 \beta ( f^t(\mathbf{w}^t) - \mathbb{E} \left[ f^{t+1}(\mathbf{w}^{t+1}) \right] ) + 2 \beta \Phi^t + \nonumber\\
    &\quad 2 (2+q) \sigma^2 \sum_{u=0}^{U-1} v_u^t \alpha_u^2 \frac{ (\delta_u^t)^2}{\kappa_u^t} + 4 \varpi^2 \sum_{u=0}^{U-1} \alpha_u \left(\delta_u^t (1 - v_u^t) \right)^2 + \nonumber\\
    &\quad 6 \varpi^2 \sum\nolimits_{u=0}^{U-1} v_u^t(1+q-v_u^t) \kappa_u^t \big(\alpha_u \delta_u^t\big)^2 \rs \triangleq \rs \theta^t \left(\pmb{\delta}^t, \pmb{\kappa}^t, \pmb{v}\right).
\end{align} 
To this end, we aim to optimize $\{\delta_u^t\}_{u=0}^{U-1}$ so that $\theta^t \left(\pmb{\delta}^t, \pmb{\kappa}^t, \pmb{v}\right)$ is minimized. 
Thus, we formulate
\begin{subequations}
\label{approxProb}
\begin{align}
    &\underset{  \left\{ \delta_u^t \right\}_{u=0}^{U-1} }  {\tt{minimize}} \squad \theta^t \left(\pmb{\delta}^t, \pmb{\kappa}^t, \pmb{v}\right) \tag{\ref{approxProb}} \\
    & {\tt{subject ~ to}} \qquad \delta_u^t \geq 0, \\
    &\sum_{u=0}^{U-1} \alpha_u ( \delta_u^t )^2 \left[16(1-v_u^t)^2 + 33 \kappa \alpha_u v_u^t(1+q-v_u^t) \right] \leq 4 \label{eq:constraint_score_optim},
\end{align}
\end{subequations}
where the first constraint is to ensure a non-negative score and the second constraint is needed to satisfy a used condition to obtain (\ref{eq:convRate_Per_Iter}).

While the exact solution of (\ref{approxProb}) is hard to find due to unknown participation probabilities, unknown data distribution drift, and implicit relationship between the global loss function and the score function, we find an approximate solution using the stationary condition from the \ac{kkt} as \cite[Ch. $5$]{boyd2004convex}
\begin{align}
    &\frac{\partial}{\partial \delta_u^t} \Big[ \theta^t  \left(\pmb{\delta}^t, \pmb{\kappa}^t, \pmb{q}\right) - \sum_{u=0}^{U-1} \rs \psi_{u}^t \delta_u^t - \tilde{\psi}^t \Big(\sum_{u=0}^{U-1} \alpha_u ( \delta_u^t )^2 \big[16(1-v_u^t)^2 + \nonumber\\
    &\mquad 33 \kappa \alpha_u v_u^t(1+q-v_u^t) \big] - 4 \Big) \Big] = 0, 
\end{align}
where $\psi_u^t$ and $\tilde{\psi}^t$ are Lagrange multipliers.

Using the definitions of $\Phi^t$ and global loss functions, the above stationary condition reduces to
\begin{align}
    & 2 \beta \alpha_u \left( f_u^t(\mathbf{w}^t) - \mathbb{E} \left[f_u^{t} (\mathbf{w}^{t+1})\right] \right) + 4 (2+q) \sigma^2 v_u^t \alpha_u^2 \frac{\delta_u^t}{\kappa_u^t} + \nonumber\\
    & 8 \varpi^2 \alpha_u \delta_u^t (1 - v_u^t)^2 +  12 \varpi^2 v_u^t(1+q-v_u^t) \kappa_u^t \delta_u^t \big(\alpha_u \big)^2 - \psi_{u}^t - \nonumber\\
    & 2\tilde{\psi}^t \alpha_u \delta_u^t \big[16(1-v_u^t)^2 + 33 \kappa \alpha_u v_u^t(1+q-v_u^t) \big] = 0, 
\end{align}
where the expectation operator is due to the randomness in the future occurrence of $\mathbf{w}^{t+1}$, i.e., $\mathbb{E} \left[f^t(\mathbf{w}^{t+1}) \right] = \mathbb{E} \left[\sum_{u=0}^{U-1}\alpha_u \delta_u^t f_u^t(\mathbf{w}^{t+1}) \right] = \sum_{u=0}^{U-1}\alpha_u \delta_u^t \mathbb{E} \left[f_u^t(\mathbf{w}^{t+1}) \right]$.

Upon solving for $\delta_u^t$, we get
\begin{align}
\label{eq:opt_score}
    &\delta_u^t 
    =  \frac{\left[\psi_{u}^t + 2 \beta \alpha_u \left( \mathbb{E} \left[ f_u^{t} (\mathbf{w}^{t+1})\right] - f_u^t(\mathbf{w}^t) \right) \right] \cdot \kappa_u^t} {\digamma_2 },
\end{align}
where $\digamma_2 \coloneqq 4 (2+q) v_u^t \sigma^2 \alpha_u^2 + 8 \alpha_u \kappa_u^t \varpi^2 (1 - v_u^t)^2 + 12 v_u^t (1+q-v_u^t) \varpi^2 (\kappa_u^t)^2 \big(\alpha_u \big)^2 - 2\tilde{\psi}^t \alpha_u \kappa_u^t \big[16(1 - v_u^t)^2 + 33 \kappa \alpha_u v_u^t(1+q-v_u^t) \big]$.

From (\ref{eq:opt_score}), we observe that $\delta_u^t$ is entangled with, \emph{inter alia}, $\kappa_u^t$, $v_u^t$, $\alpha_u$, and the change in the loss function on dataset $\mathcal{D}_u^t$ evaluated with two consecutive global models, which arises from the cancellation of the loss functions between two global rounds and the data distribution drift, $\Phi^t$, term.
It is important to note that we do not know $\mathbf{w}^{t+1}$ until we choose a score function $\delta_u^t$ to update the global model as shown in (\ref{globalUpdateRule}).
Therefore, we may not calculate $f_u^t(\mathbf{w}^{t+1})$ to find the optimal weight $\delta_u^t$.
Besides, the participation probability $v_u^t$ is also unknown.
Nonetheless, (\ref{eq:opt_score}) reveals some insights: it is a function of $\kappa_u^t$, $v_u^t$, $\varpi$, $\alpha_u$, the Lagrange multipliers, and the changes in the loss function induced by the data drift. 
Based on (\ref{eq:opt_score}) and the above facts, we {\em approximately} calculate clients' scores by leveraging gradient similarities and client participation as 
\begin{align}
\label{eq:sub_optim_score}
    \delta_u^t &
    =   \begin{cases}
        \lambda_u^t/\sum_{u'=0}^{U-1} \lambda_{u'}^t, & \text{if } t=0,\\
        \left[\varsigma \mathfrak{v}_u^t + (1-\varsigma) \lambda_u^t \right] \kappa_u^t, & \text{otherwise},\\
        \end{cases}, 
\end{align}
where $\mathfrak{v}_u^t \triangleq \exp \big[1 - \mathfrak{v} [(\sum_{t'=0}^{t-1} \mathrm{1}_u^{t'})/ (\sum_{u=0}^{U-1} \sum_{t'=0}^{t-1} \mathrm{1}_u^{t'})] \big]$ and $\mathfrak{v} > 0$ is a hyperparameter.
Besides, $\lambda_u^t \triangleq [\chi + \tilde{\lambda}_u^t] / [\chi + 1]$, where $\chi \geq 1$ is a control parameter to ensure $0 \leq \lambda_u^t \leq 1$, and $\tilde{\lambda}_u^t \triangleq {\tt{cosine}}(Q(\mathbf{d}_u^t), \mathbf{d}^t)$ is the {\tt cosine} similarity between $Q(\mathbf{d}_u^t)$ and $\mathbf{d}^t \triangleq \sum_{u=0}^{U-1} [\alpha_u \kappa_u^t/ (\sum_{u'=0}^{U-1} \alpha_{u'} \kappa_{u'}^t) ] \cdot Q(\mathbf{d}_u^t) $.  
Therefore, at the first round \ac{osafl} chooses the client's score as the normalized {\tt cosine} similarities between an individual client's quantized normalized gradient and accumulated weighted normalized quantized gradient $\mathbf{d}^t$ and for the subsequent global rounds, it balances between the participation score $\mathfrak{v}_u^t$ and gradient similarity $\lambda_u^t$.
Moreover, \ac{osafl} assigns more weights to the clients that participate less in $\mathfrak{v}_u^t$.

Intuitively, the gradient similarity provides a measure of how data drift, non-IID data distributions, and other errors arising from resource constraints affect the consecutive loss functions. 
We thus take the weighted combination of the client's gradient similarity with their participation frequencies, and weight the combined output by their corresponding local training rounds to calculate their scores.
Besides, it is easy to see that the maximum value of $\lambda_u^t$ can be $1$ when ${\tt{cosine}}(Q(\mathbf{d}_u^t), \mathbf{d}^t)=1$. 
Therefore, if the suboptimal score is chosen following (\ref{eq:sub_optim_score}), to satisfy (\ref{eq:constraint_score_optim}), the following must hold.
\begin{align} 
    &\sum\nolimits_{u=0}^{U-1} \alpha_u \digamma_3 \leq 4U^2, \text{ if } t=0. \\
    &\sum\nolimits_{u=0}^{U-1} \alpha_u \left( [\varsigma(\mathfrak{v}_u^t-1) + 1] \kappa_u^t \right)^2 \digamma_3 \leq 4, \text{ if } t>0,    
\end{align}
where $\digamma_3 \coloneqq \left[16(1-v_u^t)^2 + 33 \kappa \alpha_u v_u^t(1+q-v_u^t) \right]$.

\subsubsection{Choice of $\eta_\mathrm{ef}^t$}
We now focus on gaining insights into how to choose the step-size scheduler $\eta_\mathrm{ef}^t$.
For given learning rates $\eta_\mathrm{gl}$ and $\eta_\mathrm{lo}$, the ideal value of $\eta_\mathrm{ef}^t$ should minimize the right-hand side of (\ref{eq:conv_bound_dom_terms}).
Using the definitions of the global loss functions and $\Phi^t$, the right-hand side 
\begin{align}
    &\left(2 \sum\nolimits_{u=0}^{U-1} \alpha_u \delta_u^t \big[ f_u^t(\mathbf{w}^t) - \mathbb{E} \left[ f_u^t(\mathbf{w}^{t+1}) \right]\big] \right) / (\eta_\mathrm{ef}^t \eta_\mathrm{gl} \eta_\mathrm{lo}) + \nonumber\\
    &\qquad 2(2+q) \beta \eta_\mathrm{ef}^t \eta_\mathrm{gl} \eta_\mathrm{lo} \sigma^2 \sum\nolimits_{u=0}^{U-1} v_u^t \alpha_u^2  (\delta_u^t)^2/\kappa_u^t + \nonumber\\
    &\qquad 6 \beta \eta_\mathrm{gl} \eta_\mathrm{lo} \eta_\mathrm{ef}^t \varpi^2 \sum\nolimits_{u=0}^{U-1} v_u^t(1+q-v_u^t) \kappa_u^t \left(\alpha_u \delta_u^t\right)^2 + \nonumber\\
    &\qquad 4 \varpi^2 \sum\nolimits_{u=0}^{U-1} \alpha_u ( \delta_u^t (1 - v_u^t) )^2 \coloneqq \rs \tilde{\theta}^t \rs \left(\pmb{\delta}^t, \pmb{\kappa}^t, \pmb{v}, \eta_\mathrm{ef}^t\right)\rs.\rs
\end{align}
Then, using a similar strategy to minimize $\tilde{\theta}^t\left(\pmb{\delta}^t, \pmb{\kappa}^t, \pmb{v}, \eta_\mathrm{ef}^t\right)$ subject to the constraint $\eta_\mathrm{ef}^t >0$, and followed by considering the stationary condition of the \ac{kkt} conditions, we have
\begin{align}
    &\frac{\partial }{\partial \eta_\mathrm{ef}^t} [\tilde{\theta}^t\left(\pmb{\delta}^t, \pmb{\kappa}^t, \pmb{v}, \eta_\mathrm{ef}^t\right) - \hat{\psi}^t \eta_\mathrm{ef}^t] = \nonumber\\
    &\big[-\bviol{2}\sum\nolimits_{u=0}^{U-1} \alpha_u \delta_u^t \left(f_u^t(\mathbf{w}^t) - \mathbb{E} \left[f_u^{t}(\mathbf{w}^{t+1})\right] \right) \big] / [\eta_\mathrm{gl} \eta_\mathrm{lo} (\eta_\mathrm{ef}^t)^2] + \nonumber\\
    & 2 (2+q) \beta \eta_\mathrm{gl} \eta_\mathrm{lo} \sigma^2 \sum\nolimits_{u=0}^{U-1} v_u^t \alpha_u^2 (\delta_u^t)^2 / \kappa_u^t + \nonumber\\
    & 6 \beta \eta_\mathrm{gl} \eta_\mathrm{lo} \varpi^2 \sum\nolimits_{u=0}^{U-1} v_u^t (1 + q - v_u^t) \kappa_u^t (\alpha_u \delta_u^t)^2 - \hat{\psi}^t = 0,
\end{align}
where $\hat{\psi}^t$ is the Lagrange multiplier. 

Now, solving for $\eta_\mathrm{ef}^t$, we get
\begin{align}
\label{eq:opt_eta_ef}
    \eta_\mathrm{ef}^t 
    &= \frac{  \sqrt{2 \sum\nolimits_{u=0}^{U-1} \alpha_u \delta_u^t \left( f_u^t(\mathbf{w}^t) - \mathbb{E} \left[f_u^{t}(\mathbf{w}^{t+1})\right] \right) } } {\sqrt{\digamma_4} } \nonumber\\
    &=\sqrt{\left(2 (f^t(\mathbf{w}^t) - \mathbb{E} [f^{t}(\mathbf{w}^{t+1}) ] \right) / \digamma_4 },
\end{align}
where $\digamma_4 \coloneqq 2 (2+q) \beta \eta_\mathrm{gl} \eta_\mathrm{lo} \sigma^2 \sum\nolimits_{u=0}^{U-1} v_u^t \alpha_u^2 (\delta_u^t)^2 / \kappa_u^t + 6 \beta \eta_\mathrm{gl} \eta_\mathrm{lo} \varpi^2 \sum\nolimits_{u=0}^{U-1} v_u^t (1 + q - v_u^t) \kappa_u^t (\alpha_u \delta_u^t)^2 - \hat{\psi}^t$.

It is clear from (\ref{eq:opt_eta_ef}) that, similar to $\delta_u^t$, the global step size controller $\eta_\mathrm{ef}^t$ also depends on many {\em unknowns}.
However, we recognize that the optimal $\eta_\mathrm{ef}^t$ is the fraction of the change in the global loss functions, calculated with the same dataset with two consecutive global models, and the scores of the clients multiplied by other unknown hyperparameters.
As such, we use the following global step scheduler that leverages clients' aggregated weights and similarities of the successive rounds' global gradients.
\begin{align}
\label{eq:eta_ef}
    \eta_\mathrm{ef}^t =
    \begin{cases}
        1 / [\sum_{u=0}^{U-1} \alpha_u^t], & \text{if $t=0$},\\
        \eta_\mathrm{sch}^t ( \{ \exp[\bar{\lambda}^t - \mathfrak{a}] \} / \{ \sum_{u=0}^{U-1} \alpha_u^t \} ), & \text{otherwise},
    \end{cases},
\end{align}
where $\eta_\mathrm{sch}^t$ is a linearly decayed scalar, $\bar{\lambda} \triangleq {\tt cosine}(\mathbf{d}^{t-1}, \mathbf{d}^{t})$, and $\mathfrak{a} > 0$ is a hyperparameter.
It is worth noting that $\eta_\mathrm{ef}^t$ regulates the effective change in the global model, as shown in the update rule in (\ref{globalUpdateRule}). 
Therefore, if $\eta_\mathrm{sch}^t$ is decayed drastically, it may slow down the convergence speed.

\begin{figure*}
\begin{subfigure}{0.33\textwidth}
    \includegraphics[trim=15 5 45 25, clip, width=\textwidth]{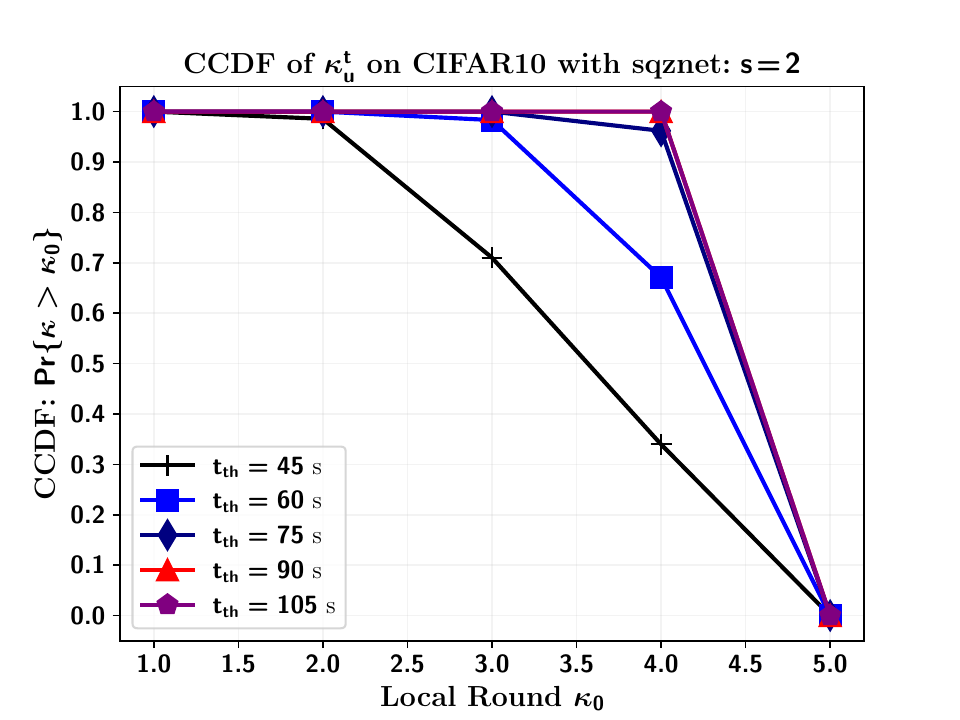}
    \caption{\sqz with $s=2$}
\end{subfigure}
\begin{subfigure}{0.33\textwidth}
    \includegraphics[trim=15 5 45 25, clip, width=\textwidth]{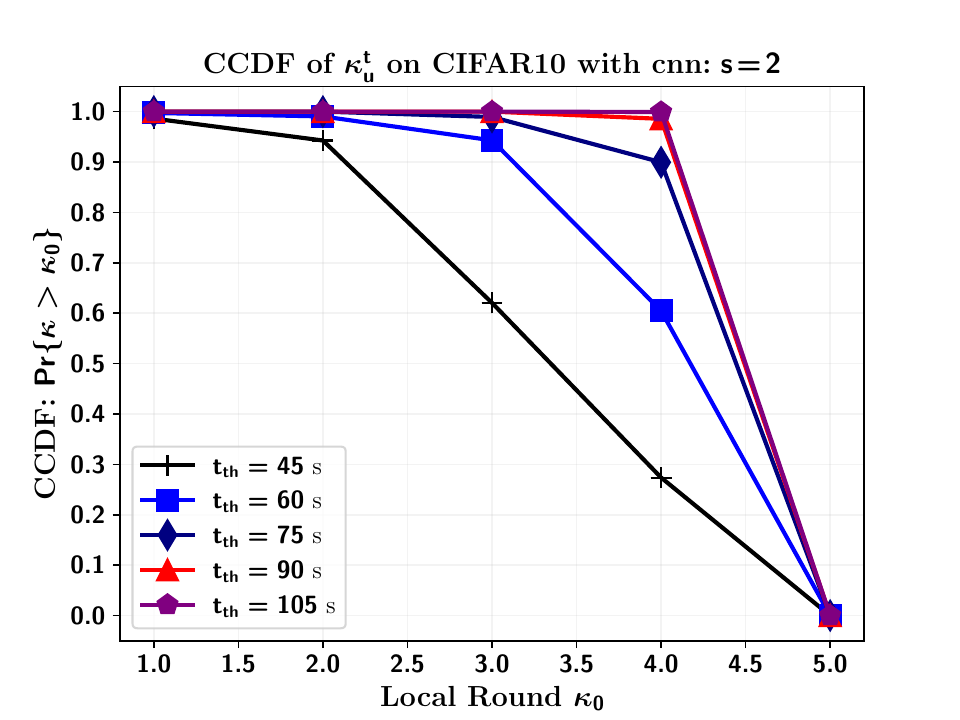}
    \caption{\ac{cnn} with $s=2$}
    \label{fig:cdfLocalItrVsQuant_cnn_s2}
\end{subfigure}
\begin{subfigure}{0.33\textwidth}
    \includegraphics[trim=15 5 45 25, clip, width=\textwidth]{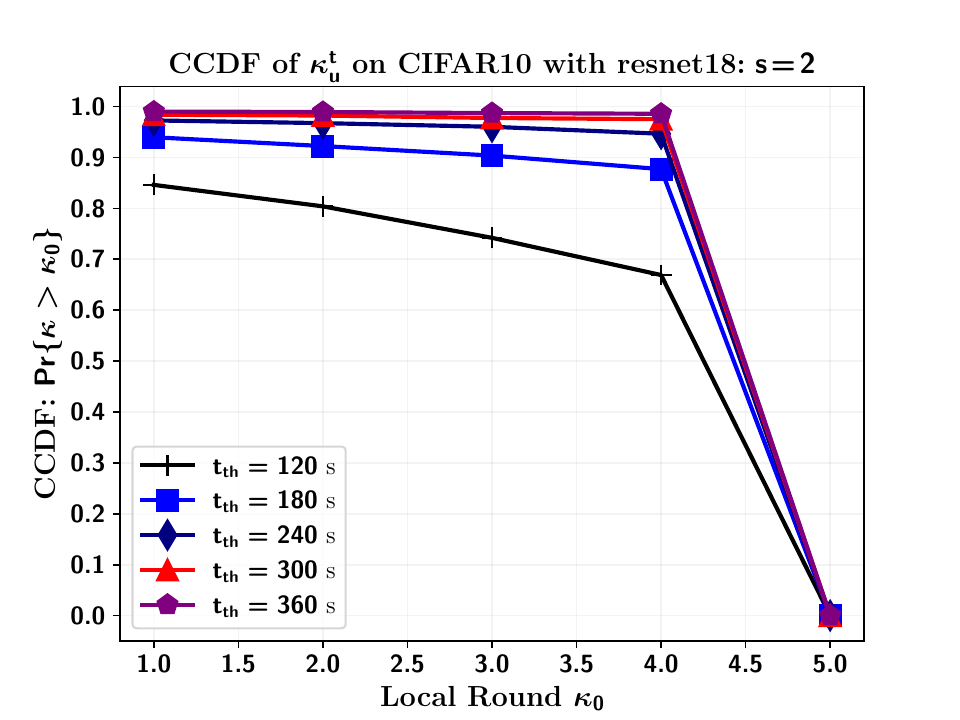}
    \caption{\ac{resnet18} with $s=2$}
    \label{fig:cdfLocalItrVsQuant_resnet_s2}
\end{subfigure}
\begin{subfigure}{0.33\textwidth}
    \includegraphics[trim=15 5 45 25, clip, width=\textwidth]{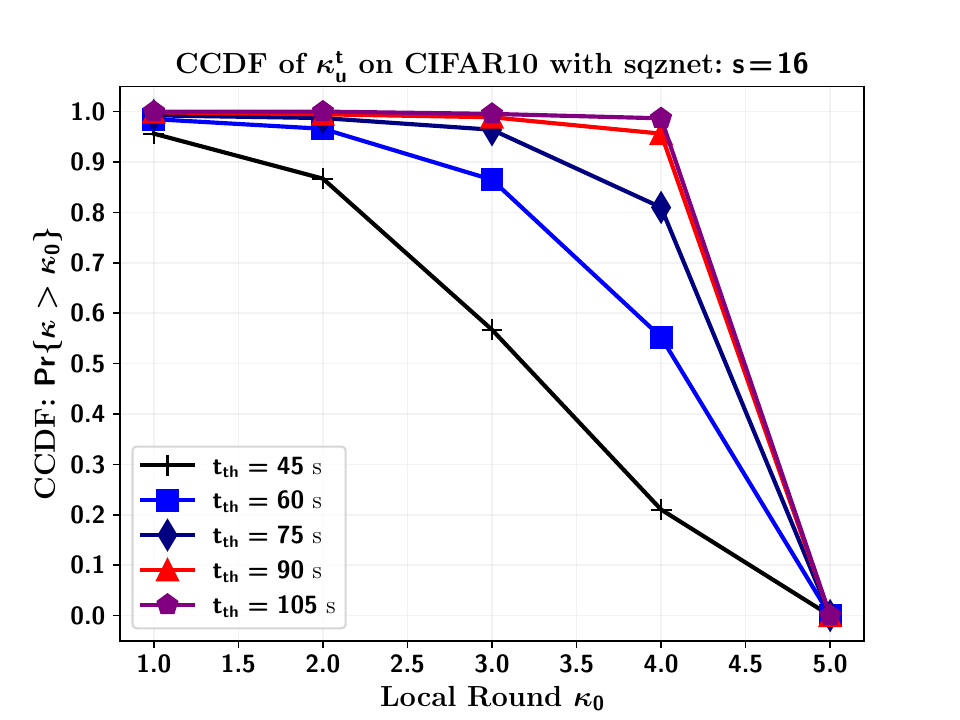}
    \caption{\sqz with $s=16$}
\end{subfigure}
\begin{subfigure}{0.33\textwidth}
    \includegraphics[trim=15 5 45 25, clip,width=\textwidth]{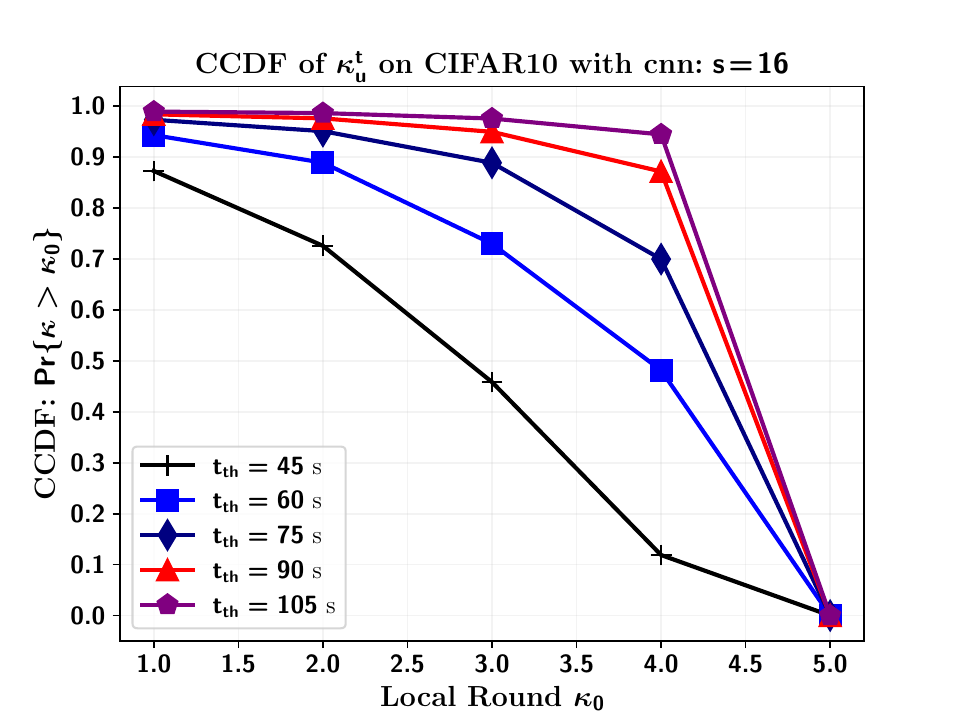}
    \caption{\ac{cnn} with $s=16$}
    \label{fig:cdfLocalItrVsQuant_cnn_s16}
\end{subfigure}
\begin{subfigure}{0.33\textwidth}
    \includegraphics[trim=1 5 45 25, clip, width=\textwidth]{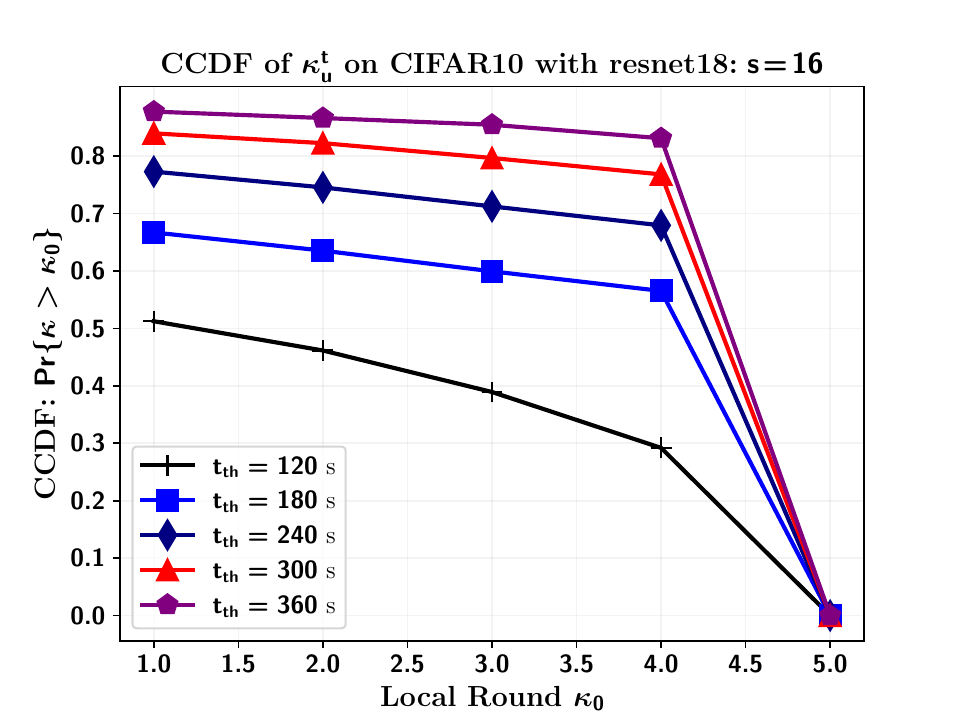}
    \caption{\ac{resnet18} with $s=16$}
    \label{fig:cdfLocalItrVsQuant_resnet_s16}
\end{subfigure}
\caption{CCDF of local iterations for different deadlines and quantization levels}
\label{fig:cdfLocalItrVsQuant}
\end{figure*}

\section{Simulation Results and Discussions}
\label{sec_results}

\noindent
In this section, we first discuss our simulation settings.
Then we present performance analysis for different parameters.
Finally, we discuss extensive performance comparisons with various baselines.

\subsection{Simulation Settings}
\subsubsection{Learning Task and Datasets}
 
While \ac{osafl} can be applicable to any wireless/other applications with time-varying datasets, obtaining real-world datasets for wireless applications is difficult, as wireless network operators do not publicly share such data.
Therefore, we use image classification as our primary learning task and evaluate the proposed \ac{osafl} algorithm and other baselines on $3$ popular datasets: CIFAR10 \cite{krizhevsky2009learning}, Fashion-MNIST\cite{xiao2017fashion}, and MNIST \cite{lecun2002gradient}.

\subsubsection{ML Models}
For the above task and datasets, we use three \ac{ml} models: (a) \ac{cnn}, (b) \sqz \cite{iandola2016squeezenet}, and (c) \ac{resnet18} \cite{he2016deep}.
The \ac{cnn} model has the following architecture: {\tt Conv2d $\rightarrow$ MaxPool2d $\rightarrow$ Conv2d $\rightarrow$ MaxPool2d $\rightarrow$ Linear $\rightarrow$ ReLU $\rightarrow$ Linear}.
With CIFAR10 dataset, these models have $727,626$; $1,206,090$; and $11,181,642$ trainable parameters, respectively. 
The reason to consider three model architectures with varying numbers of trainable parameters is to validate the performance of the proposed \ac{osafl} algorithm under different requirements.
Since batch normalization is problematic under non-IID data distribution in \ac{fl} \cite{li2021fedbn}, we replace the {\tt BatchNorm} layers in the original \sqz~ and \ac{resnet18} models with {\tt GroupNorm}.

\begin{figure*}
\begin{subfigure}{0.33\textwidth}
    \includegraphics[trim=12 5 30 25, clip, width=\textwidth]{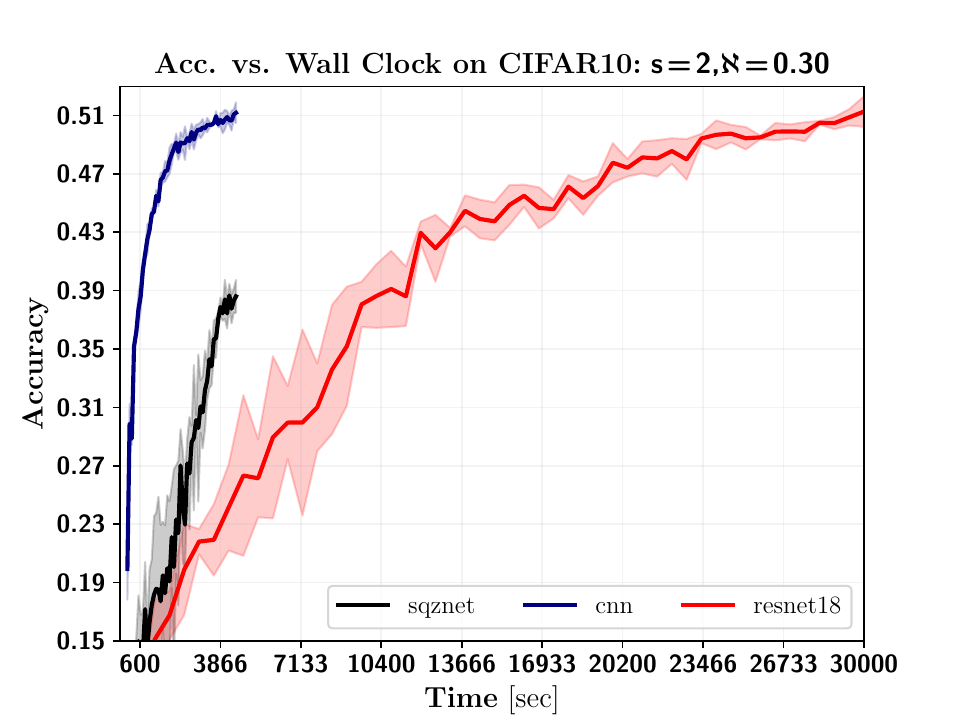}
    \caption{CIFAR10 with $s=2, \aleph=0.3$}
\end{subfigure}
\begin{subfigure}{0.33\textwidth}
    \includegraphics[trim=15 5 30 25, clip, width=\textwidth]{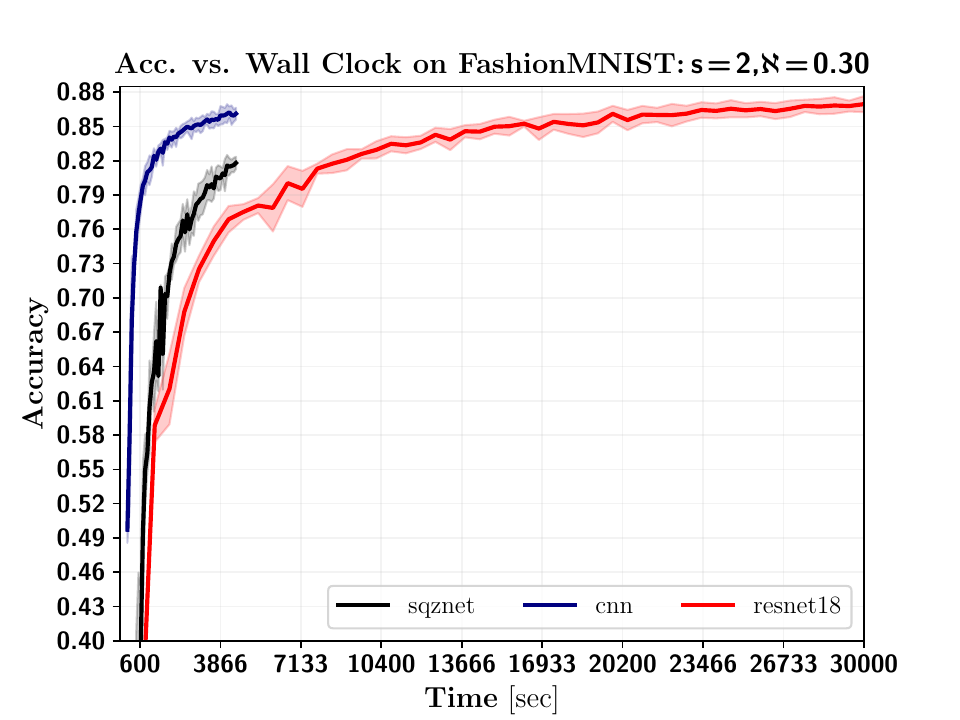}
    \caption{FashionMNIST with $s=2, \aleph=0.3$}
\end{subfigure}
\begin{subfigure}{0.33\textwidth}
    \includegraphics[trim=15 5 30 25, clip, width=\textwidth]{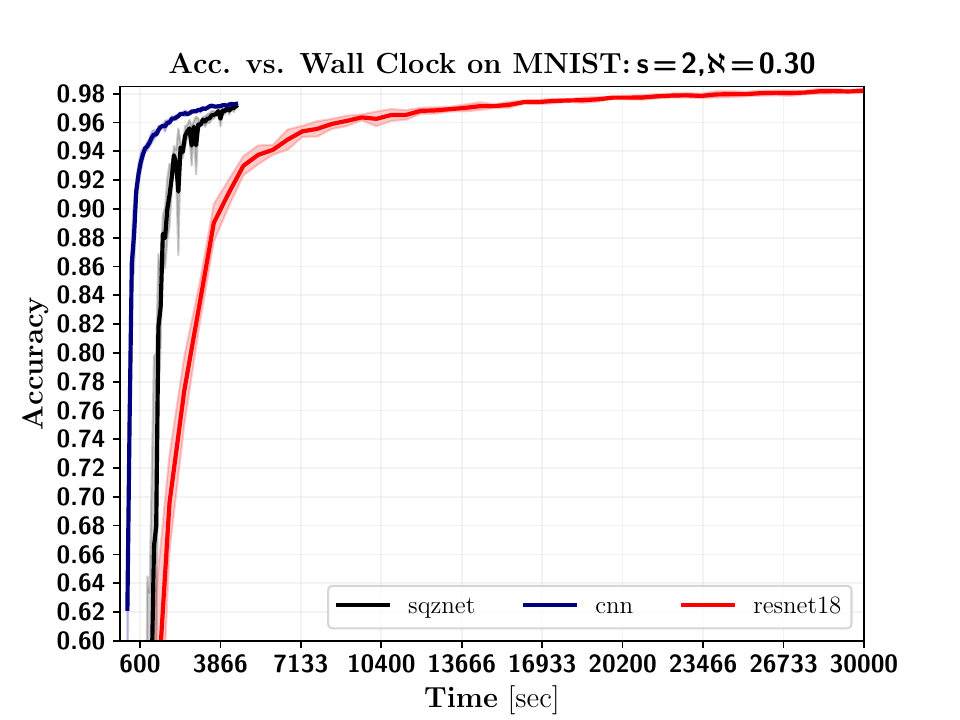}
    \caption{MNIST with $s=2, \aleph=0.3$}
\end{subfigure}
\caption{Wall clock vs test accuracies for different models on different datasets}
\label{fig:wallClock_Vs_Acc}
\end{figure*}
\begin{figure*}
\begin{subfigure}{0.33\textwidth}
    \includegraphics[trim=12 5 30 25, clip, width=\textwidth]{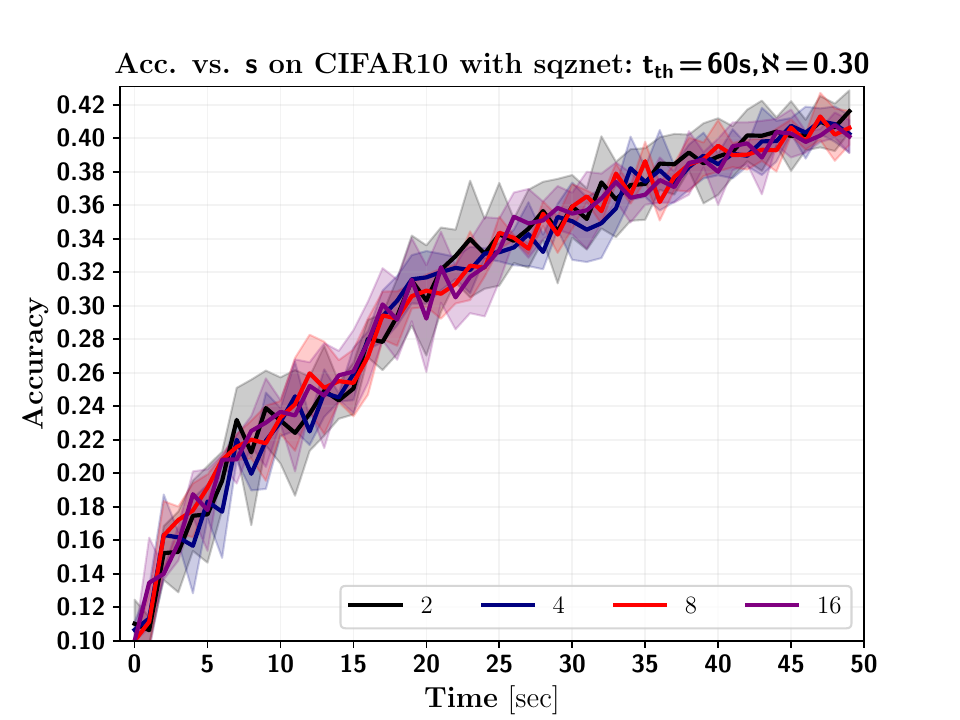}
    \caption{\sqz with $\aleph=0.3$}
    \label{fig:quantLevel_vs_acc_sqz}
\end{subfigure}
\begin{subfigure}{0.33\textwidth}
    \includegraphics[trim=15 5 42 25, clip, width=\textwidth]{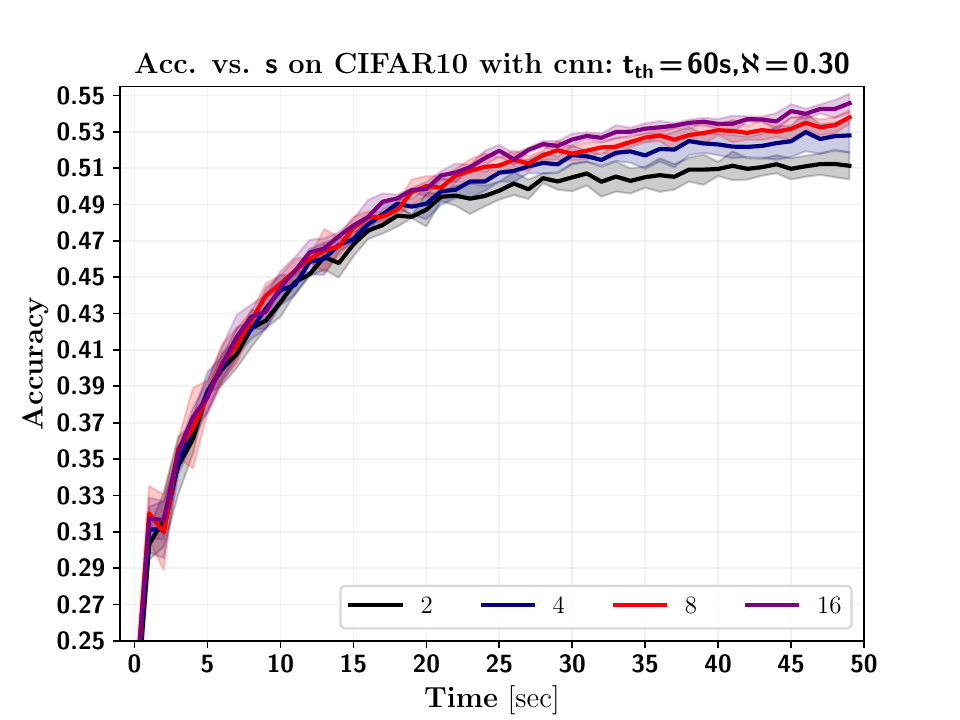}
    \caption{\ac{cnn} with $\aleph=0.3$}
    \label{fig:quantLevel_vs_acc_cnn}
\end{subfigure}
\begin{subfigure}{0.33\textwidth}
    \includegraphics[trim=15 5 45 25, clip, width=\textwidth]{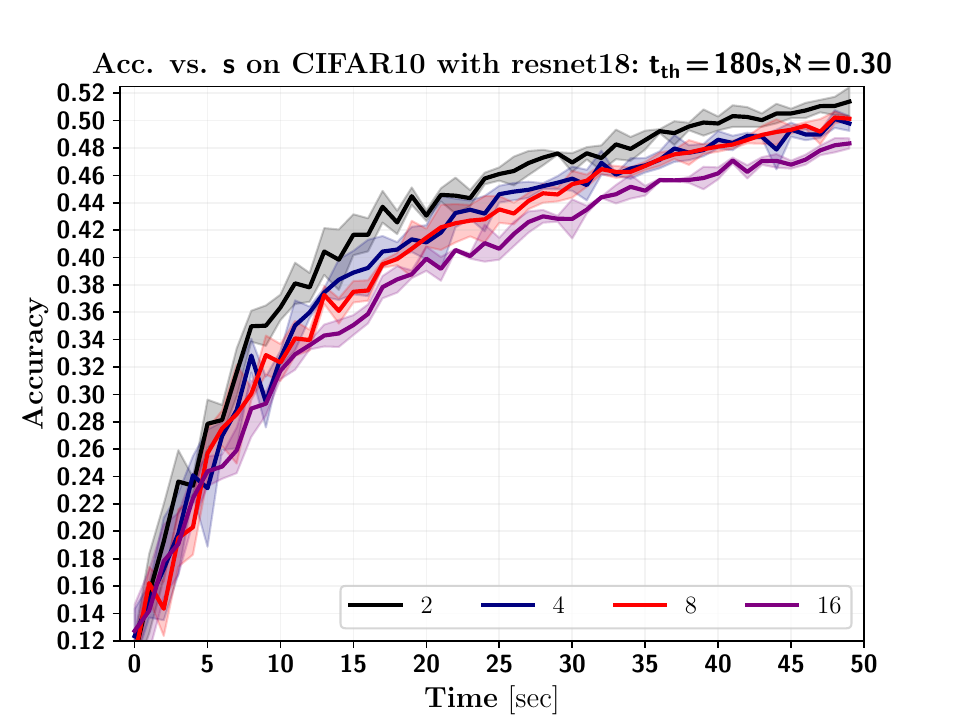}
    \caption{\ac{resnet18} with $\aleph=0.3$}
    \label{fig:quantLevel_vs_acc_resnet18}
\end{subfigure}
\caption{Quantization level vs test accuracies for different models on CIFAR10 dataset}
\label{fig:quantLevel_vs_acc}
\end{figure*}

\begin{figure*}[!t]
\begin{subfigure}{0.33\textwidth}
    \centering
    \includegraphics[trim=10 5 10 25,clip, width=\textwidth]{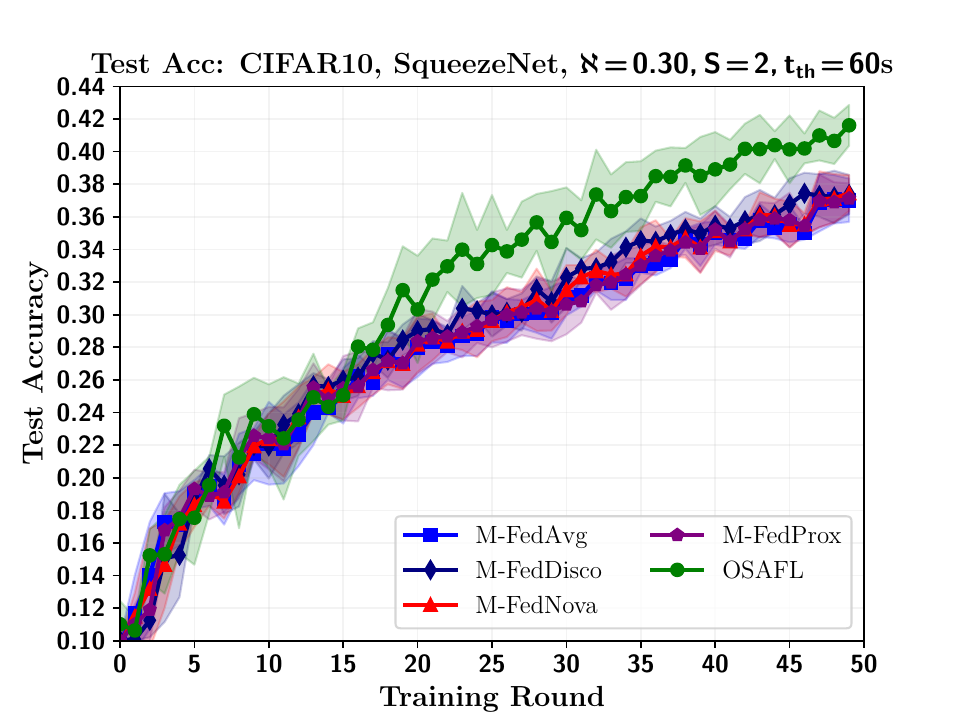}
    \caption{With \sqz on CIFAR10}
\end{subfigure}
\begin{subfigure}{0.33\textwidth}
    \centering
    \includegraphics[trim=10 5 10 25,clip, width=\textwidth]{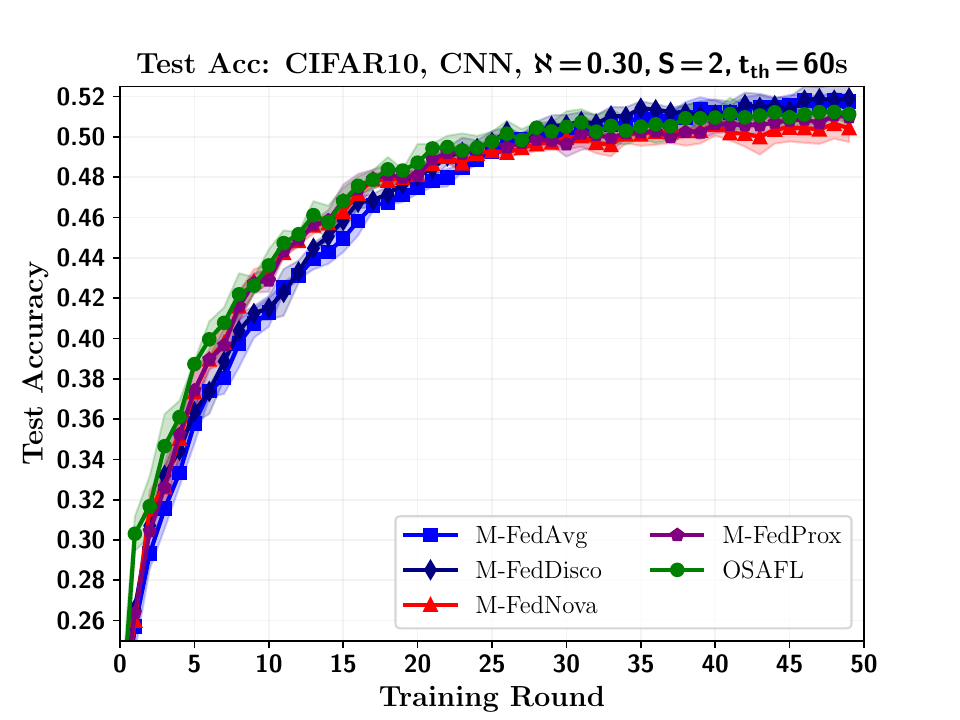}
    \caption{With \ac{cnn} on CIFAR10}
\end{subfigure}
\begin{subfigure}{0.33\textwidth}
    \centering
    \includegraphics[trim=10 5 10 25,clip, width=\textwidth]{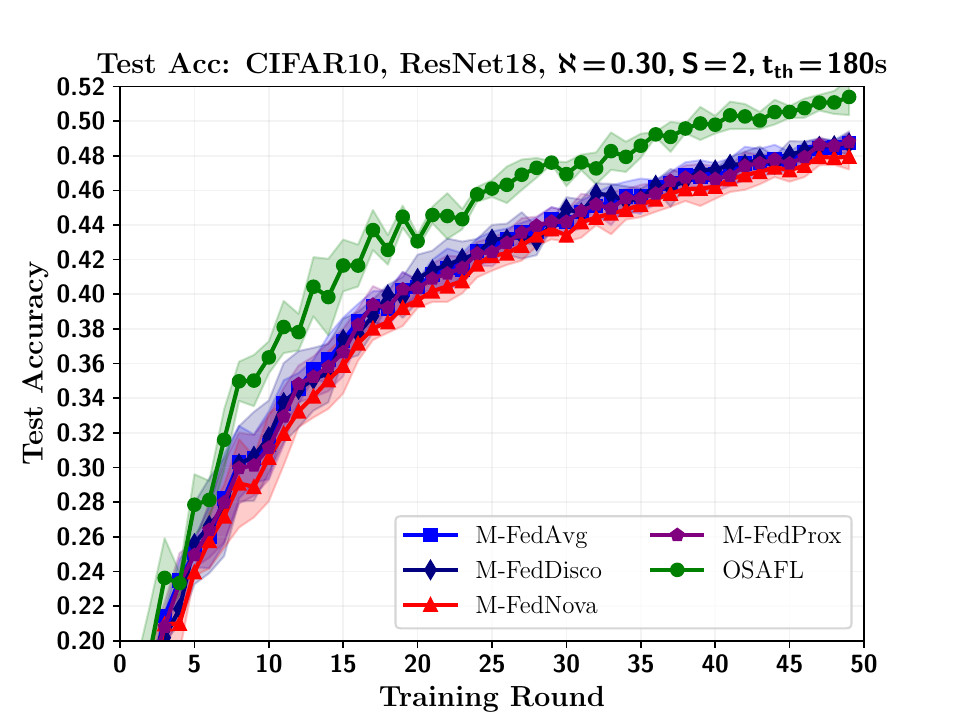}
    \caption{With \ac{resnet18} on CIFAR10}
\end{subfigure}
\begin{subfigure}{0.33\textwidth}
    \centering
    \includegraphics[trim=10 5 10 25,clip, width=\textwidth]{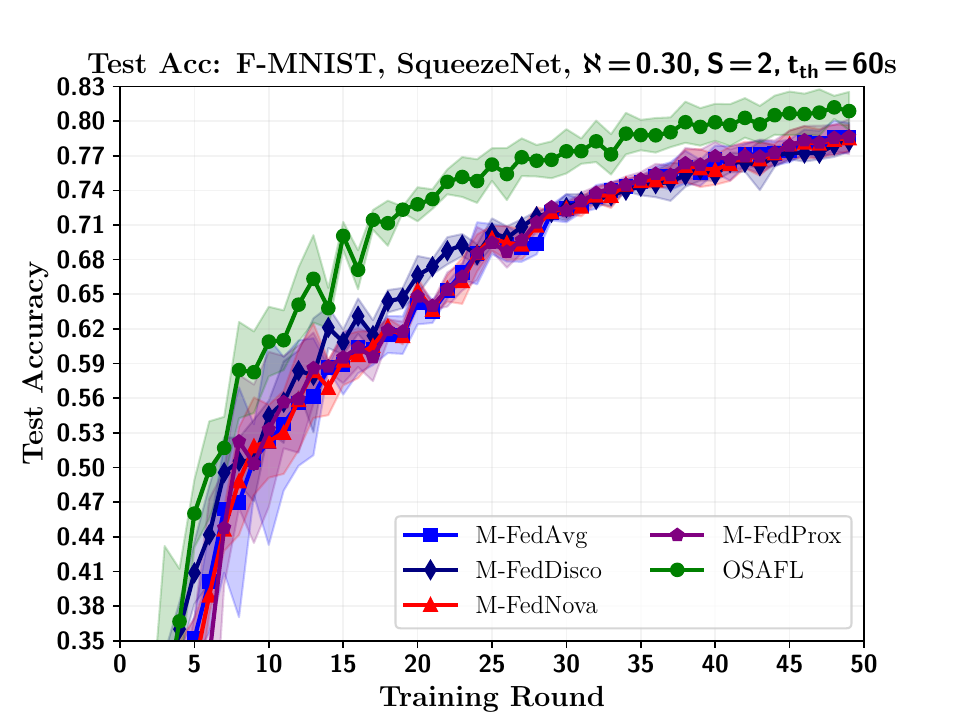}
    \caption{With \sqz on FashionMNIST}
\end{subfigure}
\begin{subfigure}{0.33\textwidth}
    \centering
    \includegraphics[trim=10 5 10 25,clip, width=\textwidth]{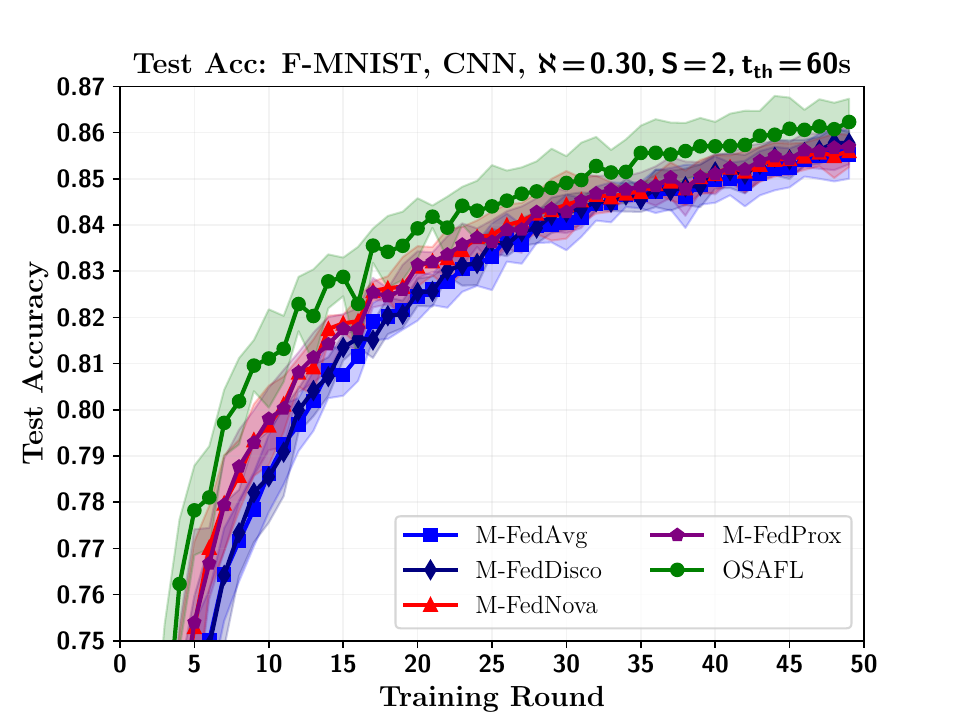}
    \caption{With \ac{cnn} on FashionMNIST}
\end{subfigure}
\begin{subfigure}{0.34\textwidth}
    \centering
    \includegraphics[trim=0 5 10 0,clip, width=\textwidth]{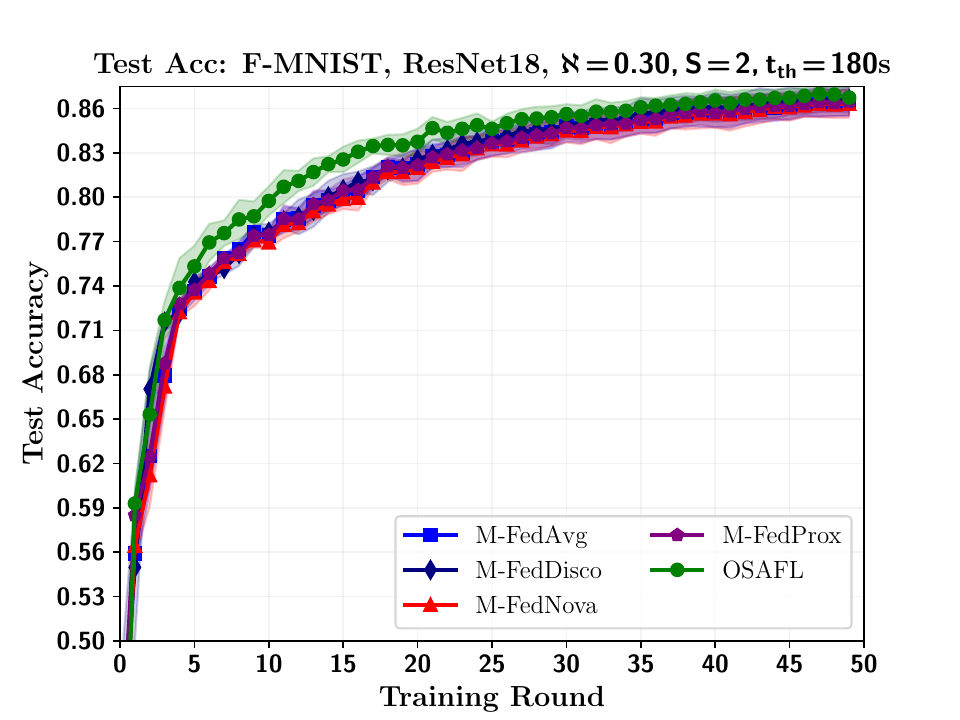}
    \caption{With \ac{resnet18} on FashionMNIST}
\end{subfigure}
\begin{subfigure}{0.33\textwidth}
    \centering
    \includegraphics[trim=10 5 10 25,clip, width=\textwidth]{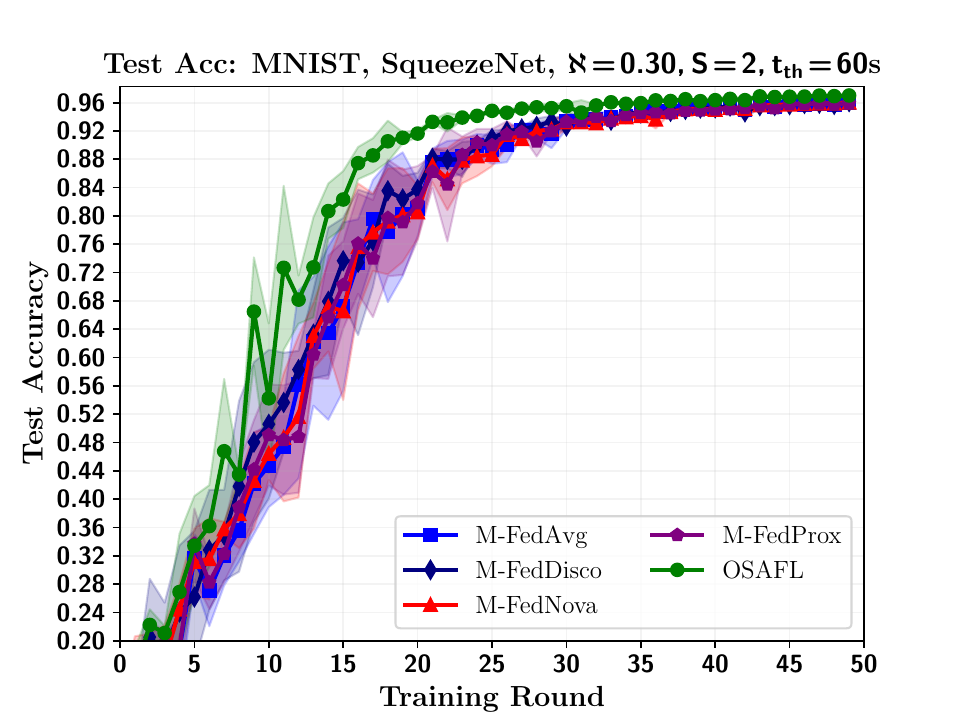}
    \caption{With \sqz on MNIST}
\end{subfigure}
\begin{subfigure}{0.33\textwidth}
    \centering
    \includegraphics[trim=10 5 10 25,clip, width=\textwidth]{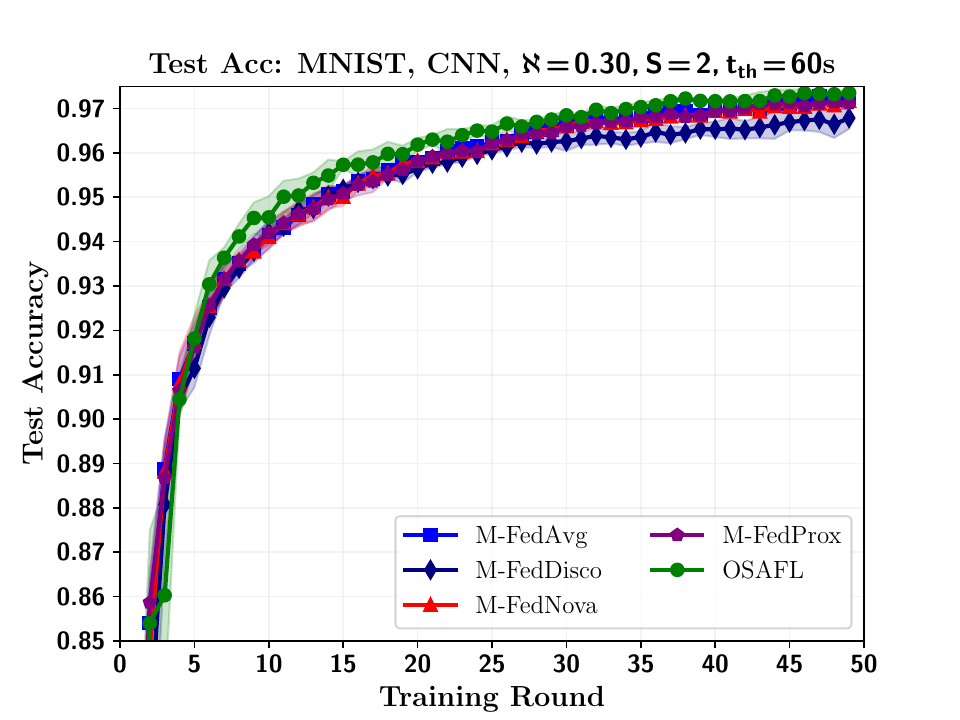}
    \caption{With \ac{cnn} on MNIST}
\end{subfigure}
\begin{subfigure}{0.34\textwidth}
    \centering
    \includegraphics[trim=10 5 10 25,clip, width=\textwidth]{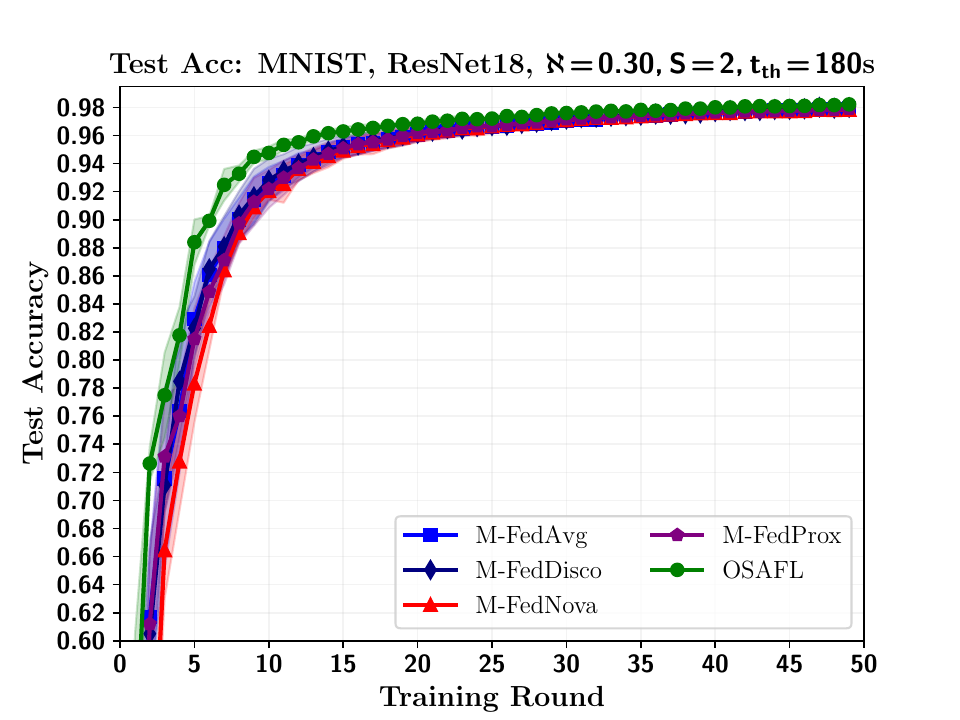}
    \caption{With \ac{resnet18} on MNIST}
\end{subfigure}
\caption{Test accuracy comparisons with different baselines on different datasets with the \ac{ml} models when $s=2$ and $\aleph=0.3$}
\label{fig:acc_comp_all_data_mod_s_2_alpha_0.3}
\end{figure*}
\begin{table*}
\centering
\caption{Test Performance Comparisons on \textbf{CIFAR10} with $s=2,\aleph=0.3$ \colorbox{green!30}{best}, \colorbox{red!30}{second best})}
\begin{tabular}{|C{0.8cm}|C{2cm}| C{1.84cm}|C{1.84cm}|C{1.84cm}|C{1.84cm}|C{1.84cm}| C{1.84cm}|} \hline
\multirow{2}{*}{\rs \textbf{Type}} &\multirow{2}{*}{\rs \textbf{Algorithms}} & \multicolumn{2}{c|}{\sqz, $\mathrm{t_{th}}=60$s } & \multicolumn{2}{c|}{\ac{cnn}, $\mathrm{t_{th}}=60$s} & \multicolumn{2}{c|}{\ac{resnet18}, $\mathrm{t_{th}}=180$s } \\ \cline{3-8}
& & \textbf{Test Acc.} $\uparrow$ & \textbf{Test Loss} $\downarrow$ & \textbf{Test Acc.} $\uparrow$ & \textbf{Test Loss} $\downarrow$ & \textbf{Test Acc.} $\uparrow$ & \textbf{Test Loss} $\downarrow$ \\ \hline
\rowcolor{blue!15} \rs \textbf{Genie} & 
SGD & $0.5658 \pm 0.0045$ & $1.2027 \pm 0.0184$ & $0.5934 \pm 0.0029$ & $1.2185 \pm 0.0157$ & $0.6393 \pm 0.0276$ & $1.3572 \pm 0.1063$ \\ \hline
\multirow{5}{*}{\makecell{\rotatebox{90}{\textbf{Federated}}}} & 
\textbf{OSA-FL (Ours)} & \cellcolor{green!30}  $0.4169 \pm 0.0132$ & \cellcolor{green!30} $1.5750 \pm 0.0451$ & $0.5135 \pm 0.0063$ & $1.4206 \pm 0.0264$ & \cellcolor{green!30} $0.5160 \pm 0.0086$ & \cellcolor{green!30} $1.3968 \pm 0.0178$ \\ \cline{2-8}
& M-FedAvg & $0.3719 \pm 0.0154$ & $1.7082 \pm 0.0159$ & \cellcolor{red!30} $0.5193 \pm 0.0030$ & \cellcolor{green!30} $1.3935 \pm 0.0160$ & $0.4872 \pm 0.0056$ & $1.4293 \pm 0.0109$ \\ \cline{2-8} 
& M-FedNova & \cellcolor{red!30} $0.3760 \pm 0.0119$ & $1.6969 \pm 0.0275$ & $0.5086 \pm 0.0063$ & $1.4563 \pm 0.0165$ & $0.4813 \pm 0.0051$ & $1.4388 \pm 0.0037$ \\ \cline{2-8}
& M-FedProx & $0.3745 \pm 0.0115$ & $1.6985 \pm 0.0108$ & $0.5123 \pm 0.0016$ & $1.4498 \pm 0.0220$ & \cellcolor{red!30} $0.4896 \pm 0.0031$ & $1.4306 \pm 0.0047$ \\ \cline{2-8}
& M-FedDisco & $0.3749 \pm 0.0129$ & \cellcolor{red!30} $1.6841 \pm 0.0257$ & \cellcolor{green!30} $0.5206 \pm 0.0072$ & \cellcolor{red!30} $1.4061 \pm 0.0271$ & $0.4888 \pm 0.0060$ & \cellcolor{red!30} $1.4232 \pm 0.0072$ \\ \hline 
\end{tabular}
\label{tab:accCompTable_cifar_s_2_alpha_0.3}
\end{table*}
\begin{table*}
\centering
\caption{Test Performance Comparisons on \textbf{FashionMNIST} with $s=2,\aleph=0.3$ (\colorbox{green!30}{best}, \colorbox{red!30}{second best})}
\begin{tabular}{|C{0.8cm}|C{2cm}| C{1.84cm}|C{1.84cm}|C{1.84cm}|C{1.84cm}|C{1.84cm}| C{1.84cm}|} \hline
\multirow{2}{*}{\rs \textbf{Type}} &\multirow{2}{*}{\rs \textbf{Algorithms}} & \multicolumn{2}{c|}{\sqz, $\mathrm{t_{th}}=60$s } & \multicolumn{2}{c|}{\ac{cnn}, $\mathrm{t_{th}}=60$s} & \multicolumn{2}{c|}{\ac{resnet18}, $\mathrm{t_{th}}=180$s } \\ \cline{3-8}
& & \textbf{Test Acc.} $\uparrow$ & \textbf{Test Loss} $\downarrow$ & \textbf{Test Acc.} $\uparrow$ & \textbf{Test Loss} $\downarrow$ & \textbf{Test Acc.} $\uparrow$ & \textbf{Test Loss} $\downarrow$ \\ \hline
\rowcolor{blue!15} \rs \textbf{Genie} & 
SGD & $0.8450 \pm 0.0378$ & $0.4338 \pm 0.1282$ & $0.8829 \pm 0.0020$ & $0.3303 \pm 0.0045$ & $0.9060 \pm 0.0030$ & $0.3127 \pm 0.0062$ \\ \hline
\multirow{5}{*}{\makecell{\rotatebox{90}{\textbf{Federated}}}} & 
\textbf{OSA-FL (Ours)} & \cellcolor{green!30} $0.8147 \pm 0.0119$ & \cellcolor{green!30} $0.4892 \pm 0.0214$ & \cellcolor{green!30} $0.8639 \pm 0.0060$ & \cellcolor{green!30} $0.3992 \pm 0.0188$ & \cellcolor{green!30} $0.8707 \pm 0.0063$ & \cellcolor{green!30} $0.3706 \pm 0.0167$ \\ \cline{2-8}
& M-FedAvg & $0.7871 \pm 0.0133$ & $0.5639 \pm 0.0248$ & $0.8557 \pm 0.0048$ & $0.4055 \pm 0.0142$ & $0.8645 \pm 0.0088$ & $0.3845 \pm 0.0208$ \\ \cline{2-8} 
& M-FedNova & $0.7863 \pm 0.0120$ & $0.5686 \pm 0.0211$ & $0.8562 \pm 0.0033$ & $0.4032 \pm 0.0149$ & $0.8637 \pm 0.0093$ & $0.3845 \pm 0.0227$ \\ \cline{2-8}
& M-FedProx & \cellcolor{red!30} $0.7880 \pm 0.0130$ & \cellcolor{red!30} $0.5636 \pm 0.0244$ & $0.8574 \pm 0.0040$ & $0.4069 \pm 0.0143$ & $0.8647 \pm 0.0095$ & $0.3840 \pm 0.0267$ \\ \cline{2-8}
& M-FedDisco & $0.7825 \pm 0.0092$ & $0.5828 \pm 0.0170$ & \cellcolor{red!30} $0.8589 \pm 0.0027$ & \cellcolor{red!30} $0.4012 \pm 0.0102$ & \cellcolor{red!30} $0.8667 \pm 0.0088$ & \cellcolor{red!30} $0.3794 \pm 0.0209$ \\ \hline 
\end{tabular}
\label{tab:accCompTable_fashion_s_2_alpha_0.3}
\end{table*}
\begin{table*}
\centering
\caption{Test Performance Comparisons on \textbf{MNIST} with $s=2,\aleph=0.3$ (\colorbox{green!30}{best}, \colorbox{red!30}{second best})}
\begin{tabular}{|C{0.8cm}|C{2cm}| C{1.84cm}|C{1.84cm}|C{1.84cm}|C{1.84cm}|C{1.84cm}| C{1.84cm}|} \hline
\multirow{2}{*}{\rs \textbf{Type}} &\multirow{2}{*}{\rs \textbf{Algorithms}} & \multicolumn{2}{c|}{\sqz, $\mathrm{t_{th}}=60$s } & \multicolumn{2}{c|}{\ac{cnn}, $\mathrm{t_{th}}=60$s} & \multicolumn{2}{c|}{\ac{resnet18}, $\mathrm{t_{th}}=180$s } \\ \cline{3-8}
& & \textbf{Test Acc.} $\uparrow$ & \textbf{Test Loss} $\downarrow$ & \textbf{Test Acc.} $\uparrow$ & \textbf{Test Loss} $\downarrow$ & \textbf{Test Acc.} $\uparrow$ & \textbf{Test Loss} $\downarrow$ \\ \hline
\rowcolor{blue!15} \rs \textbf{Genie} & 
SGD & $0.9864 \pm 0.0003$ & $0.0428 \pm 0.0022$ & $0.9787 \pm 0.0011$ & $0.0673 \pm 0.0052$ & $0.9918 \pm 0.0008$ & $0.0254 \pm 0.0008$ \\ \hline
\multirow{5}{*}{\makecell{\rotatebox{90}{\textbf{Federated}}}} & 
\textbf{OSA-FL (Ours)} & \cellcolor{green!30} $0.9708 \pm 0.0013$ & \cellcolor{green!30} $0.0954 \pm 0.0033$ & \cellcolor{green!30} $0.9744 \pm 0.0006$ & \cellcolor{green!30} $0.0846 \pm 0.0036$ & \cellcolor{green!30} $0.9822 \pm 0.0011$ & \cellcolor{green!30} $0.0559 \pm 0.0028$ \\ \cline{2-8}
& M-FedAvg  &  $0.9610 \pm 0.0012$ & $0.1272 \pm 0.0058$ &\cellcolor{red!30} $0.9728 \pm 0.0006$ & \cellcolor{red!30} $0.0912 \pm 0.0027$ & $0.9793 \pm 0.0020$ & $0.0656 \pm 0.0043$ \\ \cline{2-8} 
& M-FedNova & $0.9606 \pm 0.0024$ & $0.1276 \pm 0.0075$ & $0.9717 \pm 0.0011$ & $0.0915 \pm 0.0003$ & $0.9789 \pm 0.0018$ & $0.0659 \pm 0.0070$ \\ \cline{2-8}
& M-FedProx & $0.9605 \pm 0.0015$ & $0.1273 \pm 0.0056$ & $0.9715 \pm 0.0013$ & $0.0915 \pm 0.0023$ & $0.9790 \pm 0.0021$ & $0.0657 \pm 0.0059$ \\ \cline{2-8}
& M-FedDisco & \cellcolor{red!30} $0.9614 \pm 0.0010$ & \cellcolor{red!30} $0.1236 \pm 0.0039$ & $0.9681 \pm 0.0024$ & $0.1069 \pm 0.0061$ & \cellcolor{red!30} $0.9802 \pm 0.0028$ & \cellcolor{red!30} $0.0639 \pm 0.0081$ \\ \hline 
\end{tabular}
\label{tab:accCompTable_Mnist_s_2_alpha_0.3}
\end{table*}
\begin{table*}
\centering
\caption{Test Performance Comparisons on \textbf{CIFAR10} with $s=4,\aleph=0.3$ (\colorbox{green!30}{best}, \colorbox{red!30}{second best})}
\begin{tabular}{|C{0.8cm}|C{2cm}| C{1.84cm}|C{1.84cm}|C{1.84cm}|C{1.84cm}|C{1.84cm}| C{1.84cm}|} \hline
\multirow{2}{*}{\rs \textbf{Type}} &\multirow{2}{*}{\rs \textbf{Algorithms}} & \multicolumn{2}{c|}{\sqz, $\mathrm{t_{th}}=60$s } & \multicolumn{2}{c|}{\ac{cnn}, $\mathrm{t_{th}}=60$s} & \multicolumn{2}{c|}{\ac{resnet18}, $\mathrm{t_{th}}=180$s } \\ \cline{3-8}
& & \textbf{Test Acc.} $\uparrow$ & \textbf{Test Loss} $\downarrow$ & \textbf{Test Acc.} $\uparrow$ & \textbf{Test Loss} $\downarrow$ & \textbf{Test Acc.} $\uparrow$ & \textbf{Test Loss} $\downarrow$ \\ \hline
\rowcolor{blue!15} \rs \textbf{Genie} & 
SGD & $0.5658 \pm 0.0045$ & $1.2027 \pm 0.0184$ & $0.5934 \pm 0.0029$ & $1.2185 \pm 0.0157$ & $0.6393 \pm 0.0276$ & $1.3572 \pm 0.1063$  \\ \hline
\multirow{5}{*}{\makecell{\rotatebox{90}{\textbf{Federated}}}} & 
\textbf{OSA-FL (Ours)} & \cellcolor{green!30} $0.4135 \pm 0.0078$ & \cellcolor{green!30} $1.5626 \pm 0.0103$ &\cellcolor{red!30} $0.5309 \pm 0.0098$ & $1.3960 \pm 0.0333$ & \cellcolor{green!30} $0.5029 \pm 0.0065$ & $1.4489 \pm 0.0336$ \\ \cline{2-8}
& M-FedAvg  & $0.3624 \pm 0.0080$ & $1.7056 \pm 0.0159$ & $0.5297 \pm 0.0081$ & \cellcolor{red!30} $1.3753 \pm 0.0339$ & \cellcolor{red!30} $0.4919 \pm 0.0082$ & \cellcolor{green!30} $1.4255 \pm 0.0134$ \\ \cline{2-8} 
& M-FedNova & $0.3678 \pm 0.0127$ & $1.6987 \pm 0.0212$ & $0.5231 \pm 0.0061$ & $1.3932 \pm 0.0247$ & $0.4742 \pm 0.0057$ & $1.4479 \pm 0.0119$ \\ \cline{2-8}
& M-FedProx & $0.3679 \pm 0.0125$ & $1.6988 \pm 0.0228$ & $0.5246 \pm 0.0093$ & $1.4096 \pm 0.0392$ & $0.4870 \pm 0.0107$ & \cellcolor{red!30} $1.4284 \pm 0.0182$ \\ \cline{2-8}
& M-FedDisco & \cellcolor{red!30} $0.3752 \pm 0.0082$ & \cellcolor{red!30} $1.6887 \pm 0.0191$ & \cellcolor{green!30} $0.5328 \pm 0.0051$ & \cellcolor{green!30} $1.3742 \pm 0.0229$ & $0.4836 \pm 0.0063$ & $1.4319 \pm 0.0086$ \\ \hline 
\end{tabular}
\label{tab:accCompTable_cifar_s_4_alpha_0.3}
\end{table*}
\begin{table*}
\centering
\caption{Test Performance Comparisons on \textbf{CIFAR10} with $s=8,\aleph=0.3$ (\colorbox{green!30}{best}, \colorbox{red!30}{second best})}
\begin{tabular}{|C{0.8cm}|C{2cm}| C{1.84cm}|C{1.84cm}|C{1.84cm}|C{1.84cm}|C{1.84cm}| C{1.84cm}|} \hline
\multirow{2}{*}{\rs \textbf{Type}} &\multirow{2}{*}{\rs \textbf{Algorithms}} & \multicolumn{2}{c|}{\sqz, $\mathrm{t_{th}}=60$s } & \multicolumn{2}{c|}{\ac{cnn}, $\mathrm{t_{th}}=60$s} & \multicolumn{2}{c|}{\ac{resnet18}, $\mathrm{t_{th}}=180$s } \\ \cline{3-8}
& & \textbf{Test Acc.} $\uparrow$ & \textbf{Test Loss} $\downarrow$ & \textbf{Test Acc.} $\uparrow$ & \textbf{Test Loss} $\downarrow$ & \textbf{Test Acc.} $\uparrow$ & \textbf{Test Loss} $\downarrow$ \\ \hline
\rowcolor{blue!15} \rs \textbf{Genie} & 
SGD & $0.5658 \pm 0.0045$ & $1.2027 \pm 0.0184$ & $0.5934 \pm 0.0029$ & $1.2185 \pm 0.0157$ & $0.6393 \pm 0.0276$ & $1.3572 \pm 0.1063$ \\ \hline
\multirow{5}{*}{\makecell{\rotatebox{90}{\textbf{Federated}}}} & 
\textbf{OSA-FL (Ours)} & \cellcolor{green!30} $0.4135 \pm 0.0137$ & \cellcolor{green!30} $1.5635 \pm 0.0271$ & \cellcolor{red!30} $0.5380 \pm 0.0037$ & $1.3612 \pm 0.0195$ & \cellcolor{green!30} $0.5041 \pm 0.0030$ & $1.4427 \pm 0.0270$ \\ \cline{2-8}
& M-FedAvg & $0.3612 \pm 0.0110$ & $1.7105 \pm 0.0217$ & $0.5359 \pm 0.0024$ & \cellcolor{red!30} $1.3427 \pm 0.0146$ & $0.4853 \pm 0.0085$ & \cellcolor{green!30} $1.4337 \pm 0.0140$ \\ \cline{2-8} 
& M-FedNova & $0.3672 \pm 0.0149$ & $1.7029 \pm 0.0193$ & $0.5354 \pm 0.0045$ & $1.3612 \pm 0.0259$ & $0.4630 \pm 0.0088$ & $1.4833 \pm 0.0189$ \\ \cline{2-8}
& M-FedProx & $0.3651 \pm 0.0141$ & $1.7010 \pm 0.0214$ & $0.5303 \pm 0.0039$ & $1.3692 \pm 0.0165$ & $0.4840 \pm 0.0073$ & $1.4356 \pm 0.0152$ \\ \cline{2-8}
& M-FedDisco & \cellcolor{red!30} $0.3773 \pm 0.0146$ & \cellcolor{red!30} $1.6872 \pm 0.0234$ & \cellcolor{green!30} $0.5406 \pm 0.0011$ & \cellcolor{green!30} $1.3432 \pm 0.0137$ & \cellcolor{red!30} $0.4873 \pm 0.0063$ & \cellcolor{red!30} $1.4349 \pm 0.0115$ \\ \hline 
\end{tabular}
\label{tab:accCompTable_cifar_s_8_alpha_0.3}
\end{table*}
\begin{table*}
\centering
\caption{Test Performance Comparisons on \textbf{CIFAR10} with $s=16,\aleph=0.3$ (\colorbox{green!30}{best}, \colorbox{red!30}{second best})}
\begin{tabular}{|C{0.8cm}|C{2cm}| C{1.84cm}|C{1.84cm}|C{1.84cm}|C{1.84cm}|C{1.84cm}| C{1.84cm}|} \hline
\multirow{2}{*}{\rs \textbf{Type}} &\multirow{2}{*}{\rs \textbf{Algorithms}} & \multicolumn{2}{c|}{\sqz, $\mathrm{t_{th}}=60$s } & \multicolumn{2}{c|}{\ac{cnn}, $\mathrm{t_{th}}=60$s} & \multicolumn{2}{c|}{\ac{resnet18}, $\mathrm{t_{th}}=180$s } \\ \cline{3-8}
& & \textbf{Test Acc.} $\uparrow$ & \textbf{Test Loss} $\downarrow$ & \textbf{Test Acc.} $\uparrow$ & \textbf{Test Loss} $\downarrow$ & \textbf{Test Acc.} $\uparrow$ & \textbf{Test Loss} $\downarrow$ \\ \hline
\rowcolor{blue!15} \rs \textbf{Genie} & 
SGD & $0.5658 \pm 0.0045$ & $1.2027 \pm 0.0184$ & $0.5934 \pm 0.0029$ & $1.2185 \pm 0.0157$ & $0.6393 \pm 0.0276$ & $1.3572 \pm 0.1063$ \\ \hline
\multirow{5}{*}{\makecell{\rotatebox{90}{\textbf{Federated}}}} & 
\textbf{OSA-FL (Ours)} & \cellcolor{green!30} $0.4142 \pm 0.0038$ & \cellcolor{green!30} $1.5638 \pm 0.0158$ & \cellcolor{green!30} $0.5459 \pm 0.0050$ & $1.3378 \pm 0.0201$ & \cellcolor{green!30} $0.4844 \pm 0.0045$ & $1.4673 \pm 0.0191$ \\ \cline{2-8}
& M-FedAvg & $0.3651 \pm 0.0027$ & $1.7055 \pm 0.0232$ & $0.5400 \pm 0.0039$ & \cellcolor{red!30} $1.3372 \pm 0.0181$ & $0.4771 \pm 0.0097$ & $1.4582 \pm 0.0190$ \\ \cline{2-8} 
& M-FedNova & $0.3688 \pm 0.0045$ & $1.7006 \pm 0.0210$ & \cellcolor{red!30} $0.5428 \pm 0.0037$ & $1.3426 \pm 0.0187$ & $0.4474 \pm 0.0177$ & $1.5221 \pm 0.0366$ \\ \cline{2-8}
& M-FedProx & $0.3723 \pm 0.0041$ & $1.6994 \pm 0.0217$ & $0.5406 \pm 0.0047$ & $1.3493 \pm 0.0169$ & $0.4772 \pm 0.0095$ & \cellcolor{red!30} $1.4565 \pm 0.0233$ \\ \cline{2-8}
& M-FedDisco & \cellcolor{red!30} $0.3774 \pm 0.0087$ & \cellcolor{red!30} $1.6941 \pm 0.0228$ & $0.5409 \pm 0.0045$ & \cellcolor{green!30} $1.3356 \pm 0.0192$ & \cellcolor{red!30} $0.4825 \pm 0.0043$ & \cellcolor{green!30} $1.4436 \pm 0.0169$ \\ \hline
\end{tabular}
\label{tab:accCompTable_cifar_s_16_alpha_0.3}
\end{table*}
\begin{table*}
\centering
\caption{Test Performance Comparisons on \textbf{CIFAR10} with $s=2,\aleph=0.1$ (\colorbox{green!30}{best}, \colorbox{red!30}{second best})}
\begin{tabular}{|C{0.8cm}|C{2cm}| C{1.84cm}|C{1.84cm}|C{1.84cm}|C{1.84cm}|C{1.84cm}| C{1.84cm}|} \hline
\multirow{2}{*}{\rs \textbf{Type}} &\multirow{2}{*}{\rs \textbf{Algorithms}} & \multicolumn{2}{c|}{\sqz, $\mathrm{t_{th}}=60$s } & \multicolumn{2}{c|}{\ac{cnn}, $\mathrm{t_{th}}=60$s} & \multicolumn{2}{c|}{\ac{resnet18}, $\mathrm{t_{th}}=180$s } \\ \cline{3-8} 
& & \textbf{Test Acc.} $\uparrow$ & \textbf{Test Loss} $\downarrow$ & \textbf{Test Acc.} $\uparrow$ & \textbf{Test Loss} $\downarrow$ & \textbf{Test Acc.} $\uparrow$ & \textbf{Test Loss} $\downarrow$ \\ \hline
\rowcolor{blue!15} \rs \textbf{Genie} & 
SGD & $0.5620 \pm 0.0159$ & $1.2237 \pm 0.0488$ & $0.5856 \pm 0.0058$ & $1.2875 \pm 0.0081$ & $0.6297 \pm 0.0181$ & $1.4007 \pm 0.0645$ \\ \hline
\multirow{5}{*}{\makecell{\rotatebox{90}{\textbf{Federated}}}} & 
\textbf{OSA-FL (Ours)} & \cellcolor{green!30} $0.3603 \pm 0.0173$ & \cellcolor{green!30} $1.7102 \pm 0.0312$ & \cellcolor{red!30} $0.4927 \pm 0.0134$ & \cellcolor{red!30} $1.4660 \pm 0.0670$ & \cellcolor{green!30} $0.4737 \pm 0.0058$ & \cellcolor{green!30} $1.4824 \pm 0.0200$ \\ \cline{2-8}
& M-FedAvg & $0.3200 \pm 0.0100$ & $1.8257 \pm 0.0133$ & \cellcolor{green!30} $0.4936 \pm 0.0122$ & \cellcolor{green!30} $1.4500 \pm 0.0598$ & $0.4347 \pm 0.0083$ & $1.5515 \pm 0.0199$ \\ \cline{2-8} 
& M-FedNova & $0.3144 \pm 0.0140$ & $1.8296 \pm 0.0117$ & $0.4924 \pm 0.0148$ & $1.4724 \pm 0.0693$ & $0.4281 \pm 0.0058$ & $1.5764 \pm 0.0107$ \\ \cline{2-8}
& M-FedProx & $0.3174 \pm 0.0123$ & $1.8140 \pm 0.0034$ & $0.4875 \pm 0.0147$ & $1.4860 \pm 0.0763$ & $0.4337 \pm 0.0081$ & $1.5533 \pm 0.0183$ \\ \cline{2-8}
& M-FedDisco & \cellcolor{red!30} $0.3268 \pm 0.0107$ & \cellcolor{red!30} $1.8028 \pm 0.0332$ & $0.4868 \pm 0.0112$ & $1.4913 \pm 0.0707$ & \cellcolor{red!30} $0.4388 \pm 0.0066$ & \cellcolor{red!30} $1.5470 \pm 0.0290$ \\ \hline 
\end{tabular}
\label{tab:accCompTable_cifar_s_2_alpha_0.1}
\end{table*}
\begin{table*}
\centering
\caption{Test Performance Comparisons on \textbf{FashionMNIST} with $s=2,\aleph=0.1$ (\colorbox{green!30}{best}, \colorbox{red!30}{second best})}
\begin{tabular}{|C{0.8cm}|C{2cm}| C{1.84cm}|C{1.84cm}|C{1.84cm}|C{1.84cm}|C{1.84cm}| C{1.84cm}|} \hline
\multirow{2}{*}{\rs \textbf{Type}} &\multirow{2}{*}{\rs \textbf{Algorithms}} & \multicolumn{2}{c|}{\sqz, $\mathrm{t_{th}}=60$s } & \multicolumn{2}{c|}{\ac{cnn}, $\mathrm{t_{th}}=60$s} & \multicolumn{2}{c|}{\ac{resnet18}, $\mathrm{t_{th}}=180$s } \\ \cline{3-8}
& & \textbf{Test Acc.} $\uparrow$ & \textbf{Test Loss} $\downarrow$ & \textbf{Test Acc.} $\uparrow$ & \textbf{Test Loss} $\downarrow$ & \textbf{Test Acc.} $\uparrow$ & \textbf{Test Loss} $\downarrow$ \\ \hline
\rowcolor{blue!15} \rs \textbf{Genie} & 
SGD & $0.8647 \pm 0.0109$ & $0.3671 \pm 0.0242$ & $0.8779 \pm 0.0035$ & $0.3451 \pm 0.0100$ & $0.9046 \pm 0.0007$ & $0.3298 \pm 0.0057$ \\ \hline
\multirow{5}{*}{\makecell{\rotatebox{90}{\textbf{Federated}}}} & 
\textbf{OSA-FL (Ours)} & \cellcolor{green!30} $0.7718 \pm 0.0187$ & \cellcolor{green!30} $0.6145 \pm 0.0510$ & \cellcolor{green!30} $0.8438 \pm 0.0108$ & \cellcolor{green!30} $0.4358 \pm 0.0258$ & \cellcolor{green!30} $0.8440 \pm 0.0089$ & \cellcolor{green!30} $0.4322 \pm 0.0212$ \\ \cline{2-8}
& M-FedAvg & $0.7085 \pm 0.0194$ & $0.7636 \pm 0.0372$ & $0.8315 \pm 0.0102$ & $0.4557 \pm 0.0235$ & $0.8334 \pm 0.0102$ & $0.4626 \pm 0.0204$ \\ \cline{2-8} 
& M-FedNova & $0.7195 \pm 0.0135$ & $0.7570 \pm 0.0217$ & \cellcolor{red!30} $0.8372 \pm 0.0082$ & \cellcolor{red!30} $0.4439 \pm 0.0169$ & $0.8304 \pm 0.0098$ & $0.4702 \pm 0.0182$ \\ \cline{2-8}
& M-FedProx & $0.7242 \pm 0.0202$ & $0.7437 \pm 0.0512$ & $0.8361 \pm 0.0091$ & $0.4477 \pm 0.0188$ & $0.8336 \pm 0.0092$ & $0.4605 \pm 0.0178$ \\ \cline{2-8}
& M-FedDisco & \cellcolor{red!30} $0.7350 \pm 0.0116$ & \cellcolor{red!30} $0.7210 \pm 0.0472$
& $0.8358 \pm 0.0091$ & $0.4488 \pm 0.0180$ & \cellcolor{red!30} $0.8364 \pm 0.0098$ & \cellcolor{red!30} $0.4565 \pm 0.0240$ \\ \hline 
\end{tabular}
\label{tab:accCompTable_fashion_s_2_alpha_0.1}
\end{table*}
\begin{table*}
\centering
\caption{Test Performance Comparisons on \textbf{MNIST} with $s=2,\aleph=0.1$ (\colorbox{green!30}{best}, \colorbox{red!30}{second best})}
\begin{tabular}{|C{0.8cm}|C{2cm}| C{1.84cm}|C{1.84cm}|C{1.84cm}|C{1.84cm}|C{1.84cm}| C{1.84cm}|} \hline
\multirow{2}{*}{\rs \textbf{Type}} &\multirow{2}{*}{\rs \textbf{Algorithms}} & \multicolumn{2}{c|}{\sqz, $\mathrm{t_{th}}=60$s } & \multicolumn{2}{c|}{\ac{cnn}, $\mathrm{t_{th}}=60$s} & \multicolumn{2}{c|}{\ac{resnet18}, $\mathrm{t_{th}}=180$s } \\ \cline{3-8}
& & \textbf{Test Acc.} $\uparrow$ & \textbf{Test Loss} $\downarrow$ & \textbf{Test Acc.} $\uparrow$ & \textbf{Test Loss} $\downarrow$ & \textbf{Test Acc.} $\uparrow$ & \textbf{Test Loss} $\downarrow$ \\ \hline
\rowcolor{blue!15} \rs \textbf{Genie} & 
SGD & $0.9467 \pm 0.0548$ & $0.1362 \pm 0.1295$ & $0.9776 \pm 0.0007$ & $0.0703 \pm 0.0012$ & $0.9897 \pm 0.0003$ & $0.0302 \pm 0.0005$ \\ \hline
\multirow{5}{*}{\makecell{\rotatebox{90}{\textbf{Federated}}}} & 
\textbf{OSA-FL (Ours)} & \cellcolor{green!30} $0.9507 \pm 0.0062$ & \cellcolor{green!30} $0.1615 \pm 0.0222$ & \cellcolor{green!30} $0.9672 \pm 0.0010$ & \cellcolor{green!30} $0.1051 \pm 0.0048$ & \cellcolor{green!30} $0.9733 \pm 0.0020$ & \cellcolor{green!30} $0.0872 \pm 0.0093$ \\ \cline{2-8}
& M-FedAvg & $0.9322 \pm 0.0035$ & $0.2515 \pm 0.0106$ & $0.9626 \pm 0.0014$ & $0.1194 \pm 0.0072$ & $0.9675 \pm 0.0051$ & $0.1082 \pm 0.0185$ \\ \cline{2-8} 
& M-FedNova & $0.9321 \pm 0.0095$ & $0.2814 \pm 0.0604$ & \cellcolor{red!30} $0.9632 \pm 0.0021$ & $0.1202 \pm 0.0043$ & $0.9621 \pm 0.0096$ & $0.1239 \pm 0.0289$ \\ \cline{2-8}
& M-FedProx & $0.9300 \pm 0.0047$ & $0.2508 \pm 0.0118$ & $0.9628 \pm 0.0013$ & \cellcolor{red!30} $0.1175 \pm 0.0072$ & $0.9655 \pm 0.0074$ & $0.1113 \pm 0.0222$ \\ \cline{2-8}
& M-FedDisco & \cellcolor{red!30} $0.9406 \pm 0.0018$ & \cellcolor{red!30} $0.2181 \pm 0.0161$ & $0.9569 \pm 0.0051$ & $0.1405 \pm 0.0116$ & \cellcolor{red!30} $0.9686 \pm 0.0052$ & \cellcolor{red!30} $0.1014 \pm 0.0161$ \\ \hline 
\end{tabular}
\label{tab:accCompTable_mnist_s_2_alpha_0.1}
\end{table*}

\subsubsection{System Configurations}
Our system configurations for the resource-constrained clients and wireless network are given below. 
A total of $U=25$ clients are distributed uniformly randomly into the coverage area of a single \ac{bs}, while the \ac{cs} is embedded into the \ac{bs}. 
We assume the \ac{bs} operates in the $2.4$ GHz band and can serve each user with a bandwidth of $\omega = 540$ kHz.
The path loss model is adopted from \cite{pervej2025resource}. 
We stress that small-scale fading was not considered, assuming that modern networks have sufficient diversity to mitigate it.
Besides, $\epsilon=0.5$, $\rho = 2 \times 10^{-28}$, $\kappa=5$, $n=8$, $\bar{n}=16$, $c_u \in [25, 40]$, $\mathrm{e}_{u,\mathrm{bd}} \in [1.2, 2.5]$, $f_{u,\mathrm{max}} \in [1,1.8]$ GHz, $ p_{u,\mathrm{max}} \in [20,30] $ dBm, respectively.
The data size per sample, i.e., $s_u$, is calculated as ${\tt no\_ch} \times 128 \times 128 \times 32$, where ${\tt no\_ch}$ is the number of channels, $128$ is the height/width, and $32$ is for single-bit precision. 
Furthermore, we use $s=\{2,4,8,16\}$ quantization levels for the stochastic quantizer. 
The deadline $t_\mathrm{th}$ between two consecutive \ac{fl} rounds varies for different models, as they have different payload sizes.
For \sqz~ and \ac{cnn}, we considered $t_\mathrm{th} \in \{45, 60, 75, 90, 105\}$ seconds, while for the \ac{resnet18} model $t_\mathrm{th} \in \{102, 180, 240, 300, 360\}$ seconds.
We stress that we only vary the deadline to show how resource constraints affect performances of different \ac{fl} algorithms and whether increasing the deadline can help combat these constraints.

We use the symmetric Dirichlet distribution with concentration parameter $\aleph \in \{\pmb{0.1}$, $\pmb{0.3}$, $\pmb{0.9}\}$ to distribute the datasets across $U$ clients \cite{pervej2023resource}. 
For the continual data arrival process, we draw the client's request probabilities uniformly at random from $p_u \in [0.3, 0.8]$.
At every \ac{fl} round, $E_u = \left\lceil 40 \times p_u \right\rceil$ samples can arrive at the $u^\mathrm{th}$ client.
Besides, we assume the maximum number of data samples a client's data storage can hold is $\mathrm{D}_u = \left \lceil 400 \times p_u\right\rceil$.
As such, when a client receives $E_u$ new samples in its buffer during the training round $t$, these samples are stored in its buffer, and $E_u$ old samples from {\tt Top-K} classes are removed after the client completes $\kappa_u^t$ local rounds (during FL round $t$) to make space for the new samples in its buffer\footnote{However, our proposed algorithm can easily be extended to any other data deletion policies.}.
More specifically, we remove $E_u$ samples from the {\tt Top-4, Top-3}, and {\tt Top-2} classes, respectively for $\aleph=0.1$, $\aleph=0.3$, and $\aleph=0.9$, proportional to their label distributions in $\mathcal{D}_u^t$.

\subsubsection{Model Training Hyper-parameters}
Analogous to our theoretical analysis, we consider the \ac{sgd} optimizer for all algorithms. 
For the proposed \ac{osafl} algorithm, we set $\mathfrak{v}=10$, $\chi=1$, and $\varsigma=0.75$ to calculate the online score based on (\ref{eq:sub_optim_score}).
Besides, we set $\mathfrak{a}$ to $0.3$, $0.3$, and $0.5$ for the \sqz, \ac{cnn}, and \ac{resnet18} models, respectively.
The $\eta_\mathrm{sch}^t$ is decayed linearly from $1$ at round $t=0$ to $0.65$ at round $t=T-1$.

\subsection{Performance Analysis for Different Parameters}

\subsubsection{Impact of Deadline and Quantization Level on Local Round}
Both the deadline between two consecutive \ac{fl} rounds, $\mathrm{t_{th}}$, and the quantization level $s$ affect client participation.
Based on the optimization problem in (\ref{localIterOptim_Orig}) and the model payload size defined in (\ref{eq:model_payload}), it is quite evident that when the deadline $\mathrm{t_{th}}$ and other parameters are fixed, we need more bits to transmit the quantized accumulated gradient as the quantization level increases. 
As such, increasing $s$ would mean there is less energy and time left, if any, for local model training in order to satisfy the energy budget and deadline threshold constraints in (\ref{eq:energy_cons}) and (\ref{eq:deadline_cons}), respectively.
Therefore, given that the optimization problem is feasible with fixed $\mathrm{t_{th}}$ and other system parameters, clients are expected to perform fewer local rounds. 
Similarly, when $s$ is fixed, increasing $\mathrm{t_{th}}$ gives clients more time for local training, yielding higher values of $\kappa_u^t$.
Nonetheless, under resource constraints and poor channel conditions, some clients may not participate in model training, as the optimization problem (\ref{localIterOptim_Orig}) may be infeasible.

Our simulation results in Fig. \ref{fig:cdfLocalItrVsQuant} also show similar trends. 
More specifically, we observe that for a fixed deadline $\mathrm{t_{th}}$, the \ac{ccdf} decreases if $s$ increases.
For example, when $\mathrm{t_{th}}=60$ seconds, the probability of having local iterations larger than $4$ is around $67.09\%$ and $55.25\%$ with \sqz, and around $60.48\%$ and $48.16\%$ with \ac{cnn} when $s=2$ and $s=16$, respectively.
For \ac{resnet18}, these values are around $87.71\%$ and $56.53\%$, respectively, when $\mathrm{t_{th}}=180$ seconds.
It is also evident that for shallower \sqz~ and moderate \ac{cnn} models, increasing the deadlines increases the probability of having more local training rounds. In contrast, the probabilities do not vary significantly for the bulky \ac{resnet18} model.
This means that some clients cannot participate in model training due to extreme resource constraints when using bulky models.

\subsubsection{Impact of Model Choices on Test Accuracy and Wall Clock}
Since the wireless payload size directly affects whether clients can participate in the model training, we expect these different models to converge at different speeds.
While shallower models may enable more local rounds between consecutive FL rounds, they may not capture changes in the data distribution efficiently. 
Moreover, the use of bulky models, which often result in many stragglers, may not deliver the expected performance within a limited time. 
Besides, having a larger deadline means the wall clock to finish $T$ FL rounds will increase.
Fig. \ref{fig:wallClock_Vs_Acc} shows how the test accuracies of the three models vary across the three datasets considered. 
Across all datasets, we observe that the shallower \sqz~ and \ac{cnn} models achieve better performance faster than the bulky \ac{resnet18} model.
However, the shallowest \sqz~ model does not achieve comparable test accuracies after completing $T=50$ \ac{fl} rounds on CIFAR10 and FashionMNIST datasets. 
This is due to the fact that this client's model not only has fewer training parameters but also undergoes quantization of accumulated gradients due to resource constraints, which introduces additional noise, as shown in our convergence bound in (\ref{convRate_Eqn}).

\subsubsection{Impact of Quantization on Test Accuracy}
In this case, since the actual accumulated normalized gradients are not shared and our theoretical convergence bound in (\ref{convRate_Eqn}) shows that additional errors due to quantization, we expect slightly degraded performance across different quantization levels.
Now, while it is intuitive to expect that the error due to quantization generally increases as the number of quantization levels decreases (more extreme quantization), this intuition is indeed tricky in a resource-constrained wireless environment, given obvious system constraints and time-varying wireless links.
More specifically, decreasing the quantization levels results in a smaller wireless payload for clients to offload, as shown in (\ref{eq:model_payload}).
Thus, clients can allocate more resources to local computation, as shown in Fig. \ref{fig:cdfLocalItrVsQuant}.
Therefore, under resource constraints and time-varying wireless links, different quantization levels may not have a significant impact on shallower models, such as the \sqz.
This can also be observed in our simulation results in Fig. \ref{fig:quantLevel_vs_acc_sqz}.

However, if the model parameters are reasonable and changes in the quantization levels do not significantly alter the number of local rounds $\kappa_u^t$, then having more quantization levels should reduce quantization errors, thereby improving test accuracy.   
We also validate this intuition from the results in Figs. \ref{fig:cdfLocalItrVsQuant_cnn_s2}, \ref{fig:cdfLocalItrVsQuant_cnn_s16}, and \ref{fig:quantLevel_vs_acc_cnn}.
More specifically, we notice that with $\mathrm{t_{th}}$, the \ac{ccdf} does not change drastically when $s$ increases from $2$ to $16$. 
As such, increasing $s$ shall improve test accuracy, as shown in Fig. \ref{fig:quantLevel_vs_acc_cnn}.

Lastly, at the other extreme, when the model is bulky, such as the \ac{resnet18} model, increasing the quantization level would drastically increase clients' payload sizes, thereby reducing the number of local training rounds and increasing the number of stragglers; this is shown in Figs. \ref{fig:cdfLocalItrVsQuant_resnet_s2} and \ref{fig:cdfLocalItrVsQuant_resnet_s16}. 
As such, increasing $s$ while keeping the other parameters fixed will actually reduce test accuracy. 
This behavior is observed in Fig. \ref{fig:quantLevel_vs_acc_resnet18}.

\subsection{Performance Comparisons with Baselines}

\noindent
Since, to the best of our knowledge, there are no existing baselines that are exactly similar to our system model and assumptions, we use four existing and popular \ac{fl} algorithms, namely, $(1)$ \ac{fedavg} \cite{mcmahan17Communication}, $(2)$ FedProx \cite{li2020federated}, $(3)$ \ac{fednova} \cite{wang2020Tackling}, and $(4)$ \ac{feddisco} \cite{ye23fedDisco}, and modify those algorithms to fit our system model and assumption.
We call these baselines \ac{mfedavg}, \ac{mfedprox}, \ac{mfednova}, and \ac{mfeddisco}, respectively.
Besides, we stress that all of these modified algorithms are accompanied by the same resource optimization, as presented in Section \ref{resourceOptim}, and system configurations as in our proposed \ac{osafl} algorithm, for \emph{apple-to-apple} comparisons.
Moreover, we also implement a \emph{Genie}-aided centralized \ac{sgd} baseline, assuming that the training samples available to clients in each global round are centrally accessible to the Genie.
Finally, we used ablation studies to find the best hyperparameters for all baselines across all considered datasets and \ac{ml} models.

\subsubsection{Performance Comparisons for Fixed Quantization Level, Deadline, and Data Heterogeneity}
In practical resource-constrained environments, purely aggregating the clients' models based on dataset sizes, as in \ac{fedavg} and {\tt FedProx}, or aggregating based on normalized weights that only consider local training steps and dataset sizes, as in \ac{fednova}, or aggregating models based on discrepancy-aware weighting, as in \ac{feddisco}, may not be sufficient. 
With these naive extended baselines and our proposed joint resource optimization technique described in Section \ref{resourceOptim}, we expect poor convergence performances since these algorithms were not designed to handle our dynamic environment and data distribution shift.
Recalling the key findings from our theoretical analysis in Theorem \ref{convgRate}, we note that the continual arrival of local training data samples leads to non-IID distributions and time-varying datasets within a client, having different sample arrival probabilities with different label preferences (coming from the Dirichlet distribution for the labels), quantization errors, heterogeneous local rounds, and uncertain client participation probabilities need to be jointly considered to design the aggregation policy at the \ac{cs}. 
As such, our online-score-aided aggregation policy, although suboptimal, shall help the global model to converge faster.
Finally, since in \ac{fl} we do not have access to the entire dataset from all clients, it is expected that the convergence rate of any \ac{fl} algorithm will be slower than that of centralized \ac{sgd}.

Our simulation results in Fig. \ref{fig:acc_comp_all_data_mod_s_2_alpha_0.3} show the effectiveness of optimizing clients' scores. 
We observe that, with a fixed deadline, a fixed quantization level, and a dataset heterogeneity parameter, the proposed \ac{osafl} algorithm achieves better performance faster than the other modified baselines.
This rate depends on the considered model, dataset, and number of local rounds.
More specifically, with a shallow model like the \sqz, the convergence speed depends on the difficulty of the classification task.
For example, with \sqz, at \ac{fl} round $t=19$, our proposed \ac{osafl} algorithm achieves test accuracies that are $4.52\%$ $4.48\%$ $4.51\%$, and $3.06\%$ higher than those of \ac{mfedavg}, \ac{mfedprox}, \ac{mfednova}, and \ac{mfeddisco} algorithms on the CIFAR10 dataset. 
These differences are $10.86\%$, $10.57\%$, $10.93\%$, and $7.66\%$ on FashionMNIST, and are $10.8\%$, $11.93\%$, $10.91\%$, and $8.65\%$ on MNIST datasets, respectively.
With a relatively moderate model like the \ac{cnn}, these gaps can be smaller.
For example, these gaps are $1.22\%$, $0.35\%$, $0.35\%$, and  $0.66\%$ on CIFAR10; $1.39\%$, $0.95\%$, $0.88\%$, and $1.48\%$ on FashionMNIST; and $0.21\%$, $0.36\%$, $0.22\%$ and $0.46\%$ on MNIST, respectively.
Finally, allowing additional time to run a bulky \ac{resnet18} model may help \ac{fl} algorithms mitigate quantization error and thereby improve performance.
Nonetheless, the performance gaps depend on the complexity of the classification tasks.
Moreover, our simulation results in Fig. \ref{fig:acc_comp_all_data_mod_s_2_alpha_0.3} clearly demonstrate that the proposed \ac{osafl} algorithm converges faster across all considered datasets and models.

Given the above facts, we now report the best test accuracies and corresponding test losses for the baselines across various quantization and dataset heterogeneity parameters.

\subsubsection{Performance Comparisons with Different Quantization Levels}
Recall that increasing the quantization level increases wireless payload sizes, which also impacts clients' participation and their local training rounds.
As such, with a shallower model, we may not observe drastic performance changes when $s$ changes, whereas models like the \ac{cnn} can yield improved performance with larger $s$. 
Lastly, as we discussed above, with the bulky \ac{resnet18} model, there would be many stragglers as $s$ increases, thereby reducing performance in a resource-constrained environment. 
Nonetheless, in all configurations, we expect our proposed solution to provide better performance. 
The simulation results in Tables \ref{tab:accCompTable_cifar_s_2_alpha_0.3}-{\ref{tab:accCompTable_cifar_s_16_alpha_0.3}} validates these claims. 
More specifically, with \sqz, all \ac{fl} algorithms exhibit marginal performance changes as $s$ increases; with \ac{cnn}, performance improves as $s$ increases; and lastly, performance drops as $s$ increases with \ac{resnet18}.

\subsubsection{Performance Comparisons with Different Data Heterogeneity}
Label distributions across clients largely depend on the Dirichlet concentration parameter $\aleph$.
A smaller $\aleph$ means more skewed label distributions, while a larger $\aleph$ means more balanced label distributions. 
This also depends on the time-varying dataset $\mathcal{D}_u^t$ in our setting, given the continual arrival of data.
Nonetheless, with a larger $\aleph$, we have a more balanced data distribution, so the differences among clients' local gradients will not be very severe.
As such, the \ac{fl} algorithms shall have improved performance as $\aleph$ increases. 
This is also clearly observed across all models and datasets in Tables \ref{tab:accCompTable_cifar_s_2_alpha_0.1}-\ref{tab:accCompTable_Mnist_s_2_alpha_0.9}.
Moreover, except for the \ac{cnn} on CIFAR10 with $\aleph=0.9$ in Table \ref{tab:accCompTable_cifar_s_2_alpha_0.9}, the proposed \ac{osafl} algorithm achieves higher test accuracies across all considered $\aleph$, datasets, and models.

\begin{table*}
\centering
\caption{Test Performance Comparisons on \textbf{CIFAR10} with $s=2,\aleph=0.9$ (\colorbox{green!30}{best}, \colorbox{red!30}{second best})}
\begin{tabular}{|C{0.8cm}|C{2cm}| C{1.84cm}|C{1.84cm}|C{1.84cm}|C{1.84cm}|C{1.84cm}| C{1.84cm}|} \hline
\multirow{2}{*}{\rs \textbf{Type}} &\multirow{2}{*}{\rs \textbf{Algorithms}} & \multicolumn{2}{c|}{\sqz, $\mathrm{t_{th}}=60$s } & \multicolumn{2}{c|}{\ac{cnn}, $\mathrm{t_{th}}=60$s} & \multicolumn{2}{c|}{\ac{resnet18}, $\mathrm{t_{th}}=180$s } \\ \cline{3-8}
& & \textbf{Test Acc.} $\uparrow$ & \textbf{Test Loss} $\downarrow$ & \textbf{Test Acc.} $\uparrow$ & \textbf{Test Loss} $\downarrow$ & \textbf{Test Acc.} $\uparrow$ & \textbf{Test Loss} $\downarrow$ \\ \hline
\rowcolor{blue!15} \rs \textbf{Genie} & 
SGD & $0.5455 \pm 0.0364$ & $1.2753 \pm 0.1146$ & $0.6008 \pm 0.0059$ & $1.2129 \pm 0.0078$ & $0.6537 \pm 0.0218$ & $1.2872 \pm 0.0875$ \\ \hline
\multirow{5}{*}{\makecell{\rotatebox{90}{\textbf{Federated}}}} & 
\textbf{OSA-FL (Ours)} & \cellcolor{green!30} $0.4385 \pm 0.0060$ & \cellcolor{green!30} $1.5261 \pm 0.0113$ & $0.5253 \pm 0.0020$ & $1.4114 \pm 0.0280$ & \cellcolor{green!30} $0.5349 \pm 0.0057$ & \cellcolor{green!30} $1.3370 \pm 0.0183$ \\ \cline{2-8}
& M-FedAvg  & $0.3877 \pm 0.0012$ & $1.6711 \pm 0.0091$ & \cellcolor{red!30} $0.5323 \pm 0.0051$ & \cellcolor{green!30} $1.3673 \pm 0.0306$ & \cellcolor{red!30} $0.5128 \pm 0.0067$ & \cellcolor{red!30} $1.3670 \pm 0.0126$ \\ \cline{2-8} 
& M-FedNova & $0.3948 \pm 0.0028$ & $1.6478 \pm 0.0158$ & $0.5148 \pm 0.0071$ & $1.4358 \pm 0.0413$ & $0.5046 \pm 0.0030$ & $1.3760 \pm 0.0051$ \\ \cline{2-8}
& M-FedProx & \cellcolor{red!30} $0.3988 \pm 0.0040$ & \cellcolor{red!30} $1.6440 \pm 0.0170$ & $0.5174 \pm 0.0040$ & $1.4406 \pm 0.0335$ & $0.5120 \pm 0.0049$ & $1.3693 \pm 0.0098$ \\ \cline{2-8}
& M-FedDisco & $0.3916 \pm 0.0036$ & $1.6548 \pm 0.0188$ & \cellcolor{green!30} $0.5327 \pm 0.0048$ & \cellcolor{red!30} $1.3792 \pm 0.0330$ & $0.5101 \pm 0.0019$ & $1.3667 \pm 0.0048$ \\ \hline 
\end{tabular}
\label{tab:accCompTable_cifar_s_2_alpha_0.9}
\end{table*}
\begin{table*}
\centering
\caption{Test Performance Comparisons on \textbf{FashionMNIST} with $s=2,\aleph=0.9$ (\colorbox{green!30}{best}, \colorbox{red!30}{second best})}
\begin{tabular}{|C{0.8cm}|C{2cm}| C{1.84cm}|C{1.84cm}|C{1.84cm}|C{1.84cm}|C{1.84cm}| C{1.84cm}|} \hline
\multirow{2}{*}{\rs \textbf{Type}} &\multirow{2}{*}{\rs \textbf{Algorithms}} & \multicolumn{2}{c|}{\sqz, $\mathrm{t_{th}}=60$s } & \multicolumn{2}{c|}{\ac{cnn}, $\mathrm{t_{th}}=60$s} & \multicolumn{2}{c|}{\ac{resnet18}, $\mathrm{t_{th}}=180$s } \\ \cline{3-8}
& & \textbf{Test Acc.} $\uparrow$ & \textbf{Test Loss} $\downarrow$ & \textbf{Test Acc.} $\uparrow$ & \textbf{Test Loss} $\downarrow$ & \textbf{Test Acc.} $\uparrow$ & \textbf{Test Loss} $\downarrow$ \\ \hline
\rowcolor{blue!15} \rs \textbf{Genie} & 
SGD & $0.8747 \pm 0.0092$ & $0.3407 \pm 0.0232$ & $0.8844 \pm 0.0023$ & $0.3242 \pm 0.0046$ & $0.9086 \pm 0.0029$ & $0.3046 \pm 0.0071$ \\ \hline
\multirow{5}{*}{\makecell{\rotatebox{90}{\textbf{Federated}}}} & 
\textbf{OSA-FL (Ours)} & \cellcolor{green!30} $0.8324 \pm 0.0051$ & \cellcolor{green!30} $0.4546 \pm 0.0160$ & \cellcolor{green!30} $0.8687 \pm 0.0011$ & \cellcolor{red!30} $0.3878 \pm 0.0056$ & \cellcolor{green!30} $0.8808 \pm 0.0034$ & \cellcolor{green!30} $0.3519 \pm 0.0111$ \\ \cline{2-8}
& M-FedAvg & \cellcolor{red!30} $0.8106 \pm 0.0058$ & \cellcolor{red!30} $0.5109 \pm 0.0136$ & $0.8633 \pm 0.0015$ & \cellcolor{green!30} $0.3875 \pm 0.0062$ & \cellcolor{red!30} $0.8781 \pm 0.0048$ & \cellcolor{red!30} $0.3544 \pm 0.0126$ \\ \cline{2-8} 
& M-FedNova & $0.8102 \pm 0.0085$ & $0.5136 \pm 0.0167$ & $0.8655 \pm 0.0025$ & $0.3909 \pm 0.0076$ & $0.8766 \pm 0.0035$ & $0.3561 \pm 0.0093$ \\ \cline{2-8}
& M-FedProx & $0.8076 \pm 0.0103$ & $0.5161 \pm 0.0226$ & \cellcolor{red!30} $0.8661 \pm 0.0025$ & $0.3887 \pm 0.0094$ & $0.8776 \pm 0.0041$ & $0.3600 \pm 0.0185$ \\ \cline{2-8}
& M-FedDisco & $0.8037 \pm 0.0045$ & $0.5294 \pm 0.0089$ & $0.8631 \pm 0.0037$ & $0.3891 \pm 0.0106$ & $0.8770 \pm 0.0039$ & $0.3598 \pm 0.0117$
 \\ \hline 
\end{tabular}
\label{tab:accCompTable_fashion_s_2_alpha_0.9}
\end{table*}
\begin{table*}
\centering
\caption{Test Performance Comparisons on \textbf{MNIST} with $s=2,\aleph=0.9$ (\colorbox{green!30}{best}, \colorbox{red!30}{second best})}
\begin{tabular}{|C{0.8cm}|C{2cm}| C{1.84cm}|C{1.84cm}|C{1.84cm}|C{1.84cm}|C{1.84cm}| C{1.84cm}|} \hline
\multirow{2}{*}{\rs \textbf{Type}} &\multirow{2}{*}{\rs \textbf{Algorithms}} & \multicolumn{2}{c|}{\sqz, $\mathrm{t_{th}}=60$s } & \multicolumn{2}{c|}{\ac{cnn}, $\mathrm{t_{th}}=60$s} & \multicolumn{2}{c|}{\ac{resnet18}, $\mathrm{t_{th}}=180$s } \\ \cline{3-8}
& & \textbf{Test Acc.} $\uparrow$ & \textbf{Test Loss} $\downarrow$ & \textbf{Test Acc.} $\uparrow$ & \textbf{Test Loss} $\downarrow$ & \textbf{Test Acc.} $\uparrow$ & \textbf{Test Loss} $\downarrow$ \\ \hline
\rowcolor{blue!15} \rs \textbf{Genie} & 
SGD & $0.9860 \pm 0.0006$ & $0.0428 \pm 0.0034$ & $0.9803 \pm 0.0014$ & $0.0617 \pm 0.0045$ & $0.9915 \pm 0.0008$ & $0.0279 \pm 0.0030$ \\ \hline
\multirow{5}{*}{\makecell{\rotatebox{90}{\textbf{Federated}}}} & 
\textbf{OSA-FL (Ours)} & \cellcolor{green!30} $0.9753 \pm 0.0009$ & \cellcolor{green!30} $0.0765 \pm 0.0015$ & \cellcolor{green!30} $0.9756 \pm 0.0008$ & $0.0841 \pm 0.0026$ & \cellcolor{green!30} $0.9836 \pm 0.0007$ & \cellcolor{green!30} $0.0510 \pm 0.0039$ \\ \cline{2-8}
& M-FedAvg & \cellcolor{red!30} $0.9669 \pm 0.0016$ & \cellcolor{red!30} $0.1043 \pm 0.0079$ & $0.9744 \pm 0.0003$ & $0.0829 \pm 0.0027$ & $0.9816 \pm 0.0016$ & \cellcolor{red!30} $0.0555 \pm 0.0052$ \\ \cline{2-8} 
& M-FedNova & $0.9666 \pm 0.0013$ & $0.1070 \pm 0.0054$ & $0.9749 \pm 0.0002$ & \cellcolor{green!30} $0.0787 \pm 0.0013$ &\cellcolor{red!30} $0.9824 \pm 0.0015$ & $0.0565 \pm 0.0040$ \\ \cline{2-8}
& M-FedProx & $0.9662 \pm 0.0021$ & $0.1071 \pm 0.0087$ & \cellcolor{red!30} $0.9753 \pm 0.0012$ & \cellcolor{red!30} $0.0789 \pm 0.0052$ & $0.9817 \pm 0.0014$ & $0.0563 \pm 0.0052$ \\ \cline{2-8}
& M-FedDisco & $0.9668 \pm 0.0013$ & $0.1052 \pm 0.0051$ & $0.9727 \pm 0.0004$ & $0.0919 \pm 0.0025$ & $0.9816 \pm 0.0016$ & $0.0579 \pm 0.0056$  \\ \hline 
\end{tabular}
\label{tab:accCompTable_Mnist_s_2_alpha_0.9}
\end{table*}

\section{Conclusions}
\label{sec_conclusion}
\noindent
This paper proposed \ac{osafl}, a new algorithm for taming the resource constraints of distributed clients deployed in wireless networks that can continually sense the environment and gather new training data. 
We first theoretically analyzed how data distribution shifts within clients, unknown device participation probabilities due to time-varying wireless channels and resource constraints, gradient quantization errors, and gradient dissimilarities across different clients affect the convergence rate.
We then optimize the client aggregation weights, referred to as the client's score, and a global learning rate scheduler to improve the convergence rate. 
Finally, we validate the effectiveness of the proposed algorithm using various ML models with different trainable parameters across three popular datasets.

\section*{Acknowledgment}
\noindent
The authors gratefully acknowledge the support and resources from the Center for High Performance Computing at the University of Utah and the Center for Advanced Research Computing (CARC) at the University of Southern California.

\bibliographystyle{IEEEtran}
\bibliography{ref.bib}

\newpage 
\clearpage
\appendices 
\onecolumn

\section*{Supplementary Materials}

\section{Important Notations}

\begin{table}[!h]
\centering
\caption{Important Notations in Use}
    \begin{tabular}{|C{2cm}|C{15cm}|}
    \hline
    \textbf{Notation} & \textbf{Description} \\ \hline
      $N$  & Total number of trainable model parameters \\\hline
      $t$ & Global round \\ \hline
      $\tau$ & Local {\tt SGD} step \\ \hline
      $u$, $U$, $\mathcal{U}$ & $u^{th}$ user, total users, set of users \\ \hline
      $\mathbf{w}^t$, $\mathbf{w}_u^t$ & Global model at round $t$, local model of client $u$ at round $t$ \\ \hline
      $\mathcal{D}_u^t$ & Client $u$'s local dataset during global round $t$ \\ \hline
      $\alpha_u^t$ & Client $u$'s aggregation weight during global round $t$ \\ \hline
      $f_u^t(\mathbf{w})$ & Client $u$'s local loss function evaluated on dataset $\mathcal{D}_u^t$ using model $\mathbf{w}$ \\ \hline
      $p_{u,\mathrm{ac}}$ & Client $u$'s new sample request probability \\ \hline
      $E_u$ & Client $u$'s total new samples \\ \hline
      $n$, $\bar{n}$ & Total mini-batches, mini-batch size \\ \hline
      $c_u$ & Client $u$'s number of CPU cycle to compute $1$-bit data \\ \hline
      $s_u$ & Client $u$'s data sample size in bits \\ \hline
      $\bar{f}_u^t$, $\bar{f}_{u,\mathrm{max}}$ & Client $u$'s CPU cycle, client $u$'s maximum CPU cycle \\ \hline
      $\rho$ & Effective capacitance of \ac{cpu} chip \\ \hline
      $\Upsilon$ & Model payload size in bits \\ \hline
      $\omega$ & Bandwidth size \\ \hline 
      $\Xi_u^t$ & Client $u$'s large-scale path loss \\ \hline 
      $\Gamma_u^t$ & Client $u$'s log-Normal shadowing \\ \hline 
      $p_u^t$, $p_{u,\mathrm{max}}$ & Client $u$'s transmission power, client $u$'s maximum transmission power \\ \hline 
      $\xi^2$ & Noise variance \\ \hline
      $\mathrm{t}_{u,\mathrm{cp}}^t$, $\mathrm{t}_{u,\mathrm{up}}^t$ & Client $u$'s local model training time, client $u$'s normalized gradient offload time \\ \hline
      $\mathrm{t_{th}}$ & Deadline until next global round begins  \\ \hline
      $\mathrm{e}_{u,\mathrm{cp}}^t$, $\mathrm{e}_{u,\mathrm{up}}^t$ & Client $u$'s local model training energy expense, client $u$'s normalized gradient offload energy expense \\ \hline
      $\mathrm{e_{th}}$ & Energy budget to perform local training and offloading trained model parameters in each global round  \\ \hline
      $\epsilon$ & Weighting parameter to balance between energy expense for local training and energy efficiency for offloading payload \\ \hline 
      $\kappa_u^t$, $\kappa$ & Client $u$'s local rounds, maximum local {\tt SGD} steps \\ \hline
      $\mathrm{1}_u^t$; $v_u^t$ & Binary indicator function for denoting if client $u$ participates in FL round $t$; success probability of $\mathrm{1}_u^t$ \\ \hline
      
    \end{tabular}
    \label{tab:notations}
\end{table}

\newpage
\section{Proof of Lemma \ref{lemma_CPUfreqGiven_kappa_p}}
\label{LRO_App}

Given the $\kappa_u^{t}$ and the transmission power $p_u^{t,i}$, we optimize the \ac{cpu} frequency by solving 
\begin{subequations}
\label{cpuFreqOptim_Sub_Prob_supp}
\begin{align}
    \underset{ \bar{f}_u^t }  { \tt{maximize} ~~ } & \quad \epsilon \left[\frac{\kappa_u^{t}}{0.5 \rho n \bar{n} c_u s_u \left(\bar{f}_u^t\right)^2} \right] + \left(1-\epsilon\right) \left[ \frac {\omega \log_2 \left( 1 + \frac{\Xi_{u}^t \Gamma_u^t p_u^{t,i}} {\omega \xi^2}\right) } {p_u^{t,i}} \right] \tag{\ref{cpuFreqOptim_Sub_Prob_supp}} \\
    {\tt{subject ~ to}} &\quad 0 \leq \bar{f}_u^t \leq \bar{f}_{u,\mathrm{max}},\\ 
    &\quad 0.5 \rho n \bar{n} c_u s_u \kappa_u^{t} \left(\bar{f}_u^t\right)^2 + \frac{\Upsilon \cdot p_u^{t,i}} {\omega \log_2 \left( 1 + \frac{\Xi_{u}^t \Gamma_u^t p_u^{t,i}}{\omega \xi^2}\right) } \leq \mathrm{e_{bd}}, \\
    &\quad \frac{n\bar{n} c_u s_u \kappa_u^{t} }{\bar{f}_u^t} + \frac{\Upsilon} {\omega \log_2 \left( 1 + \frac{\Xi_{u}^t \Gamma_u^t p_u^{t,i}}{\omega \xi^2}\right)} \leq \mathrm{t_{th}},
\end{align}
\end{subequations}

\begin{theorem-box}[Lemma \ref{lemma_CPUfreqGiven_kappa_p}]
Given $\kappa_u^{t}$ and $p_u^{t,i}$, the optimal solution of (\ref{cpuFreqOptim_Sub_Prob_supp}) is
\begin{align}
\label{optimal_CPU_freq_supp}
    \bar{f}_u^{t^*} = \frac {n\bar{n} c_u s_u \kappa_u^{t} \times \omega \log_2 \left( 1 + \frac{\Xi_{u}^t \Gamma_u^t p_u^{t,i} } {\omega \xi^2}\right)} { \mathrm{t_{th}} \times \omega \log_2 \left( 1 + \frac{\Xi_{u}^t \Gamma_u^t p_u^{t,i}}{\omega \xi^2}\right) - \Upsilon}.
\end{align}
\end{theorem-box}

\begin{proof}
Observe that the objective function is monotonically decreasing with respect to $\bar{f}_u^t$ because 
\begin{align}
    \frac{ \partial \left( \epsilon \left[\frac{\kappa_u^{t}}{0.5 \rho n \bar{n} c_u s_u \left(\bar{f}_u^t\right)^2} \right] + \left(1-\epsilon\right) \left[ \frac {\omega \log_2 \left( 1 + \frac{\Xi_{u}^t \Gamma_u^t p_u^{t,i}} {\omega \xi^2}\right) } {p_u^{t,i}} \right] \right) } { \partial \bar{f}_u^t} = \frac{-4\epsilon \kappa_u^{t}}{\rho n \bar{n} c_u s_u \left(\bar{f}_u^t\right)^3}. 
\end{align}

Then, the energy budget constraint gives the upper bound for the \ac{cpu} frequency as
\begin{align}
    \bar{f}_u^t   
    \leq \sqrt{ \frac{ \mathrm{e_{bd}} - \frac{\Upsilon \cdot p_u^{t,i}} {\omega \log_2 \left( 1 + \frac{\Xi_{u}^t \Gamma_u^t p_u^{t,i}}{\omega \xi^2}\right) } } {0.5 \rho n \bar{n} c_u s_u \kappa_u^{t}} } = \sqrt{ \frac{ \mathrm{e_{bd}} \times \omega \log_2 \left( 1 + \frac{\Xi_{u}^t \Gamma_u^t p_u^{t,i}}{\omega \xi^2}\right)  - \Upsilon \cdot p_u^{t,i} } { 0.5 \rho n \bar{n} c_u s_u \kappa_u^{t} \times \omega \log_2 \left( 1 + \frac{\Xi_{u}^t \Gamma_u^t p_u^{t,i}}{\omega \xi^2}\right) } }.
\end{align}
Besides, the deadline constraint gives the lower bound as 
\begin{align}
    \bar{f}_u^t 
    \geq \frac {n\bar{n} c_u s_u \kappa_u^{t}} { \mathrm{t_{th}} - \frac{\Upsilon} {\omega \log_2 \left( 1 + \frac{\Xi_{u}^t \Gamma_u^t p_u^{t,i}}{\omega \xi^2}\right)} } 
    = \frac {n\bar{n} c_u s_u \kappa_u^{t} \times \omega \log_2 \left( 1 + \frac{\Xi_{u}^t \Gamma_u^t p_u^{t,i}}{\omega \xi^2}\right) } { \mathrm{t_{th}} \times \omega \log_2 \left( 1 + \frac{\Xi_{u}^t \Gamma_u^t p_u^{t,i}}{\omega \xi^2}\right) - \Upsilon}.
\end{align}
Therefore, if $\sqrt{ \frac{ \mathrm{e_{bd}} \times \omega \log_2 \left( 1 + \frac{\Xi_{u}^t \Gamma_u^t p_u^{t,i}}{\omega \xi^2}\right)  - \Upsilon \cdot p_u^{t,i} } { 0.5 \rho n \bar{n} c_u s_u \kappa_u^{t} \times \omega \log_2 \left( 1 + \frac{\Xi_{u}^t \Gamma_u^t p_u^{t,i}}{\omega \xi^2}\right)} } < \frac {n\bar{n} c_u s_u \kappa_u^{t} \times \omega \log_2 \left( 1 + \frac{\Xi_{u}^t \Gamma_u^t p_u^{t,i}}{\omega \xi^2}\right) } { \mathrm{t_{th}} \times \omega \log_2 \left( 1 + \frac{\Xi_{u}^t \Gamma_u^t p_u^{t,i}}{\omega \xi^2}\right) - \Upsilon}$, the optimization problem is infeasible.
Besides, if $\frac {n\bar{n} c_u s_u \kappa_u^{t} \times \omega \log_2 \left( 1 + \frac{\Xi_{u}^t \Gamma_u^t p_u^{t,i}}{\omega \xi^2}\right) } { \mathrm{t_{th}} \times \omega \log_2 \left( 1 + \frac{\Xi_{u}^t \Gamma_u^t p_u^{t,i}}{\omega \xi^2}\right) - \Upsilon} > \bar{f}_{u,\mathrm{max}}$ or $\sqrt{ \frac{ \mathrm{e_{bd}} \times \omega \log_2 \left( 1 + \frac{\Xi_{u}^t \Gamma_u^t p_u^{t,i}}{\omega \xi^2}\right)  - \Upsilon \cdot p_u^{t,i} } { 0.5 \rho n \bar{n} c_u s_u \kappa_u^{t} \times \omega \log_2 \left( 1 + \frac{\Xi_{u}^t \Gamma_u^t p_u^{t,i}}{\omega \xi^2}\right)} } < 0$, the problem is infeasible.

Since the objective function is monotonically decreasing with respect to CPU frequency, the optimal \ac{cpu} frequency is 
\begin{align}
\label{optimal_CPU_freq_supp_proof}
    \bar{f}_u^{t^*} = \frac {n\bar{n} c_u s_u \kappa_u^{t} \times \omega \log_2 \left( 1 + \frac{\Xi_{u}^t \Gamma_u^t p_u^{t,i} } {\omega \xi^2}\right)} { \mathrm{t_{th}} \times \omega \log_2 \left( 1 + \frac{\Xi_{u}^t \Gamma_u^t p_u^{t,i}}{\omega \xi^2}\right) - \Upsilon}.
\end{align}
\end{proof}

\section{Proof of Lemma \ref{lemmaTxPowerGiven_Kappa_CPUFreq}}
\noindent
Given the local rounds $\kappa_u^{t}$ and \ac{cpu} frequency $\bar{f}_u^{t^*}$, we optimize the transmission power by solving 
\begin{subequations}
\label{txPowerOptim_Sub_Prob_supp}
\begin{align}
    \underset{ p_u^t }  { \tt{maximize} ~~ } & \quad \epsilon \left[\frac{\kappa_u^{t}}{0.5 \rho n \bar{n} c_u s_u \left(\bar{f}_u^{t^*}\right)^2} \right] + \left(1-\epsilon\right) \left[ \frac {\omega \log_2 \left( 1 + \frac{\Xi_{u}^t \Gamma_u^t p_u^t}{\omega \xi^2}\right) } {p_u^t} \right] \tag{\ref{txPowerOptim_Sub_Prob_supp}} \\
    {\tt{subject ~ to}} & \quad 0 \leq p_u^t \leq p_{u,\mathrm{max}}, \\
    &\quad 0.5 \rho n \bar{n} c_u s_u \kappa_u^{t} \left(\bar{f}_u^{t^*}\right)^2 + \frac{\Upsilon \cdot p_u^t} {\omega \log_2 \left( 1 + \frac{\Xi_{u}^t \Gamma_u^t p_u^t}{\omega \xi^2}\right) } \leq \mathrm{e_{bd}}, \\
    &\quad p_u^t \geq \frac{ \omega \xi^2 \left( 2^{\left[\frac{\Upsilon\bar{f}_u^{t^*} } { \omega \left(\mathrm{t_{th}} \bar{f}_u^{t^*} - n\bar{n} c_u s_u \kappa_u^{t} \right) } \right]} - 1 \right) } {\Xi_{u}^t \Gamma_u^t}, \label{eq:power_cons}
\end{align}
\end{subequations}

\begin{theorem-box}[Lemma \ref{lemmaTxPowerGiven_Kappa_CPUFreq}]
Given $\kappa_u^t$ and $\bar{f}_u^{t^*}$, the optimal solution (\ref{txPowerOptim_Sub_Prob_supp}) is
\begin{align}
\label{eq:optPowerSup}
    p_u^{t^*} = \Bigg[ \omega \xi^2 \Bigg( 2^{\Big[\frac{\Upsilon\bar{f}_u^{t^*} } { \omega \left(\mathrm{t_{th}} \bar{f}_u^{t^*} - n\bar{n} c_u s_u \kappa_u^{t} \right) } \Big]} - 1 \Bigg) \Bigg] / \left(\Xi_{u}^t \Gamma_u^t \right).
\end{align}
\end{theorem-box}
\begin{proof}
    We want to check the convexity and monotonicity of the objective function \eqref{txPowerOptim_Sub_Prob_supp} on the transmit power $p_u^t$.
    Since the first term does not depend on $p_u^t$, we only need to check the second term, which is equivalent to $\frac {\omega \log_2 \left( 1 + \frac{\Xi_{u}^t \Gamma_u^t p_u^t}{\omega \xi^2}\right) } {p_u^t} = \frac {\omega \ln \left( 1 + \frac{\Xi_{u}^t \Gamma_u^t p_u^t}{\omega \xi^2}\right) } {\ln(2) p_u^t}$. 
    Dropping the constants, we have
    Let us define
    \begin{align}
    \label{eq:power_optim_obj}
        \mathcal{F} (p_u^t) \coloneqq \frac { \ln \left( 1 + \frac{\Xi_{u}^t \Gamma_u^t p_u^t}{\omega \xi^2}\right) } {p_u^t} = \frac { \ln \left( 1 + \frac{ \mathfrak{b}_{u}^t p_u^t}{\mathfrak{b}_1}\right) } {p_u^t}, 
    \end{align}
    where $\mathfrak{b}_{u}^t \coloneqq \Xi_{u}^t \Gamma_u^t$ and $\mathfrak{b}_1 \coloneqq \omega \xi^2$.

    Then the first derivative of \eqref{eq:power_optim_obj} with respect to $p_u^t$ is calculated as 
    \begin{align}
    \label{eq:power_optim_first_derivative}
        \mathcal{F}'(p_u^t) & \coloneqq \frac{d}{d p_u^t}\left[ \frac { \ln \left( 1 + \frac{ \mathfrak{b}_{u}^t p_u^t}{\mathfrak{b}_1}\right) } {p_u^t} \right] 
        =\frac{ \left( \frac{ \frac{ \mathfrak{b}_{u}^t}{\mathfrak{b}_1} }{ 1 + \frac{ \mathfrak{b}_{u}^t p_u^t}{\mathfrak{b}_1}} \right) p_u^t -  \ln \left( 1 + \frac{ \mathfrak{b}_{u}^t p_u^t}{\mathfrak{b}_1}\right) }{ (p_u^t)^2 } \nonumber\\
        &= \frac{\mathfrak{b}_{u}^t}{\mathfrak{b}_1 p_u^t + \mathfrak{b}_{u}^t (p_u^t)^2} - \frac{\ln \left( 1 + \frac{ \mathfrak{b}_{u}^t p_u^t}{\mathfrak{b}_1}\right)}{(p_u^t)^2 }.
    \end{align}
    Similarly, the second derivative with respect to $p_u^t$ is
    \begin{align}
        \mathcal{F}''(p_u^t) & \coloneqq \frac{d}{d p_u^t} \left[ \mathcal{F}' (p_u^t) \right] 
        = \frac{d}{d p_u^t} \left[\frac{\mathfrak{b}_{u}^t}{\mathfrak{b}_1 p_u^t + \mathfrak{b}_{u}^t (p_u^t)^2} \right] - \frac{d}{d p_u^t} \left[ \frac{\ln \left( 1 + \frac{ \mathfrak{b}_{u}^t p_u^t}{\mathfrak{b}_1}\right)}{(p_u^t)^2 } \right] \nonumber\\
        &=\frac{1}{(p_u^t)^3} \left[ \underbrace{2\ln \left(1 + \frac{\mathfrak{b}_u^t p_u^t}{\mathfrak{b}_1}\right) - \frac{\mathfrak{b}_u^t p_u^t (2\mathfrak{b}_1 + 3 \mathfrak{b}_u^t p_u^t) }{\left(\mathfrak{b}_1 + \mathfrak{b}_u^t p_u^t \right)^2}}_{\mathcal{F}_1(p_u^t)} \right].
    \end{align}
    Since $0 \leq p_u^t$, we now need to focus on the $\mathcal{F}_1(p_u^t)$ term. 
    Suppose $\mathfrak{m} \coloneqq \frac{\mathfrak{b}_u^t p_u^t}{\mathfrak{b}_1 + \mathfrak{b}_u^t p_u^t}$. 
    Therefore, $0 < \mathfrak{m} <1$.
    Then, we have 
    \begin{align}
        \mathcal{F}_1(p_u^t) &= 2\ln \left(1 + \frac{\mathfrak{b}_u^t p_u^t}{\mathfrak{b}_1}\right) - \frac{\mathfrak{b}_u^t p_u^t (2\mathfrak{b}_1 + 3 \mathfrak{b}_u^t p_u^t) }{\left(\mathfrak{b}_1 + \mathfrak{b}_u^t p_u^t \right)^2} \nonumber\\
        &= 2\ln\left( \frac{1}{1 - \mathfrak{m} } \right) - \mathfrak{m}(2+\mathfrak{m}) \nonumber\\
        &= -2 \underbrace{\ln(1 - \mathfrak{m})}_{\mathcal{G}(\mathfrak{m})} - 2\mathfrak{m} - \mathfrak{m}^2. 
    \end{align}
    Then, using the Maclaurin series, we have 
    \begin{align}
        \mathcal{F}_1(p_u^t) 
        &= -2 \left[\mathcal{G}(0) + \frac{\mathcal{G}'(0)}{1!} \mathfrak{m} + \frac{\mathcal{G}''(0)}{2!} \mathfrak{m}^2 + \frac{\mathcal{G}^{3}(\mathfrak{m})}{3!} \mathfrak{m}^3 + \frac{\mathcal{G}^ {4} (0)}{4!} \mathfrak{m}^4 + \dots \right] - 2\mathfrak{m} - \mathfrak{m}^2 \nonumber\\
        &=\left[0 + 2 \mathfrak{m} + 2. \frac{\mathfrak{m}^2}{2} + 2. \frac{\mathfrak{m}^3}{3} + 2 . \frac{\mathfrak{m}^4}{4} + \dots \right] - 2\mathfrak{m} - \mathfrak{m}^2 \nonumber\\
        &= \frac{2\mathfrak{m}^3}{3} + \frac{\mathfrak{m}^4}{2} + \dots
    \end{align}
    Therefore, it is clear that $\mathcal{F}_1(p_u^t) >0$. 
    As such, $\mathcal{F}''(p_u^t) >0$, i.e., positive definite.
    In other words, $\mathcal{F}(p_u^t)$ is {\em convex} on the transmit power $p_u^t$.

    Now, we check the monotonicity of the objective function. 
    Using the first derivative in \eqref{eq:power_optim_first_derivative}, we have 
    \begin{align}
        \mathcal{F}'(p_u^t) & \coloneqq \frac{d}{d p_u^t}\left[ \frac { \ln \left( 1 + \frac{ \mathfrak{b}_{u}^t p_u^t}{\mathfrak{b}_1}\right) } {p_u^t} \right] 
        =\frac{ \left( \frac{ \frac{ \mathfrak{b}_{u}^t}{\mathfrak{b}_1} }{ 1 + \frac{ \mathfrak{b}_{u}^t p_u^t}{\mathfrak{b}_1}} \right) p_u^t -  \ln \left( 1 + \frac{ \mathfrak{b}_{u}^t p_u^t}{\mathfrak{b}_1}\right) }{ (p_u^t)^2 }.
    \end{align}
    Since $p_u^t >0$, the sign of $\mathcal{F}'(p_u^t)$ depends on the numerator. 
    \begin{align}
        \tilde{\mathcal{F}}'(p_u^t) 
        & \coloneqq \left( \frac{ \frac{ \mathfrak{b}_{u}^t}{\mathfrak{b}_1} }{ 1 + \frac{ \mathfrak{b}_{u}^t p_u^t}{\mathfrak{b}_1}} \right) p_u^t -  \ln \left( 1 + \frac{ \mathfrak{b}_{u}^t p_u^t}{\mathfrak{b}_1}\right) \nonumber\\
        &= \frac{\mathfrak{z}_u^t}{1+\mathfrak{z}_u^t} - \ln(1+\mathfrak{z}_u^t),
    \end{align}
    where $\mathfrak{z}_u^t \coloneqq \frac{ \mathfrak{b}_{u}^t p_u^t}{\mathfrak{b}_1}$.
    Now, denote $\mathfrak{y}_u^t \coloneqq \frac{\mathfrak{z}_u^t}{1+\mathfrak{z}_u^t}$.
    Then, we have
    \begin{align}
        \tilde{\mathcal{F}}'(p_u^t) 
        & = \mathfrak{y}_u^t - \ln(1+\mathfrak{z}_u^t).
    \end{align}
    However, notice that $1 - \mathfrak{y}_u^t = \frac{1 + \mathfrak{z}_u^t - \mathfrak{z}_u^t}{1+\mathfrak{z}_u^t} = \frac{1}{1 + \mathfrak{z}_u^t}$.
    Thus, we have 
    \begin{align}
        \tilde{\mathcal{F}}'(p_u^t) 
        & = \mathfrak{y}_u^t - \ln(1+\mathfrak{z}_u^t) \nonumber\\
        &=\mathfrak{y}_u^t - \ln\left(\frac{1}{1 - \mathfrak{y}_u^t}\right) \nonumber\\
        &= \mathfrak{y}_u^t - \left[-\ln\left(1 - \mathfrak{y}_u^t\right) \right] \nonumber\\
        &= \mathfrak{y}_u^t -\left[\mathfrak{y}_u^t + \frac{\mathfrak{y}_u^t}{2} + \frac{\mathfrak{y}_u^t}{3}+ \dots \right] \nonumber\\
        &= - \frac{\mathfrak{y}_u^t}{2} - \frac{\mathfrak{y}_u^t}{3} -  \dots. 
    \end{align}
    Since $\mathfrak{y}_u^t$ is non-negative, $\tilde{\mathcal{F}}'(p_u^t) \leq 0$, and thus, $\mathcal{F}'(p_u^t) \leq 0$.
    Therefore, the objective function is monotonously decreasing on transmit power $p_u^t$.

    From the above analysis, we state that since this is a maximization problem and the objective function is convex and monotonically decreasing on the optimization parameter $p_u^t$, the maximum value of this function is the lowest possible $p_u^t$ that satisfies all constraint. 
    This is essentially the lower bound of $p_u^t$ to satisfy constraint \eqref{eq:power_cons}.
    Therefore, the optimal transmit power is 
    \begin{align}
        p_u^{t^*} = \frac{ \omega \xi^2 \left( 2^{\left[\frac{\Upsilon\bar{f}_u^{t^*} } { \omega \left(\mathrm{t_{th}} \bar{f}_u^{t^*} - n\bar{n} c_u s_u \kappa_u^{t} \right) } \right]} - 1 \right) } {\Xi_{u}^t \Gamma_u^t}.
    \end{align}

    We also stress that if $p_u^{t^*}> p_{u,\mathrm{max}}$ or forces $0.5 \rho n \bar{n} c_u s_u \kappa_u^{t} \left(\bar{f}_u^{t^*}\right)^2 + \frac{\Upsilon \cdot p_u^{t^*}} {\omega \log_2 \left( 1 + \frac{\Xi_{u}^t \Gamma_u^t p_u^t}{\omega \xi^2}\right) }$ to be larger than the energy budget $\mathrm{e_{bd}}$, then the problem is infeasible.
\end{proof}

\newpage \clearpage

\section{Proof of Theorem \ref{convgRate}}
\label{proofConvRate}
\acresetall
\subsection{Key Equations}
\noindent
Each client have the following local objective 
\begin{align}
\label{localObj_Sup}
    f_u^t (\mathbf{w}) \coloneqq \frac{1}{|\mathcal{D}_u^t|} \sum_{(\mathbf{x}, y) \in \mathcal{D}_u^t} l(\mathbf{w}|(\mathbf{x}, y)), 
\end{align}
where $l(\mathbf{w}|(\mathbf{x},y)$ is the loss associated to training sample $(\mathbf{x},y)$ and $\mathcal{D}_u^t$ is the available training dataset of client $u$ during global round $t$.

\bigskip
Upon receiving the global model $\mathbf{w}^t$ from the \ac{cs}, the clients synchronize their local models $\mathbf{w}_u^{t,0} \gets \mathbf{w}^{t}$ and take $\kappa_u^t \in [1, \kappa]$ number of local \ac{sgd} steps 
\begin{align}
\label{UEsLocalSGDrounds_Supp}
    \mathbf{w}_u^{t,\kappa_u^t} = \mathbf{w}_u^{t,0} - \eta_\mathrm{lo} \sum_{\tau=0}^{\kappa_u^t-1} g_u^t \left(\mathbf{w}_u^{t,\tau} \right),
\end{align}
where $\eta_\mathrm{lo}$ is the local learning rate.

\bigskip
Upon finishing the local training, the clients calculate normalized accumulated gradients as
\begin{equation}
    \mathbf{d}_u^t \coloneqq \frac{\eta_\mathrm{lo}}{\kappa_u^t}\sum_{\tau=0}^{\kappa_u^t-1} g_u^t \left(\mathbf{w}_u^{t,\tau} \right) = \frac{\mathbf{w}_u^{t,0} - \mathbf{w}_u^{t,\kappa_u^t}}{\kappa_u^t}    
\end{equation}
They then offload quantized normalized gradient updates $Q\left(\mathbf{d}_u^t\right)$, where $Q(\cdot)$ is a stochastic quantizer.
For $\mathbf{d} \in \mathbb{R}^N$ and $\mathbb{d} \neq \mathbf{0}$, this is defined as \cite{alistarh2017qsgd}  
\begin{align}
    Q(\mathbf{d}) \coloneqq \Vert \mathbf{d}\Vert_2 \cdot \mathrm{sign}(d_i) \cdot \zeta_i (\mathbf{d},s), ~\forall i \in [N],
\end{align}
where $\zeta_i(\mathbf{d},s)$ is a random variable and satisfies $\zeta_i(\mathbf{d},s)= 
\begin{cases}
 \frac{l+1}{s}, & \mathrm{w.p.~} \frac{\Vert d_i\Vert}{\Vert \mathbf{d}\Vert_2}s - l, \\
 \frac{l}{s}, & \mathrm{otherwise},
\end{cases}$. 
Besides, $s$ is the tuning parameter defining the total quantization levels
and $l \in [0,s)$ is an integer such that $\frac{\Vert d_i\Vert}{\Vert \mathbf{d}\Vert_2} \in \left[\frac{l}{s}, \frac{l+1}{s} \right]$

The \ac{cs} takes a global \ac{sgd} step with step size $\eta_\mathrm{gl}$ using the normalized accumulated gradients as 
\begin{align}
\label{globalUpdateRule_Supp}
    \mathbf{w}^{t+1} = \mathbf{w}^t - \eta_{\mathrm{gl}} \eta_{\mathrm{ef}}^t \sum_{u=0}^{U-1} \alpha_u \delta_u^t \cdot \mathrm{1}_u^t \cdot Q\left( \mathbf{d}_u^t \right)  
    = \mathbf{w}^t - \eta_{\mathrm{gl}} \eta_{\mathrm{ef}}^t \sum_{u=0}^{U-1} \alpha_u^t \cdot \mathrm{1}_u^t \cdot Q\left(\mathbf{d}_u^t\right),
\end{align}
where $\alpha_u^t \coloneqq \alpha_u \delta_u^t$, $\alpha_u = \frac{\mathrm{D}_u}{\sum_{u=0}^{U-1} \mathrm{D}_u}$, and $\delta_u^t \geq 0$ is the \emph{score} of client $u$ during time $t$.
Besides, $\mathrm{1}_u^t$ is an binary indicator function
\begin{align}
    \mathrm{1}_u^t=
    \begin{cases}
        1, & \text{with probability } v_u^t,\\
        0, & \text{otherwise},
    \end{cases}.
\end{align}

The proposed \ac{osafl} has the following surrogate global objective function.
\begin{align}
\label{globalObj_OSAFL_Supp}
    f^t \left(\mathbf{w}^t \right) \coloneqq \sum_{u=0}^{U-1} \alpha_u^t f_u^t (\mathbf{w}^t). \squad \nabla f^t \left(\mathbf{w}^t \right) &\coloneqq \sum_{u=0}^{U-1} \alpha_u^t \nabla f_u^t (\mathbf{w}^t).
\end{align}

\subsection{Key Assumptions}
\setcounter{Assumption}{0}
\noindent
We make the following standard assumptions \cite{wang2020Tackling, yang22Anarchic, ye23fedDisco, pervej2023resource, pervej2024hierarchical} that are needed for the theoretical analysis. 

\begin{Assumption}
    The client participation indicator function $\mathrm{1}_u^t$ is independent across different global rounds and across different users due to time-varying wireless channels, and is also independent of mini-batch sampling in \ac{sgd}.
\end{Assumption}
\begin{Assumption}[Smoothness]
    The local loss functions are $\beta$-Lipschitz smooth. That is, for some $\beta > 0$, $\Vert \nabla f_u^t \left(\mathbf{w} \right) - \nabla f_u^t \left(\mathbf{w}' \right) \Vert \leq \beta \Vert \mathbf{w} - \mathbf{w}' \Vert$, for all $\mathbf{w}$, $\mathbf{w}' \in \mathbb{R}^N$, $\mathcal{D}_u^t$, and $u \in \mathcal{U}$.
\end{Assumption}

\begin{Assumption}[Unbiased gradient with bounded variance]
    The stochastic gradient at each client is an unbiased estimate of the client's true gradient, i.e., $\mathbb{E}_{\zeta \sim \mathcal{D}_u^t} [g_u^t \left(\mathbf{w}  \right)] = \nabla f_u^t \left(\mathbf{w} \right)$, where $\mathbb{E}[\cdot]$ is the expectation operator. 
    Besides, the stochastic gradient has a bounded variance, i.e., $\mathbb{E}_{\zeta \sim \mathcal{D}_u^t} \left[\Vert g_u^t \left(\mathbf{w} \right) - \nabla f_u^t \left(\mathbf{w} \right) \Vert^2\right] \leq \sigma^2$, for some $\sigma \geq 0$ and for all $u \in \mathcal{U}$. 
\end{Assumption}

\begin{Assumption}[Bounded gradient dissimilarity]
The divergence between the local and global gradients is bounded, i.e.,
\begin{equation}
    \left\Vert \nabla f_u^t \left(\mathbf{w}  \right) - \nabla f^t \left(\mathbf{w} \right) \right \Vert^2 \leq \varpi^2,
\end{equation}
for some $\varpi \geq 0$, $\forall u \in \mathcal{U}$.
\end{Assumption}

\begin{Assumption}[Unbiased quantizer and bounded variance]
The stochastic quantizer is unbiased, i.e., $\mathbb{E}_Q[Q(\mathbf{d})] = \mathbf{d}$ and has bounded variance $\mathbb{E}_Q \left[ \left\Vert Q(\mathbf{d}) - \mathbf{d} \right\Vert^2 \right] \leq q \left\Vert \mathbf{d}\right\Vert^2$.    
\end{Assumption}

\newpage
\begin{Lemma}
\label{lemma_convrate}
If assumptions 1-5 hold, and the learning rates satisfy $\eta_\mathrm{lo} \leq \frac{1}{8\beta \kappa}$ and $\eta_\mathrm{gl}\eta_\mathrm{lo} \leq \frac{1}{\beta \eta_\mathrm{ef}^t}$, then we have the following bound
\begin{align}
    &\mathbb{E}_{\pmb{1,\zeta},Q}\left[f^{t} (\mathbf{w}^{t+1})\right] - f^t(\mathbf{w}^t) \nonumber\\
    &\leq - \frac{\eta_\mathrm{gl} \eta_\mathrm{lo} \eta_\mathrm{ef}^t}{2} \Big[1 - 4 \sum_{u=0}^{U-1} \alpha_u ( \delta_u^t (1 - v_u^t) )^2 - 48 \beta^2 \eta_\mathrm{lo}^2 \sum_{u=0}^{U-1} \alpha_u (\kappa_u^t v_u^t \delta_u^t)^2 - 6 \beta \eta_\mathrm{gl} \eta_\mathrm{lo} \eta_\mathrm{ef}^t \sum_{u=0}^{U-1} v_u^t (1+q-v_u^t) \kappa_u^t \left(\alpha_u \delta_u^t\right)^2 - \nonumber\\
    & \mquad 144 \eta_\mathrm{gl} \eta_\mathrm{ef}^t \beta^3 \eta_\mathrm{lo}^3 \sum_{u=0}^{U-1} v_u^t(1+q-v_u^t) \left(\alpha_u \delta_u^t\right)^2 \left(\kappa_u^t \right)^3 \Big] \left\Vert \nabla f^{t}(\mathbf{w}^t) \right\Vert^2 + \digamma_\sigma (\sigma, \pmb{v, \kappa^t,\delta^t}) + \digamma_\varpi (\sigma, \pmb{v,\kappa^t,\delta^t}),
\end{align}
where $\digamma_\sigma (\sigma, \pmb{v, \kappa^t,\delta^t}) \coloneqq (2+q) \beta \left(\sigma \eta_\mathrm{gl} \eta_\mathrm{lo} \eta_\mathrm{ef}^t \right)^2 \sum_{u=0}^{U-1} v_u^t \alpha_u^2 \frac{ (\delta_u^t)^2}{\kappa_u^t} + 4 \eta_\mathrm{gl} \eta_\mathrm{ef}^t \beta^2 \sigma^2 \eta_\mathrm{lo}^3 \sum_{u=0}^{U-1} \alpha_u \kappa_u^t (v_u^t\delta_u^t)^2 + 12 \sigma^2 \left(\eta_\mathrm{gl} \eta_\mathrm{ef}^t \right)^2 \beta^3 \eta_\mathrm{lo}^4  \sum_{u=0}^{U-1} v_u^t(1+q-v_u^t) \left(\alpha_u \delta_u^t \kappa_u^t \right)^2 $ and $\digamma_\varpi (\sigma, \pmb{v,\kappa^t,\delta^t}) \coloneqq 2\eta_\mathrm{gl} \eta_\mathrm{lo} \eta_\mathrm{ef}^t \varpi^2 \sum_{u=0}^{U-1} \alpha_u ( \delta_u^t (1 - v_u^t) )^2 + 24 \eta_\mathrm{gl} \eta_\mathrm{ef}^t \left(\beta \varpi \right)^2 \eta_\mathrm{lo}^3  \sum_{u=0}^{U-1} \alpha_u (v_u^t\delta_u^t)^2 \left(\kappa_u^t\right)^2 + 3 \beta \left(\eta_\mathrm{gl} \eta_\mathrm{lo} \eta_\mathrm{ef}^t \varpi\right)^2 \sum_{u=0}^{U-1} v_u^t(1+q-v_u^t) \left(\alpha_u \delta_u^t\right)^2 \kappa_u^t + 72 \left(\varpi \eta_\mathrm{gl} \eta_\mathrm{ef}^t \right)^2 \beta^3 \eta_\mathrm{lo}^4 \sum_{u=0}^{U-1} v_u^t(1+q-v_u^t) \left(\alpha_u \delta_u^t\right)^2 \left(\kappa_u^t \right)^3 $ in $(a)$.
\end{Lemma}

\begin{proof}
For convenience, we denote 
\begin{align}
    \mathbf{d}_u^t &\coloneqq \frac{\eta_{\mathrm{lo}}} {\kappa_u^t}\sum_{\tau=0}^{\kappa_u^t-1} g_u^t \left( \mathbf{w}_u^{t,\tau} \right). \\
    \tilde{\mathbf{d}}_u^t & \coloneqq \frac{\eta_{\mathrm{lo}}} {\kappa_u^t}\sum_{\tau=0}^{\kappa_u^t-1} \nabla f_u^t \left( \mathbf{w}_u^{t,\tau} \right). \label{tilde_dut}
\end{align}
From $\beta$-smoothness, we have
\begin{align}
    f^{t} (\mathbf{w}^{t+1}) - f^t(\mathbf{w}^t) 
    &\leq \left< \nabla f^t (\mathbf{w}^t), \mathbf{w}^{t+1} - \mathbf{w}^t \right> + \frac{\beta}{2} \left\Vert \mathbf{w}^{t+1} - \mathbf{w}^t \right\Vert^2.
\end{align}

Now, taking expectation over all randomness (randomness due to client participation, mini-batch sampling for the stochastic gradients, and stochastic quantization of the normalized gradients), we have the following due to the law of total expectation.
\begin{align}
\label{eq:main_conv_1}
    &\mathbb{E}_{\pmb{1,\zeta},Q}\left[f^{t} (\mathbf{w}^{t+1})\right] - f^t(\mathbf{w}^t) 
    \leq \mathbb{E}_{\pmb{1,\zeta},Q}\left[\left< \nabla f^t (\mathbf{w}^t), \mathbf{w}^{t+1} - \mathbf{w}^t \right>\right] + \frac{\beta}{2} \mathbb{E}_{\pmb{1,\zeta},Q}\left[\left\Vert \mathbf{w}^{t+1} - \mathbf{w}^t \right\Vert^2\right] \nonumber\\
    &\overset{(a)}{=} \mathbb{E}_{\pmb{1,\zeta},Q} \left[\left< \nabla f^t (\mathbf{w}^t), \mathbf{w}^{t} - \eta_\mathrm{gl} \eta_\mathrm{ef}^t \sum_{u=0}^{U-1} \alpha_u^t \cdot \mathrm{1}_u^t \cdot Q(\mathbf{d}_u^t) - \mathbf{w}^t \right>\right] + \frac{\beta}{2} \mathbb{E}_{\pmb{1,\zeta},Q}\left[ \left\Vert \mathbf{w}^{t} - \eta_\mathrm{gl} \eta_\mathrm{ef}^t \sum_{u=0}^{U-1} \alpha_u^t \cdot \mathrm{1}_u^t \cdot Q(\mathbf{d}_u^t) - \mathbf{w}^t \right\Vert^2\right] \nonumber\\
    &= \underbrace{\mathbb{E}_{\pmb{1,\zeta},Q} \left[\left< \nabla f^t (\mathbf{w}^t), - \eta_\mathrm{gl} \eta_\mathrm{ef}^t \sum_{u=0}^{U-1} \alpha_u^t \cdot \mathrm{1}_u^t \cdot Q(\mathbf{d}_u^t) \right>\right]}_{\mathrm{T_{1}}} + \underbrace{\frac{\beta \left(\eta_\mathrm{gl}\eta_\mathrm{ef}^t\right)^2}{2} \mathbb{E}_{\pmb{1,\zeta},Q} \left[\left\Vert \sum_{u=0}^{U-1} \alpha_u^t \cdot \mathrm{1}_u^t \cdot Q(\mathbf{d}_u^t) \right\Vert^2\right]}_{\mathrm{T}_2}\,
\end{align}
where we used the global model update rule, defined in (\ref{globalUpdateRule_Supp}), in step $(a)$.

We simplify $\mathrm{T}_1$ as 
\begin{align}
    \mathrm{T}_1 
    &= - \eta_\mathrm{gl} \eta_\mathrm{ef}^t \mathbb{E}_{\pmb{1,\zeta},Q} \left[\left< \nabla f^t (\mathbf{w}^t), \sum_{u=0}^{U-1} \alpha_u^t \cdot \mathrm{1}_u^t \cdot Q(\mathbf{d}_u^t) \right>\right] \nonumber \\
    &\overset{(a)}{=} - \eta_\mathrm{gl} \eta_\mathrm{ef}^t \mathbb{E}_{\pmb{1,\zeta},Q} \left[\left< \nabla f^t (\mathbf{w}^t), \sum_{u=0}^{U-1} \alpha_u^t \cdot \mathrm{1}_u^t \cdot Q(\mathbf{d}_u^t) - \sum_{u=0}^{U-1} \alpha_u^t \cdot \mathrm{1}_u^t \cdot \tilde{\mathbf{d}}_u^t + \sum_{u=0}^{U-1} \alpha_u^t \cdot \mathrm{1}_u^t \cdot \tilde{\mathbf{d}}_u^t \right>\right] \nonumber \\
    &=- \eta_\mathrm{gl} \eta_\mathrm{ef}^t \mathbb{E}_{\pmb{1,\zeta},Q} \left[\left< \nabla f^t (\mathbf{w}^t), \sum_{u=0}^{U-1} \alpha_u^t \cdot \mathrm{1}_u^t \left(Q(\mathbf{d}_u^t) - \tilde{\mathbf{d}}_u^t \right) + \sum_{u=0}^{U-1} \alpha_u^t \cdot \mathrm{1}_u^t \cdot \tilde{\mathbf{d}}_u^t \right>\right] \nonumber \\
    &= - \eta_\mathrm{gl} \eta_\mathrm{ef}^t \mathbb{E}_{\pmb{1,\zeta},Q} \left[\left< \nabla f^t (\mathbf{w}^t), \sum_{u=0}^{U-1} \alpha_u^t \cdot \mathrm{1}_u^t \left( Q(\mathbf{d}_u^t) - \tilde{\mathbf{d}}_u^t \right) \right>\right] - \eta_\mathrm{gl} \eta_\mathrm{ef}^t \mathbb{E}_{\pmb{1,\zeta},Q} \left[\left< \nabla f^t (\mathbf{w}^t), \sum_{u=0}^{U-1} \alpha_u^t \cdot \mathrm{1}_u^t \cdot \tilde{\mathbf{d}}_u^t \right>\right] \nonumber \\
    &= - \eta_\mathrm{gl} \eta_\mathrm{ef}^t \mathbb{E}_{\pmb{\zeta},Q} \left[ \mathbb{E}_{\pmb{1}|\pmb{\zeta},Q}  \left[\left< \nabla f^t (\mathbf{w}^t), \sum_{u=0}^{U-1} \alpha_u^t \cdot \mathrm{1}_u^t \left( Q(\mathbf{d}_u^t) - \tilde{\mathbf{d}}_u^t \right) \right>\right] \right] - \eta_\mathrm{gl} \eta_\mathrm{ef}^t \mathbb{E}_{\pmb{\zeta},Q} \left[ \mathbb{E}_{\pmb{1}|\pmb{\zeta},Q} \left[\left< \nabla f^t (\mathbf{w}^t), \sum_{u=0}^{U-1} \alpha_u^t \cdot \mathrm{1}_u^t \cdot \tilde{\mathbf{d}}_u^t \right>\right] \right] \nonumber \\ 
    &\overset{(b)}{=} - \eta_\mathrm{gl} \eta_\mathrm{ef}^t \mathbb{E}_{\pmb{\zeta},Q} \left[ \left< \nabla f^t (\mathbf{w}^t), \sum_{u=0}^{U-1} \alpha_u^t \cdot v_u^t \left( Q(\mathbf{d}_u^t) - \tilde{\mathbf{d}}_u^t \right) \right>\right] - \eta_\mathrm{gl} \eta_\mathrm{ef}^t \mathbb{E}_{\pmb{\zeta},Q} \left[ \left< \nabla f^t (\mathbf{w}^t), \sum_{u=0}^{U-1} \alpha_u^t \cdot v_u^t \cdot \tilde{\mathbf{d}}_u^t \right>\right] \nonumber\\
    &= - \eta_\mathrm{gl} \eta_\mathrm{ef}^t \mathbb{E}_{\pmb{\zeta}} \left[ \mathbb{E}_{Q|\pmb{\zeta}} \left[ \left< \nabla f^t (\mathbf{w}^t), \sum_{u=0}^{U-1} \alpha_u^t \cdot v_u^t \left( Q(\mathbf{d}_u^t) - \tilde{\mathbf{d}}_u^t \right) \right>\right] \right] - \eta_\mathrm{gl} \eta_\mathrm{ef}^t \mathbb{E}_{\pmb{\zeta}} \left[ \mathbb{E}_{Q|\pmb{\zeta}} \left[ \left< \nabla f^t (\mathbf{w}^t), \sum_{u=0}^{U-1} \alpha_u^t \cdot v_u^t \cdot \tilde{\mathbf{d}}_u^t \right>\right] \right] \nonumber\\
    &\overset{(c)}{=} - \eta_\mathrm{gl} \eta_\mathrm{ef}^t \mathbb{E}_{\pmb{\zeta}} \left[ \left< \nabla f^t (\mathbf{w}^t), \sum_{u=0}^{U-1} \alpha_u^t \cdot v_u^t \left( \mathbf{d}_u^t - \tilde{\mathbf{d}}_u^t \right) \right> \right] - \eta_\mathrm{gl} \eta_\mathrm{ef}^t \mathbb{E}_{\pmb{\zeta}} \left[ \left< \nabla f^t (\mathbf{w}^t), \sum_{u=0}^{U-1} \alpha_u^t \cdot v_u^t \cdot \tilde{\mathbf{d}}_u^t \right> \right] \nonumber\\
    &= - \eta_\mathrm{gl} \eta_\mathrm{ef}^t \mathbb{E}_{\pmb{\zeta}} \left[\left< \nabla f^t (\mathbf{w}^t), \sum_{u=0}^{U-1} \alpha_u^t v_u^t \left(\frac{\eta_\mathrm{lo}} {\kappa_u^t} \sum_{\tau=0}^{\kappa_u^t-1} g_u^t \left( \mathbf{w}_u^{t,\tau} \right) - \frac{\eta_\mathrm{lo}} {\kappa_u^t}\sum_{\tau=0}^{\kappa_u^t-1} \nabla f_u^t \left( \mathbf{w}_u^{t,\tau} \right) \right) \right>\right] - \nonumber\\
    &\bquad \bquad \mquad \eta_\mathrm{gl} \eta_\mathrm{ef}^t \mathbb{E}_{\pmb{\zeta}} \left[\left< \nabla f^t (\mathbf{w}^t), \sum_{u=0}^{U-1} \alpha_u^t v_u^t \frac{\eta_\mathrm{lo}} {\kappa_u^t}\sum_{\tau=0}^{\kappa_u^t-1} \nabla f_u^t \left( \mathbf{w}_u^{t,\tau} \right) \right>\right] \nonumber \\
    &\overset{(d)}{=} - \eta_\mathrm{gl} \eta_\mathrm{lo} \eta_\mathrm{ef}^t \bblue{\mathbb{E}_{\pmb{\zeta}}} \left[\left< \nabla f^t (\mathbf{w}^t), \sum_{u=0}^{U-1} v_u^t \alpha_u^t \frac{1} {\kappa_u^t}\sum_{\tau=0}^{\kappa_u^t-1} \nabla f_u^t \left( \mathbf{w}_u^{t,\tau} \right) \right> \right] \nonumber \\
    &=- \eta_\mathrm{gl} \eta_\mathrm{lo} \eta_\mathrm{ef}^t \bblue{\mathbb{E}_{\pmb{\zeta}}} \left[\left< \nabla f^t (\mathbf{w}^t), \sum_{u=0}^{U-1} \alpha_u v_u^t \left(\frac{\delta_u^t} {\kappa_u^t} \right)\sum_{\tau=0}^{\kappa_u^t-1} \nabla f_u^t \left( \mathbf{w}_u^{t,\tau} \right) \right> \right] \nonumber \\
    &\overset{(e)}{=} - \frac{\eta_\mathrm{gl} \eta_\mathrm{lo} \eta_\mathrm{ef}^t}{2} \bblue{\mathbb{E}_{\pmb{\zeta}}} \left[\left\Vert \nabla f^t (\mathbf{w}^t) \right\Vert^2 \right] - \frac{\eta_\mathrm{gl} \eta_\mathrm{lo} \eta_\mathrm{ef}^t}{2} \bblue{\mathbb{E}_{\pmb{\zeta}}} \left[\left\Vert \sum_{u=0}^{U-1} \alpha_u v_u^t \left( \frac{\delta_u^t} {\kappa_u^t} \right) \sum_{\tau=0}^{\kappa_u^t-1} \nabla f_u^t \left( \mathbf{w}_u^{t,\tau} \right) \right\Vert^2 \right] + \nonumber\\
    &\bquad \bquad \frac{\eta_\mathrm{gl} \eta_\mathrm{lo} \eta_\mathrm{ef}^t}{2} \bblue{\mathbb{E}_{\pmb{\zeta}}} \left[\left\Vert \nabla f^t (\mathbf{w}^t) - \sum_{u=0}^{U-1} \alpha_u v_u^t  \left( \frac{\delta_u^t} {\kappa_u^t} \right) \sum_{\tau=0}^{\kappa_u^t-1} \nabla f_u^t \left( \mathbf{w}_u^{t,\tau} \right) \right\Vert^2\right], 
\end{align}
where $\tilde{\mathbf{d}}_u^t$ in $(a)$ is defined in (\ref{tilde_dut}), $(b)$ uses the fact that $\mathbb{E}[\mathrm{1}_u^t]=v_u^t$, $(c)$ comes from the unbiased stochastic quantizer assumption, $(d)$ stems from the unbiased stochastic gradient assumption, and finally, $(e)$ is true since sl$\Vert \mathbf{x} - \mathbf{y}\Vert^2 = \Vert \mathbf{x} \Vert^2 + \Vert \mathbf{y} \Vert^2 - 2<\mathbf{x}, \mathbf{y}>$.

We further simplify $\mathrm{T_2}$ as follows,
\begin{align}
    \mathrm{T}_2 
    &= \frac{\beta \left(\eta_\mathrm{gl} \eta_\mathrm{ef}^t \right)^2}{2} \mathbb{E}_{\pmb{\mathrm{1},\zeta},Q} \left[\left\Vert \sum_{u=0}^{U-1} \alpha_u^t \cdot \mathrm{1}_u^t \cdot Q(\mathbf{d}_u^t) \right\Vert^2\right] \nonumber\\ 
    &\overset{(a)}{=} \frac{\beta \left(\eta_\mathrm{gl} \eta_\mathrm{ef}^t \right)^2}{2} \mathbb{E}_{\pmb{\mathrm{1},\zeta},Q} \left[\left\Vert \sum_{u=0}^{U-1} \alpha_u^t \cdot \mathrm{1}_u^t \cdot Q(\mathbf{d}_u^t) - \mathbb{E}_{\pmb{\mathrm{1},\zeta},Q} \left[\sum_{u=0}^{U-1} \alpha_u^t \cdot \mathrm{1}_u^t \cdot Q(\mathbf{d}_u^t)\right] \right\Vert^2\right] + \frac{\beta \left(\eta_\mathrm{gl} \eta_\mathrm{ef}^t \right)^2}{2} \left(\mathbb{E}_{\pmb{\mathrm{1},\zeta},Q} \left[\sum_{u=0}^{U-1} \alpha_u^t \cdot \mathrm{1}_u^t \cdot Q(\mathbf{d}_u^t)\right]\right)^2\nonumber\\
    &= \frac{\beta \left(\eta_\mathrm{gl} \eta_\mathrm{ef}^t \right)^2}{2} \mathbb{E}_{\pmb{\mathrm{1},\zeta},Q} \left[\left\Vert \sum_{u=0}^{U-1} \alpha_u^t \cdot \mathrm{1}_u^t \cdot Q(\mathbf{d}_u^t) - \sum_{u=0}^{U-1} \alpha_u^t v_u^t \tilde{\mathbf{d}}_u^t  \right\Vert^2\right] + \frac{\beta \left(\eta_\mathrm{gl} \eta_\mathrm{ef}^t \right)^2}{2} \left\Vert \sum_{u=0}^{U-1} \alpha_u^t v_u^t \tilde{\mathbf{d}}_u^t)\right\Vert^2\nonumber\\
    &\overset{(b)}{=} \frac{\beta \left(\eta_\mathrm{gl} \eta_\mathrm{ef}^t \right)^2}{2}  \sum_{u=0}^{U-1} (\alpha_u^t)^2 \mathbb{E}_{\pmb{\mathrm{1},\zeta},Q} \left[\left\Vert \mathrm{1}_u^t \cdot Q(\mathbf{d}_u^t) - v_u^t \tilde{\mathbf{d}}_u^t  \right\Vert^2\right] + \frac{\beta \left(\eta_\mathrm{gl} \eta_\mathrm{ef}^t \right)^2}{2} \left\Vert \sum_{u=0}^{U-1} \alpha_u^t v_u^t \tilde{\mathbf{d}}_u^t)\right\Vert^2\nonumber\\
    &=\frac{\beta \left(\eta_\mathrm{gl} \eta_\mathrm{ef}^t \right)^2}{2}  \sum_{u=0}^{U-1} (\alpha_u^t)^2 \mathbb{E}_{\pmb{\mathrm{1},\zeta},Q} \left[\left\Vert \mathrm{1}_u^t \cdot \left( Q(\mathbf{d}_u^t) - \mathbf{d}_u^t\right) + \mathrm{1}_u^t \mathbf{d}_u^t - v_u^t \tilde{\mathbf{d}}_u^t  \right\Vert^2\right] + \frac{\beta \left(\eta_\mathrm{gl} \eta_\mathrm{ef}^t \right)^2}{2} \left\Vert \sum_{u=0}^{U-1} \alpha_u^t v_u^t \tilde{\mathbf{d}}_u^t)\right\Vert^2\nonumber\\
    &\leq \beta \left(\eta_\mathrm{gl} \eta_\mathrm{ef}^t \right)^2  \sum_{u=0}^{U-1} (\alpha_u^t)^2 \mathbb{E}_{\pmb{\mathrm{1},\zeta},Q} \left[\left\Vert \mathrm{1}_u^t \cdot \left( Q(\mathbf{d}_u^t) - \mathbf{d}_u^t\right) \right\Vert^2 \right] + \beta \left(\eta_\mathrm{gl} \eta_\mathrm{ef}^t \right)^2  \sum_{u=0}^{U-1} (\alpha_u^t)^2 \mathbb{E}_{\pmb{\mathrm{1},\zeta},Q} \left[ \left\Vert \mathrm{1}_u^t \mathbf{d}_u^t - v_u^t \tilde{\mathbf{d}}_u^t  \right\Vert^2\right] + \nonumber\\
    &\bquad \frac{\beta \left(\eta_\mathrm{gl} \eta_\mathrm{ef}^t \right)^2}{2} \left\Vert \sum_{u=0}^{U-1} \alpha_u^t v_u^t \tilde{\mathbf{d}}_u^t)\right\Vert^2\nonumber\\
    &= \beta \left(\eta_\mathrm{gl} \eta_\mathrm{ef}^t \right)^2  \sum_{u=0}^{U-1} (\alpha_u^t)^2 v_u^t \mathbb{E}_{\pmb{\zeta}} \left[ \mathbb{E}_{Q|\pmb{\zeta}} \left[ \left\Vert Q(\mathbf{d}_u^t) - \mathbf{d}_u^t \right\Vert^2 \right] \right] + \beta \left(\eta_\mathrm{gl} \eta_\mathrm{ef}^t \right)^2  \sum_{u=0}^{U-1} (\alpha_u^t)^2 \mathbb{E}_{\pmb{\mathrm{1},\zeta}} \left[ \left\Vert \mathrm{1}_u^t( \mathbf{d}_u^t - \tilde{\mathbf{d}}_u^t) + (\mathrm{1}_u^t - v_u^t) \tilde{\mathbf{d}}_u^t  \right\Vert^2\right] + \nonumber\\
    &\bquad \frac{\beta \left(\eta_\mathrm{gl} \eta_\mathrm{ef}^t \right)^2}{2} \left\Vert \sum_{u=0}^{U-1} \alpha_u^t v_u^t \tilde{\mathbf{d}}_u^t)\right\Vert^2\nonumber\\
    &\overset{(c)}{\leq} \beta \left(\eta_\mathrm{gl} \eta_\mathrm{ef}^t \right)^2  \sum_{u=0}^{U-1} (\alpha_u^t)^2 v_u^t \mathbb{E}_{\pmb{\zeta}} \left[q \left\Vert \mathbf{d}_u^t \right\Vert^2  \right] + 2\beta \left(\eta_\mathrm{gl} \eta_\mathrm{ef}^t \right)^2  \sum_{u=0}^{U-1} (\alpha_u^t)^2 \mathbb{E}_{\pmb{\mathrm{1},\zeta}} \left[ \left\Vert \mathrm{1}_u^t( \mathbf{d}_u^t - \tilde{\mathbf{d}}_u^t) \right\Vert^2 \right] + \nonumber\\
    & \bquad 2\beta \left(\eta_\mathrm{gl} \eta_\mathrm{ef}^t \right)^2  \sum_{u=0}^{U-1} (\alpha_u^t)^2 \mathbb{E}_{\pmb{\mathrm{1},\zeta}} \left[ \left\Vert (\mathrm{1}_u^t - v_u^t) \tilde{\mathbf{d}}_u^t  \right\Vert^2\right] + \frac{\beta \left(\eta_\mathrm{gl} \eta_\mathrm{ef}^t \right)^2}{2} \left\Vert \sum_{u=0}^{U-1} \alpha_u^t v_u^t \tilde{\mathbf{d}}_u^t)\right\Vert^2\nonumber\\
    &= \beta \left(\eta_\mathrm{gl} \eta_\mathrm{ef}^t \right)^2  \sum_{u=0}^{U-1} (\alpha_u^t)^2 v_u^t \mathbb{E}_{\pmb{\zeta}} \left[q \left\Vert \mathbf{d}_u^t \right\Vert^2  \right] + 2\beta \left(\eta_\mathrm{gl} \eta_\mathrm{ef}^t \right)^2  \sum_{u=0}^{U-1} (\alpha_u^t)^2 \left\{ v_u^t \mathbb{E}_{\zeta} \left[ \left\Vert 1 ( \mathbf{d}_u^t - \tilde{\mathbf{d}}_u^t) \right\Vert^2 \right] + (1-v_u^t) \mathbb{E}_{\zeta} \left[ \left\Vert 0 ( \mathbf{d}_u^t - \tilde{\mathbf{d}}_u^t) \right\Vert^2 \right] \right\}  + \nonumber\\
    & \mquad 2\beta \left(\eta_\mathrm{gl} \eta_\mathrm{ef}^t \right)^2  \sum_{u=0}^{U-1} (\alpha_u^t)^2 \left\{ v_u^t \mathbb{E}_{\pmb{\zeta}} \left[ \left\Vert (1 - v_u^t) \tilde{\mathbf{d}}_u^t  \right\Vert^2\right] + (1-v_u^t) \mathbb{E}_{\pmb{\zeta}} \left[ \left\Vert (0 - v_u^t) \tilde{\mathbf{d}}_u^t  \right\Vert^2\right] \right\} + \frac{\beta \left(\eta_\mathrm{gl} \eta_\mathrm{ef}^t \right)^2}{2} \left\Vert \sum_{u=0}^{U-1} \alpha_u^t v_u^t \tilde{\mathbf{d}}_u^t)\right\Vert^2 \nonumber\\ 
    &= q\beta \left(\eta_\mathrm{gl} \eta_\mathrm{ef}^t \right)^2  \sum_{u=0}^{U-1} (\alpha_u^t)^2 v_u^t \mathbb{E}_{\pmb{\zeta}} \left[ \left\Vert \mathbf{d}_u^t \right\Vert^2 \right] + 2\beta \left(\eta_\mathrm{gl} \eta_\mathrm{ef}^t \right)^2  \sum_{u=0}^{U-1} v_u^t (\alpha_u^t)^2 \mathbb{E}_{\pmb{\zeta}} \left[ \left\Vert \mathbf{d}_u^t - \tilde{\mathbf{d}}_u^t \right\Vert^2 \right] + \nonumber\\
    & \bquad 2\beta \left(\eta_\mathrm{gl} \eta_\mathrm{ef}^t \right)^2  \sum_{u=0}^{U-1} v_u^t (\alpha_u^t)^2 \mathbb{E}_{\pmb{\zeta}} \left[ \left\Vert (1 - v_u^t) \tilde{\mathbf{d}}_u^t  \right\Vert^2\right] + 2\beta \left(\eta_\mathrm{gl} \eta_\mathrm{ef}^t \right)^2  \sum_{u=0}^{U-1} (1 - v_u^t) (\alpha_u^t)^2 \mathbb{E}_{\pmb{\zeta}} \left[ \left\Vert v_u^t \tilde{\mathbf{d}}_u^t  \right\Vert^2\right] + \nonumber\\
    &\bquad \frac{\beta \left(\eta_\mathrm{gl} \eta_\mathrm{ef}^t \right)^2}{2} \left\Vert \sum_{u=0}^{U-1} \alpha_u^t v_u^t \tilde{\mathbf{d}}_u^t)\right\Vert^2\nonumber\\
    &= q \beta \left(\eta_\mathrm{gl} \eta_\mathrm{ef}^t \right)^2  \sum_{u=0}^{U-1} (\alpha_u^t)^2 v_u^t \mathbb{E}_{\pmb{\zeta}} \left[ \left\Vert \mathbf{d}_u^t - \mathbb{E}_{\pmb{\zeta}} \left[\mathbf{d}_u^t \right] \right\Vert^2 \right] + q \beta \left(\eta_\mathrm{gl} \eta_\mathrm{ef}^t \right)^2  \sum_{u=0}^{U-1} (\alpha_u^t)^2 v_u^t \left(\mathbb{E}_{\pmb{\zeta}} [\mathbf{d}_u^t] \right)^2 + \nonumber\\
    &\bquad 2\beta \left(\eta_\mathrm{gl} \eta_\mathrm{ef}^t \right)^2  \sum_{u=0}^{U-1} v_u^t (\alpha_u^t)^2 \mathbb{E}_{\pmb{\zeta}} \left[ \left\Vert \frac{\eta_\mathrm{lo}}{\kappa_u^t} \sum_{\tau=0}^{\kappa_u^t-1} \left[ g_u^t (\mathbf{w}_u^{t,\tau}) - \nabla f_u^t (\mathbf{w}_u^{t,\tau}) \right] \right\Vert^2 \right] + \nonumber\\
    & \bquad 2\beta \left(\eta_\mathrm{gl} \eta_\mathrm{ef}^t \right)^2  \sum_{u=0}^{U-1} v_u^t (1-v_u^t)^2 (\alpha_u^t)^2 \mathbb{E}_{\pmb{\zeta}} \left[ \left\Vert \tilde{\mathbf{d}}_u^t  \right\Vert^2\right] + 2\beta \left(\eta_\mathrm{gl} \eta_\mathrm{ef}^t \right)^2  \sum_{u=0}^{U-1} (1 - v_u^t) (v_u^t \alpha_u^t)^2 \mathbb{E}_{\pmb{\zeta}} \left[ \left\Vert \tilde{\mathbf{d}}_u^t  \right\Vert^2\right] + \nonumber\\
    &\bquad \frac{\beta \left(\eta_\mathrm{gl} \eta_\mathrm{ef}^t \right)^2}{2} \left\Vert \sum_{u=0}^{U-1} \alpha_u^t v_u^t \tilde{\mathbf{d}}_u^t)\right\Vert^2\nonumber\\
    &= q \beta \left(\eta_\mathrm{gl} \eta_\mathrm{ef}^t \right)^2  \sum_{u=0}^{U-1} (\alpha_u^t)^2 v_u^t \mathbb{E}_{\pmb{\zeta}} \left[ \left\Vert \frac{\eta_\mathrm{lo}}{\kappa_u^t} \sum_{\tau=0}^{\kappa_u^t-1} \left[ g_u^t (\mathbf{w}_u^{t,\tau}) - \nabla f_u^t (\mathbf{w}_u^{t,\tau} \right] \right\Vert^2 \right] + q \beta \left(\eta_\mathrm{gl} \eta_\mathrm{ef}^t \right)^2  \sum_{u=0}^{U-1} (\alpha_u^t)^2 v_u^t \left\Vert \tilde{\mathbf{d}}_u^t \right\Vert^2 + \nonumber\\
    &\bquad 2\beta \left(\eta_\mathrm{gl} \eta_\mathrm{lo} \eta_\mathrm{ef}^t \right)^2 \sum_{u=0}^{U-1} \frac{v_u^t (\alpha_u^t)^2}{(\kappa_u^t)^2} \sum_{\tau=0}^{\kappa_u^t-1} \mathbb{E}_{\pmb{\zeta}} \left[ \left\Vert  g_u^t (\mathbf{w}_u^{t,\tau}) - \nabla f_u^t (\mathbf{w}_u^{t,\tau}) \right\Vert^2 \right] + \nonumber\\
    & \bquad 2\beta \left(\eta_\mathrm{gl} \eta_\mathrm{ef}^t \right)^2  \sum_{u=0}^{U-1} v_u^t (1-v_u^t) (\alpha_u^t)^2 \mathbb{E}_{\pmb{\zeta}} \left[ \left\Vert \tilde{\mathbf{d}}_u^t  \right\Vert^2\right] + \frac{\beta \left(\eta_\mathrm{gl} \eta_\mathrm{ef}^t \right)^2}{2} \left\Vert \sum_{u=0}^{U-1} \alpha_u^t v_u^t \tilde{\mathbf{d}}_u^t\right\Vert^2\nonumber\\
    &\overset{(d)}{\leq} q \beta \left(\eta_\mathrm{gl} \eta_\mathrm{lo} \eta_\mathrm{ef}^t \right)^2 \sum_{u=0}^{U-1} \frac{(\alpha_u^t)^2 v_u^t}{(\kappa_u^t)^2}  \sum_{\tau=0}^{\kappa_u^t-1} \sigma^2 + q \beta \left(\eta_\mathrm{gl} \eta_\mathrm{ef}^t \right)^2  \sum_{u=0}^{U-1} (\alpha_u^t)^2 v_u^t \left\Vert \tilde{\mathbf{d}}_u^t \right\Vert^2 + 2\beta \left(\eta_\mathrm{gl} \eta_\mathrm{lo} \eta_\mathrm{ef}^t \right)^2 \sum_{u=0}^{U-1} \frac{v_u^t (\alpha_u^t)^2}{(\kappa_u^t)^2} \sum_{\tau=0}^{\kappa_u^t-1} \sigma^2 + \nonumber\\
    & \bquad 2\beta \left(\eta_\mathrm{gl} \eta_\mathrm{ef}^t \right)^2  \sum_{u=0}^{U-1} v_u^t (1-v_u^t) (\alpha_u^t)^2 \bblue{\mathbb{E}_{\pmb{\zeta}}} \left[ \left\Vert \tilde{\mathbf{d}}_u^t \right\Vert^2 \right] + \frac{\beta \left(\eta_\mathrm{gl} \eta_\mathrm{ef}^t \right)^2}{2} \left\Vert \sum_{u=0}^{U-1} \alpha_u^t v_u^t \tilde{\mathbf{d}}_u^t\right\Vert^2\nonumber\\
    &=(2+q) \beta \left(\sigma \eta_\mathrm{gl} \eta_\mathrm{lo} \eta_\mathrm{ef}^t \right)^2 \sum_{u=0}^{U-1} \frac{(\alpha_u^t)^2 v_u^t}{\kappa_u^t} + \beta \left(\eta_\mathrm{gl} \eta_\mathrm{ef}^t \right)^2  \sum_{u=0}^{U-1} (\alpha_u^t)^2 v_u^t(1+q-v_u^t) \bblue{\mathbb{E}_{\pmb{\zeta}}} \left[\left\Vert \tilde{\mathbf{d}}_u^t \right\Vert^2 \right] + \frac{\beta \left(\eta_\mathrm{gl} \eta_\mathrm{ef}^t \right)^2}{2} \left\Vert \sum_{u=0}^{U-1} \alpha_u^t v_u^t \tilde{\mathbf{d}}_u^t\right\Vert^2 \nonumber\\
    &=(2+q) \beta \left(\sigma \eta_\mathrm{gl} \eta_\mathrm{lo} \eta_\mathrm{ef}^t \right)^2 \sum_{u=0}^{U-1} \frac{(\alpha_u^t)^2 v_u^t}{\kappa_u^t} + \beta \left(\eta_\mathrm{gl} \eta_\mathrm{ef}^t \right)^2 \sum_{u=0}^{U-1} v_u^t(1+q-v_u^t) (\alpha_u \delta_u^t)^2 \bblue{\mathbb{E}_{\pmb{\zeta}}} \left[ \left\Vert \frac{\eta_\mathrm{lo}} {\kappa_u^t} \sum_{\tau=0}^{\kappa_u^t-1} \nabla f_u^t (\mathbf{w}_u^{t,\tau}) \right\Vert^2 \right] + \nonumber\\
    &\bquad \frac{\beta \left(\eta_\mathrm{gl} \eta_\mathrm{ef}^t \right)^2}{2} \left\Vert \sum_{u=0}^{U-1} \alpha_u \delta_u^t v_u^t \frac{\eta_\mathrm{lo}}{\kappa_u^t} \sum_{\tau=0}^{\kappa_u^t-1} \nabla f_u^t (\mathbf{w}_u^{t,\tau}) \right\Vert^2 \nonumber\\
    &= (2+q) \beta \left(\sigma \eta_\mathrm{gl} \eta_\mathrm{lo} \eta_\mathrm{ef}^t \right)^2 \sum_{u=0}^{U-1} v_u^t \alpha_u^2 \frac{ (\delta_u^t)^2}{\kappa_u^t} + \beta \left(\eta_\mathrm{gl} \eta_\mathrm{lo} \eta_\mathrm{ef}^t \right)^2  \sum_{u=0}^{U-1} v_u^t(1+q-v_u^t) \left(\frac{\alpha_u \delta_u^t}{\kappa_u^t}\right)^2 \bblue{\mathbb{E}_{\pmb{\zeta}}} \left[ \left\Vert \sum_{\tau=0}^{\kappa_u^t-1} \nabla f_u^t (\mathbf{w}_u^{t,\tau}) \right\Vert^2 \right] + \nonumber\\
    &\bquad \frac{\beta \left(\eta_\mathrm{gl} \eta_\mathrm{lo} \eta_\mathrm{ef}^t \right)^2}{2} \left\Vert \sum_{u=0}^{U-1} \alpha_u v_u^t \left(\frac{\delta_u^t} {\kappa_u^t} \right) \sum_{\tau=0}^{\kappa_u^t-1} \nabla f_u^t (\mathbf{w}_u^{t,\tau}) \right\Vert^2,
\end{align}
where $(a)$ comes from the definition of variance. 
Besides, $(b)$ stems due to the cross-product term becoming zero in expectation.
In $(c)$, we used the bounded variance of the stochastic quantizer assumption, whereas, the bounded variance of the stochastic gradient assumptions were used in $(d)$.

Now, plugging $\mathrm{T}_1$ and $\mathrm{T}_2$ into (\ref{eq:main_conv_1}), we get
\begin{align}
\label{eq:main_conv_1_1} 
    &\mathbb{E}_{\pmb{1,\zeta},Q}\left[f^{t} (\mathbf{w}^{t+1})\right] - f^t(\mathbf{w}^t) 
    \leq - \frac{\eta_\mathrm{gl} \eta_\mathrm{lo} \eta_\mathrm{ef}^t}{2} \left\Vert \nabla f^t (\mathbf{w}^t) \right\Vert^2 - \frac{\eta_\mathrm{gl} \eta_\mathrm{lo} \eta_\mathrm{ef}^t}{2} \bblue{\mathbb{E}_{\pmb{\zeta}}} \left[\left\Vert \sum_{u=0}^{U-1} \alpha_u v_u^t \left( \frac{\delta_u^t} {\kappa_u^t} \right) \sum_{\tau=0}^{\kappa_u^t-1} \nabla f_u^t \left( \mathbf{w}_u^{t,\tau} \right) \right\Vert^2 \right] + \nonumber\\
    &\mquad \frac{\eta_\mathrm{gl} \eta_\mathrm{lo} \eta_\mathrm{ef}^t}{2} \bblue{\mathbb{E}_{\pmb{\zeta}}} \left[\left\Vert \nabla f^t (\mathbf{w}^t) - \sum_{u=0}^{U-1} \alpha_u v_u^t  \left( \frac{\delta_u^t} {\kappa_u^t} \right) \sum_{\tau=0}^{\kappa_u^t-1} \nabla f_u^t \left( \mathbf{w}_u^{t,\tau} \right) \right\Vert^2 \right] + (2+q) \beta \left(\sigma \eta_\mathrm{gl} \eta_\mathrm{lo} \eta_\mathrm{ef}^t \right)^2 \sum_{u=0}^{U-1} v_u^t \alpha_u^2 \frac{ (\delta_u^t)^2}{\kappa_u^t} + \nonumber\\
    &\mquad \beta \left(\eta_\mathrm{gl} \eta_\mathrm{lo} \eta_\mathrm{ef}^t \right)^2  \sum_{u=0}^{U-1} v_u^t(1+q-v_u^t) \left(\frac{\alpha_u \delta_u^t}{\kappa_u^t}\right)^2 \bblue{\mathbb{E}_{\pmb{\zeta}}} \left[ \left\Vert \sum_{\tau=0}^{\kappa_u^t-1} \nabla f_u^t (\mathbf{w}_u^{t,\tau}) \right\Vert^2 \right] + \nonumber\\
    &\mquad \frac{\beta \left(\eta_\mathrm{gl} \eta_\mathrm{lo} \eta_\mathrm{ef}^t \right)^2}{2} \left\Vert \sum_{u=0}^{U-1} \alpha_u v_u^t \left(\frac{\delta_u^t} {\kappa_u^t} \right) \sum_{\tau=0}^{\kappa_u^t-1} \nabla f_u^t (\mathbf{w}_u^{t,\tau}) \right\Vert^2 \nonumber\\
    &= (2+q) \beta \left(\sigma \eta_\mathrm{gl} \eta_\mathrm{lo} \eta_\mathrm{ef}^t \right)^2 \sum_{u=0}^{U-1} v_u^t \alpha_u^2 \frac{ (\delta_u^t)^2}{\kappa_u^t} - \frac{\eta_\mathrm{gl} \eta_\mathrm{lo} \eta_\mathrm{ef}^t}{2} \left\Vert \nabla f^t (\mathbf{w}^t) \right\Vert^2 + \nonumber\\
    &\bquad \frac{\eta_\mathrm{gl} \eta_\mathrm{lo} \eta_\mathrm{ef}^t}{2} \bblue{\mathbb{E}_{\pmb{\zeta}}} \left[\left\Vert \nabla f^t (\mathbf{w}^t) - \sum_{u=0}^{U-1} \alpha_u v_u^t  \left( \frac{\delta_u^t} {\kappa_u^t} \right) \sum_{\tau=0}^{\kappa_u^t-1} \nabla f_u^t \left( \mathbf{w}_u^{t,\tau} \right) \right\Vert^2 \right] +  \nonumber\\
    &\bquad \beta \left(\eta_\mathrm{gl} \eta_\mathrm{lo} \eta_\mathrm{ef}^t \right)^2 \sum_{u=0}^{U-1} v_u^t(1+q-v_u^t) \left(\frac{\alpha_u \delta_u^t}{\kappa_u^t}\right)^2 \bblue{\mathbb{E}_{\pmb{\zeta}}} \left[ \left\Vert \sum_{\tau=0}^{\kappa_u^t-1} \nabla f_u^t (\mathbf{w}_u^{t,\tau}) \right\Vert^2 \right] \nonumber\\
    &\bquad - \frac{\eta_\mathrm{gl} \eta_\mathrm{lo} \eta_\mathrm{ef}^t}{2} (1 - \beta \eta_\mathrm{gl} \eta_\mathrm{lo} \eta_\mathrm{ef}^t ) \left\Vert \sum_{u=0}^{U-1} \alpha_u v_u^t \left( \frac{\delta_u^t} {\kappa_u^t} \right) \sum_{\tau=0}^{\kappa_u^t-1} \nabla f_u^t \left( \mathbf{w}_u^{t,\tau} \right) \right\Vert^2,
\end{align}
where, in the last term, we get $(1 -\beta\eta_\mathrm{gl}\eta_\mathrm{lo} \eta_\mathrm{ef}^t) \geq 0$, when $\eta_\mathrm{gl} \eta_\mathrm{lo} \leq \frac{1}{\beta \eta_\mathrm{ef}^t}$, which makes this term negative.
Thus, we drop this and re-write the bound as
\begin{align}
\label{eq:main_conv_1_1_0} 
    \mathbb{E}_{\pmb{1,\zeta},Q}\left[f^{t} (\mathbf{w}^{t+1})\right] - f^t(\mathbf{w}^t) 
    &\leq (2+q) \beta \left(\sigma \eta_\mathrm{gl} \eta_\mathrm{lo} \eta_\mathrm{ef}^t \right)^2 \sum_{u=0}^{U-1} v_u^t \alpha_u^2 \frac{ (\delta_u^t)^2}{\kappa_u^t} - \frac{\eta_\mathrm{gl} \eta_\mathrm{lo} \eta_\mathrm{ef}^t}{2} \left\Vert \nabla f^t (\mathbf{w}^t) \right\Vert^2 + \nonumber\\
    &\frac{\eta_\mathrm{gl} \eta_\mathrm{lo} \eta_\mathrm{ef}^t}{2} \bblue{\mathbb{E}_{\pmb{\zeta}}} \left[\left\Vert \nabla f^t (\mathbf{w}^t) - \sum_{u=0}^{U-1} \alpha_u v_u^t  \left( \frac{\delta_u^t} {\kappa_u^t} \right) \sum_{\tau=0}^{\kappa_u^t-1} \nabla f_u^t \left( \mathbf{w}_u^{t,\tau} \right) \right\Vert^2 \right] +  \nonumber\\
    &\beta \left(\eta_\mathrm{gl} \eta_\mathrm{lo} \eta_\mathrm{ef}^t \right)^2  \sum_{u=0}^{U-1} v_u^t(1+q-v_u^t) \left(\frac{\alpha_u \delta_u^t}{\kappa_u^t}\right)^2 \bblue{\mathbb{E}_{\pmb{\zeta}}} \left[ \left\Vert \sum_{\tau=0}^{\kappa_u^t-1} \nabla f_u^t (\mathbf{w}_u^{t,\tau}) \right\Vert^2 \right] \nonumber\\
    &\leq (2+q) \beta \left(\sigma \eta_\mathrm{gl} \eta_\mathrm{lo} \eta_\mathrm{ef}^t \right)^2 \sum_{u=0}^{U-1} v_u^t \alpha_u^2 \frac{ (\delta_u^t)^2}{\kappa_u^t} - \frac{\eta_\mathrm{gl} \eta_\mathrm{lo} \eta_\mathrm{ef}^t}{2} \left\Vert \nabla f^t (\mathbf{w}^t) \right\Vert^2 + \nonumber\\
    &\underbrace{\frac{\eta_\mathrm{gl} \eta_\mathrm{lo} \eta_\mathrm{ef}^t}{2} \bblue{\mathbb{E}_{\pmb{\zeta}}} \left[\left\Vert \nabla f^t (\mathbf{w}^t) - \sum_{u=0}^{U-1} \alpha_u v_u^t  \left( \frac{\delta_u^t} {\kappa_u^t} \right) \sum_{\tau=0}^{\kappa_u^t-1} \nabla f_u^t \left( \mathbf{w}_u^{t,\tau} \right) \right\Vert^2 \right]}_\mathrm{T_3} +  \nonumber\\
    &\beta \left(\eta_\mathrm{gl} \eta_\mathrm{lo} \eta_\mathrm{ef}^t \right)^2  \sum_{u=0}^{U-1} v_u^t(1+q-v_u^t) \left(\alpha_u \delta_u^t\right)^2 \sum_{\tau=0}^{\kappa_u^t-1} \underbrace{\bblue{\mathbb{E}_{\pmb{\zeta}}} \left[ \left\Vert \nabla f_u^t (\mathbf{w}_u^{t,\tau}) \right\Vert^2 \right]}_\mathrm{T_4},
\end{align}
where the last inequality is true since $\Vert \sum_{i=0}^{I-1} \mathbf{x}_i \Vert^2 \leq I \sum_{i=0}^{I-1} \Vert \mathbf{x}_i \Vert^2$ from Cauchy-Schwarz inequality.

We now simplify $\mathrm{T}_3$ as 
\begin{align}
    \mathrm{T}_3 
    &= \frac{\eta_\mathrm{gl} \eta_\mathrm{lo} \eta_\mathrm{ef}^t}{2} \bblue{\mathbb{E}_{\pmb{\zeta}}} \left[\left\Vert \nabla f^t (\mathbf{w}^t) - \sum_{u=0}^{U-1} \alpha_u v_u^t \left(\frac{\delta_u^t} {\kappa_u^t} \right)\sum_{\tau=0}^{\kappa_u^t-1} \nabla f_u^t \left( \mathbf{w}_u^{t,\tau} \right) \right\Vert^2\right] \nonumber\\
    &\overset{(a)}{=} \frac{\eta_\mathrm{gl} \eta_\mathrm{lo} \eta_\mathrm{ef}^t}{2} \bblue{\mathbb{E}_{\pmb{\zeta}}} \left[\left\Vert \sum_{u=0}^{U-1} \alpha_u^t \nabla f_u^t (\mathbf{w}^t) - \sum_{u=0}^{U-1} \alpha_u v_u^t \left(\frac{\delta_u^t} {\kappa_u^t} \right)\sum_{\tau=0}^{\kappa_u^t-1} \nabla f_u^t \left( \mathbf{w}_u^{t,\tau} \right) \right\Vert^2\right] \nonumber\\
    &= \frac{\eta_\mathrm{gl} \eta_\mathrm{lo} \eta_\mathrm{ef}^t}{2} \bblue{\mathbb{E}_{\pmb{\zeta}}} \left[\left\Vert \sum_{u=0}^{U-1} \alpha_u \delta_u^t \nabla f_u^t (\mathbf{w}^t) - \sum_{u=0}^{U-1} \alpha_u v_u^t \left(\frac{\delta_u^t} {\kappa_u^t} \right)\sum_{\tau=0}^{\kappa_u^t-1} \nabla f_u^t \left( \mathbf{w}_u^{t,\tau} \right) \right\Vert^2\right] \nonumber \\
    &= \frac{\eta_\mathrm{gl} \eta_\mathrm{lo} \eta_\mathrm{ef}^t}{2} \bblue{\mathbb{E}_{\pmb{\zeta}}} \left[\left\Vert \sum_{u=0}^{U-1} \alpha_u \delta_u^t \nabla f_u^t (\mathbf{w}^t) - \sum_{u=0}^{U-1} \alpha_u \delta_u^t v_u^t \nabla f_u^t (\mathbf{w}^t) + \sum_{u=0}^{U-1} \alpha_u v_u^t \delta_u^t \nabla f_u^t (\mathbf{w}^t) - \sum_{u=0}^{U-1} \alpha_u v_u^t \left(\frac{\delta_u^t} {\kappa_u^t} \right)\sum_{\tau=0}^{\kappa_u^t-1} \nabla f_u^t \left( \mathbf{w}_u^{t,\tau} \right) \right\Vert^2\right] \nonumber \\
    &\overset{(b)}{\leq} \eta_\mathrm{gl} \eta_\mathrm{lo} \eta_\mathrm{ef}^t \bblue{\mathbb{E}_{\pmb{\zeta}}} \left[\left\Vert \sum_{u=0}^{U-1} \alpha_u \delta_u^t \nabla f_u^t (\mathbf{w}^t) - \sum_{u=0}^{U-1} \alpha_u \delta_u^t v_u^t \nabla f_u^t (\mathbf{w}^t)\right\Vert^2\right] + \nonumber\\
    &\mquad \eta_\mathrm{gl} \eta_\mathrm{lo} \eta_\mathrm{ef}^t \bblue{\mathbb{E}_{\pmb{\zeta}}} \left[\left\Vert \sum_{u=0}^{U-1} \alpha_u v_u^t \delta_u^t \nabla f_u^t (\mathbf{w}^t) - \sum_{u=0}^{U-1} \alpha_u v_u^t \left(\frac{\delta_u^t} {\kappa_u^t} \right)\sum_{\tau=0}^{\kappa_u^t-1} \nabla f_u^t \left( \mathbf{w}_u^{t,\tau} \right) \right\Vert^2\right] \nonumber \\
    &= \eta_\mathrm{gl} \eta_\mathrm{lo} \eta_\mathrm{ef}^t \bblue{\mathbb{E}_{\pmb{\zeta}}} \left[\left\Vert \sum_{u=0}^{U-1} \alpha_u \delta_u^t (1 - v_u^t) \nabla f_u^t (\mathbf{w}^t)\right\Vert^2\right] + \eta_\mathrm{gl} \eta_\mathrm{lo} \eta_\mathrm{ef}^t \bblue{\mathbb{E}_{\pmb{\zeta}}} \left[\left\Vert \sum_{u=0}^{U-1} \alpha_u v_u^t \left(\frac{\delta_u^t} {\kappa_u^t} \right)\sum_{\tau=0}^{\kappa_u^t-1} \left(\nabla f_u^t (\mathbf{w}^t) -  \nabla f_u^t \left( \mathbf{w}_u^{t,\tau} \right) \right) \right\Vert^2\right] \nonumber \\
    &\overset{(c)}{\leq} \eta_\mathrm{gl} \eta_\mathrm{lo} \eta_\mathrm{ef}^t \sum_{u=0}^{U-1} \alpha_u \bblue{\mathbb{E}_{\pmb{\zeta}}} \left[\left\Vert  \delta_u^t (1 - v_u^t) \nabla f_u^t (\mathbf{w}^t)\right\Vert^2\right] + \eta_\mathrm{gl} \eta_\mathrm{lo} \eta_\mathrm{ef}^t \sum_{u=0}^{U-1} \alpha_u \bblue{\mathbb{E}_{\pmb{\zeta}}} \left[\left\Vert v_u^t \left(\frac{\delta_u^t} {\kappa_u^t} \right)\sum_{\tau=0}^{\kappa_u^t-1} \left(\nabla f_u^t (\mathbf{w}^t) -  \nabla f_u^t \left( \mathbf{w}_u^{t,\tau} \right) \right) \right\Vert^2\right] \nonumber \\
    &=\eta_\mathrm{gl} \eta_\mathrm{lo} \eta_\mathrm{ef}^t \sum_{u=0}^{U-1} \alpha_u ( \delta_u^t (1 - v_u^t) )^2\bblue{\mathbb{E}_{\pmb{\zeta}}} \left[\left\Vert \nabla f_u^t (\mathbf{w}^t)\right\Vert^2\right] + \eta_\mathrm{gl} \eta_\mathrm{lo} \eta_\mathrm{ef}^t \sum_{u=0}^{U-1} \alpha_u \left[v_u^t \left(\frac{\delta_u^t} {\kappa_u^t} \right) \right]^2 \bblue{\mathbb{E}_{\pmb{\zeta}}} \left[\left\Vert \sum_{\tau=0}^{\kappa_u^t-1} \left(\nabla f_u^t (\mathbf{w}^t) -  \nabla f_u^t \left( \mathbf{w}_u^{t,\tau} \right) \right) \right\Vert^2\right] \nonumber \\
    &\overset{(d)}{\leq} \eta_\mathrm{gl} \eta_\mathrm{lo} \eta_\mathrm{ef}^t \sum_{u=0}^{U-1} \alpha_u ( \delta_u^t (1 - v_u^t) )^2\bblue{\mathbb{E}_{\pmb{\zeta}}} \left[\left\Vert \nabla f_u^t (\mathbf{w}^t) - \nabla f^t (\mathbf{w}^t) + \nabla f^t (\mathbf{w}^t) \right\Vert^2\right] + \nonumber\\
    &\mquad \eta_\mathrm{gl} \eta_\mathrm{lo} \eta_\mathrm{ef}^t \sum_{u=0}^{U-1} \alpha_u \left[v_u^t \left(\frac{\delta_u^t} {\kappa_u^t} \right) \right]^2 \cdot \kappa_u^t \sum_{\tau=0}^{\kappa_u^t-1} \bblue{\mathbb{E}_{\pmb{\zeta}}} \left[\left\Vert \nabla f_u^t (\mathbf{w}^t) -  \nabla f_u^t \left( \mathbf{w}_u^{t,\tau} \right) \right\Vert^2\right] \nonumber \\
    &\overset{(e)}{\leq} 2\eta_\mathrm{gl} \eta_\mathrm{lo} \eta_\mathrm{ef}^t \sum_{u=0}^{U-1} \alpha_u ( \delta_u^t (1 - v_u^t) )^2\bblue{\mathbb{E}_{\pmb{\zeta}}} \left[\left\Vert \nabla f_u^t (\mathbf{w}^t) - \nabla f^t (\mathbf{w}^t) \right\Vert^2\right] + 2\eta_\mathrm{gl} \eta_\mathrm{lo} \eta_\mathrm{ef}^t \sum_{u=0}^{U-1} \alpha_u ( \delta_u^t (1 - v_u^t) )^2\bblue{\mathbb{E}_{\pmb{\zeta}}} \left[\left\Vert \nabla f^t (\mathbf{w}^t) \right\Vert^2\right] + \nonumber\\
    &\mquad \eta_\mathrm{gl} \eta_\mathrm{lo} \eta_\mathrm{ef}^t \beta^2 \sum_{u=0}^{U-1} \alpha_u (v_u^t\delta_u^t)^2 \frac{1} {\kappa_u^t} \sum_{\tau=0}^{\kappa_u^t-1} \bblue{\mathbb{E}_{\pmb{\zeta}}} \left[\left\Vert \mathbf{w}^t - \mathbf{w}_u^{t,\tau} \right\Vert^2\right] \nonumber \\
    &\overset{(f)}{\leq} 2\eta_\mathrm{gl} \eta_\mathrm{lo} \eta_\mathrm{ef}^t \varpi^2 \sum_{u=0}^{U-1} \alpha_u ( \delta_u^t (1 - v_u^t) )^2 + 2\eta_\mathrm{gl} \eta_\mathrm{lo} \eta_\mathrm{ef}^t \sum_{u=0}^{U-1} \alpha_u ( \delta_u^t (1 - v_u^t) )^2 \left\Vert \nabla f^t (\mathbf{w}^t) \right\Vert^2 + \nonumber\\
    &\mquad \eta_\mathrm{gl} \eta_\mathrm{lo} \eta_\mathrm{ef}^t \beta^2 \sum_{u=0}^{U-1} \alpha_u (v_u^t\delta_u^t)^2 \frac{1} {\kappa_u^t} \sum_{\tau=0}^{\kappa_u^t-1} \underbrace{\bblue{\mathbb{E}_{\pmb{\zeta}}} \left[\left\Vert \mathbf{w}^t - \mathbf{w}_u^{t,\tau} \right\Vert^2\right]}_{\mathrm{T}_5},
\end{align}
where we used the fact that global loss function $f^t(\mathbf{w}^t) = \sum_{u=0}^{U-1}\alpha_u^t f_u^t(\mathbf{w}^t)$, by definition, in $(a)$.
$(b)$ appears from Cauchy-Schwarz inequality, whereas $(c)$ comes from Jensen inequality due to the convexity of $\Vert \cdot \Vert$ and the fact that $\Vert \sum_{i=0}^{I-1} \alpha_i \mathbf{x}_i \Vert^2 \leq \sum_{i=0}^{I-1} \alpha_i \Vert \mathbf{x}_i \Vert^2$ since $\sum_{i=0}^{I-1}\alpha_i=1$.
Besides, $(d)$ and $(e)$ appears from Cauchy-Schwarz inequality. 
In $(e)$, we also used the $\beta$-smoothness assumption.
Finally, the last inequality arose from the bounded gradient divergence assumption.

Now, using Lemma \ref{lema:model_drfit_per_sgd} to replace $\mathrm{T}_5$, we get 
\begin{align}
    \mathrm{T}_3 
    &= \frac{\eta_\mathrm{gl} \eta_\mathrm{lo} \eta_\mathrm{ef}^t}{2} \bblue{\mathbb{E}_{\pmb{\zeta}}} \left[\left\Vert \nabla f^t (\mathbf{w}^t) - \sum_{u=0}^{U-1} \alpha_u v_u^t \left(\frac{\delta_u^t} {\kappa_u^t} \right)\sum_{\tau=0}^{\kappa_u^t-1} \nabla f_u^t \left( \mathbf{w}_u^{t,\tau} \right) \right\Vert^2\right] \nonumber\\
    &\leq 
    2\eta_\mathrm{gl} \eta_\mathrm{lo} \eta_\mathrm{ef}^t \varpi^2 \sum_{u=0}^{U-1} \alpha_u ( \delta_u^t (1 - v_u^t) )^2 + 2\eta_\mathrm{gl} \eta_\mathrm{lo} \eta_\mathrm{ef}^t \sum_{u=0}^{U-1} \alpha_u ( \delta_u^t (1 - v_u^t) )^2 \left\Vert \nabla f^t (\mathbf{w}^t) \right\Vert^2 + \nonumber\\
    &\mquad \eta_\mathrm{gl} \eta_\mathrm{lo} \eta_\mathrm{ef}^t \beta^2 \sum_{u=0}^{U-1} \alpha_u (v_u^t\delta_u^t)^2 \frac{1} {\kappa_u^t} \sum_{\tau=0}^{\kappa_u^t-1} \Big\{4 \kappa_u^t \eta_\mathrm{lo}^2 \sigma^2 + 24 \left(\kappa_u^t \varpi \eta_\mathrm{lo}\right)^2 + 24 \left(\kappa_u^t \eta_\mathrm{lo}\right)^2 \left\Vert \nabla f^{t}(\mathbf{w}^t) \right\Vert^2 \Big\} \nonumber\\
    &= 2\eta_\mathrm{gl} \eta_\mathrm{lo} \eta_\mathrm{ef}^t \varpi^2 \sum_{u=0}^{U-1} \alpha_u ( \delta_u^t (1 - v_u^t) )^2 + 2\eta_\mathrm{gl} \eta_\mathrm{lo} \eta_\mathrm{ef}^t \sum_{u=0}^{U-1} \alpha_u ( \delta_u^t (1 - v_u^t) )^2 \left\Vert \nabla f^t (\mathbf{w}^t) \right\Vert^2 + \nonumber\\
    &\mquad \eta_\mathrm{gl} \eta_\mathrm{lo} \eta_\mathrm{ef}^t \beta^2 \sum_{u=0}^{U-1} \alpha_u (v_u^t\delta_u^t)^2 \Big\{4 \kappa_u^t \eta_\mathrm{lo}^2 \sigma^2 + 24 \left(\kappa_u^t \varpi \eta_\mathrm{lo}\right)^2 + 24 \left(\kappa_u^t \eta_\mathrm{lo}\right)^2 \left\Vert \nabla f^{t}(\mathbf{w}^t) \right\Vert^2 \Big\} \nonumber\\
    &= 2\eta_\mathrm{gl} \eta_\mathrm{lo} \eta_\mathrm{ef}^t \varpi^2 \sum_{u=0}^{U-1} \alpha_u ( \delta_u^t (1 - v_u^t) )^2 + 24 \eta_\mathrm{gl} \eta_\mathrm{ef}^t \left(\beta \varpi \right)^2 \eta_\mathrm{lo}^3  \sum_{u=0}^{U-1} \alpha_u (v_u^t\delta_u^t)^2  \left(\kappa_u^t\right)^2 + 4 \eta_\mathrm{gl} \eta_\mathrm{ef}^t \beta^2 \sigma^2 \eta_\mathrm{lo}^3 \sum_{u=0}^{U-1} \alpha_u \kappa_u^t (v_u^t\delta_u^t)^2 + \nonumber\\
    &\squad 2\eta_\mathrm{gl} \eta_\mathrm{lo} \eta_\mathrm{ef}^t \sum_{u=0}^{U-1} \alpha_u ( \delta_u^t (1 - v_u^t) )^2 \left\Vert \nabla f^t (\mathbf{w}^t) \right\Vert^2 + 24 \eta_\mathrm{gl} \eta_\mathrm{ef}^t \beta^2 \eta_\mathrm{lo}^3 \sum_{u=0}^{U-1} \alpha_u (\kappa_u^t v_u^t \delta_u^t)^2 \left\Vert \nabla f^{t}(\mathbf{w}^t) \right\Vert^2.
\end{align}

Now, we simplify $\mathrm{T}_4$ as follows
\begin{align}
    \mathrm{T}_4 
    &= \bblue{\mathbb{E}_{\pmb{\zeta}}} \left[ \left\Vert \nabla f_u^t (\mathbf{w}_u^{t,\tau}) \right\Vert^2 \right] \\ \nonumber
    &= \bblue{\mathbb{E}_{\pmb{\zeta}}} \left[ \left\Vert \nabla f_u^t \left(\mathbf{w}_u^{t,\tau} \right) - \nabla f_u^t \left(\mathbf{w}^{t} \right) + \nabla f_u^t \left(\mathbf{w}^{t} \right) - \nabla f^t \left(\mathbf{w}^{t} \right) + \nabla f^t \left(\mathbf{w}^{t} \right) \right\Vert^2 \right] \nonumber\\
    &\leq 3\bblue{\mathbb{E}_{\pmb{\zeta}}} \left[\left\Vert  \nabla f_u^t \left(\mathbf{w}_u^{t,\tau} \right) - \nabla f_u^t \left(\mathbf{w}^{t} \right) \right\Vert^2\right] + 3 \bblue{\mathbb{E}_{\pmb{\zeta}}} \left[ \left\Vert \nabla f_u^t \left(\mathbf{w}^{t} \right) - \nabla f^t \left(\mathbf{w}^{t} \right) \right\Vert^2 \right] + 3 \left\Vert \nabla f^t \left(\mathbf{w}^{t} \right) \right\Vert^2 \nonumber\\
    &\leq 3\beta^2 \bblue{\mathbb{E}_{\pmb{\zeta}}} \left[\left\Vert \mathbf{w}_u^{t,\tau} - \mathbf{w}^{t} \right\Vert^2\right] + 3\varpi^2 + 3\left\Vert \nabla f^t \left(\mathbf{w}^{t} \right) \right\Vert^2 \nonumber\\
    &\overset{(a)}{\leq} 3\beta^2 \left[ 4 \kappa_u^t \eta_\mathrm{lo}^2 \sigma^2 + 24 \left(\kappa_u^t \varpi \eta_\mathrm{lo}\right)^2 + 24 \left(\kappa_u^t \eta_\mathrm{lo}\right)^2 \left\Vert \nabla f^{t}(\mathbf{w}^t) \right\Vert^2 \right] + 3\varpi^2 + 3 \left\Vert \nabla f^t \left(\mathbf{w}^{t} \right) \right\Vert^2 \nonumber\\
    &= 12 \kappa_u^t \eta_\mathrm{lo}^2 \beta^2 \sigma^2 + 72 \left(\kappa_u^t \beta \varpi \eta_\mathrm{lo}\right)^2 + 72 \left(\beta \eta_\mathrm{lo} \kappa_u^t \right)^2 \left\Vert \nabla f^{t}(\mathbf{w}^t) \right\Vert^2 + 3\varpi^2 + 3 \left\Vert \nabla f^t \left(\mathbf{w}^{t} \right) \right\Vert^2 \nonumber\\
    &= 3\varpi^2 + 12 \kappa_u^t \eta_\mathrm{lo}^2 \beta^2 \sigma^2 + 72 \left(\kappa_u^t \beta \varpi \eta_\mathrm{lo}\right)^2 + 3[1 + 24 \left(\beta \eta_\mathrm{lo} \kappa_u^t \right)^2 ]\left\Vert \nabla f^{t}(\mathbf{w}^t) \right\Vert^2,
\end{align}
where $(a)$ appears from (\ref{lema:model_drfit_per_sgd}).

Now, plugging $\mathrm{T}_3$ and $\mathrm{T}_4$ into (\ref{eq:main_conv_1_1_0}), we 
\begin{align}
\label{eq:main_conv_2}
    &\mathbb{E}_{\pmb{1,\zeta},Q}\left[f^{t} (\mathbf{w}^{t+1})\right] - f^t(\mathbf{w}^t) 
    \leq (2+q) \beta \left(\sigma \eta_\mathrm{gl} \eta_\mathrm{lo} \eta_\mathrm{ef}^t \right)^2 \sum_{u=0}^{U-1} v_u^t \alpha_u^2 \frac{ (\delta_u^t)^2}{\kappa_u^t} - \frac{\eta_\mathrm{gl} \eta_\mathrm{lo} \eta_\mathrm{ef}^t}{2} \left\Vert \nabla f^t (\mathbf{w}^t) \right\Vert^2 + \nonumber\\
    & 2\eta_\mathrm{gl} \eta_\mathrm{lo} \eta_\mathrm{ef}^t \varpi^2 \sum_{u=0}^{U-1} \alpha_u ( \delta_u^t (1 - v_u^t) )^2 + 24 \eta_\mathrm{gl} \eta_\mathrm{ef}^t \left(\beta \varpi \right)^2 \eta_\mathrm{lo}^3  \sum_{u=0}^{U-1} \alpha_u (v_u^t\delta_u^t)^2  \left(\kappa_u^t\right)^2 + 4 \eta_\mathrm{gl} \eta_\mathrm{ef}^t \beta^2 \sigma^2 \eta_\mathrm{lo}^3 \sum_{u=0}^{U-1} \alpha_u \kappa_u^t (v_u^t\delta_u^t)^2 + \nonumber\\
    &2\eta_\mathrm{gl} \eta_\mathrm{lo} \eta_\mathrm{ef}^t \sum_{u=0}^{U-1} \alpha_u ( \delta_u^t (1 - v_u^t) )^2 \left\Vert \nabla f^t (\mathbf{w}^t) \right\Vert^2 + 24 \eta_\mathrm{gl} \eta_\mathrm{ef}^t \beta^2 \eta_\mathrm{lo}^3 \sum_{u=0}^{U-1} \alpha_u (\kappa_u^t v_u^t \delta_u^t)^2 \left\Vert \nabla f^{t}(\mathbf{w}^t) \right\Vert^2 +  \nonumber\\
    &\beta \left(\eta_\mathrm{gl} \eta_\mathrm{lo} \eta_\mathrm{ef}^t \right)^2  \sum_{u=0}^{U-1} v_u^t(1+q-v_u^t) \left(\alpha_u \delta_u^t\right)^2 \sum_{\tau=0}^{\kappa_u^t-1} \Big\{3\varpi^2 + 12 \kappa_u^t \eta_\mathrm{lo}^2 \beta^2 \sigma^2 + 72 \left(\kappa_u^t \beta \varpi \eta_\mathrm{lo}\right)^2 + 3[1 + 24 \left(\beta \eta_\mathrm{lo} \kappa_u^t \right)^2 ]\left\Vert \nabla f^{t}(\mathbf{w}^t) \right\Vert^2 \Big\} \nonumber\\
    &=- \frac{\eta_\mathrm{gl} \eta_\mathrm{lo} \eta_\mathrm{ef}^t}{2} \left\Vert \nabla f^t (\mathbf{w}^t) \right\Vert^2 + 2\eta_\mathrm{gl} \eta_\mathrm{lo} \eta_\mathrm{ef}^t \sum_{u=0}^{U-1} \alpha_u ( \delta_u^t (1 - v_u^t) )^2 \left\Vert \nabla f^t (\mathbf{w}^t) \right\Vert^2 + 24 \eta_\mathrm{gl} \eta_\mathrm{ef}^t \beta^2 \eta_\mathrm{lo}^3 \sum_{u=0}^{U-1} \alpha_u (\kappa_u^t v_u^t \delta_u^t)^2 \left\Vert \nabla f^{t}(\mathbf{w}^t) \right\Vert^2 + \nonumber\\
    &(2+q) \beta \left(\sigma \eta_\mathrm{gl} \eta_\mathrm{lo} \eta_\mathrm{ef}^t \right)^2 \sum_{u=0}^{U-1} v_u^t \alpha_u^2 \frac{ (\delta_u^t)^2}{\kappa_u^t} + 4 \eta_\mathrm{gl} \eta_\mathrm{ef}^t \beta^2 \sigma^2 \eta_\mathrm{lo}^3 \sum_{u=0}^{U-1} \alpha_u \kappa_u^t (v_u^t\delta_u^t)^2 + \nonumber\\
    & 2\eta_\mathrm{gl} \eta_\mathrm{lo} \eta_\mathrm{ef}^t \varpi^2 \sum_{u=0}^{U-1} \alpha_u ( \delta_u^t (1 - v_u^t) )^2 + 24 \eta_\mathrm{gl} \eta_\mathrm{ef}^t \left(\beta \varpi \right)^2 \eta_\mathrm{lo}^3  \sum_{u=0}^{U-1} \alpha_u (v_u^t\delta_u^t)^2  \left(\kappa_u^t\right)^2 + \nonumber\\
    &\beta \left(\eta_\mathrm{gl} \eta_\mathrm{lo} \eta_\mathrm{ef}^t \right)^2  \sum_{u=0}^{U-1} v_u^t(1+q-v_u^t) \left(\alpha_u \delta_u^t\right)^2 \kappa_u^t \Big\{3\varpi^2 + 12 \kappa_u^t \eta_\mathrm{lo}^2 \beta^2 \sigma^2 + 72 \left(\kappa_u^t \beta \varpi \eta_\mathrm{lo}\right)^2 + 3[1 + 24 \left(\beta \eta_\mathrm{lo} \kappa_u^t \right)^2 ]\left\Vert \nabla f^{t}(\mathbf{w}^t) \right\Vert^2 \Big\} \nonumber\\
    &\overset{(a)}{=}- \frac{\eta_\mathrm{gl} \eta_\mathrm{lo} \eta_\mathrm{ef}^t}{2} \left\Vert \nabla f^t (\mathbf{w}^t) \right\Vert^2 + 2\eta_\mathrm{gl} \eta_\mathrm{lo} \eta_\mathrm{ef}^t \sum_{u=0}^{U-1} \alpha_u ( \delta_u^t (1 - v_u^t) )^2 \left\Vert \nabla f^t (\mathbf{w}^t) \right\Vert^2 + 24 \eta_\mathrm{gl} \eta_\mathrm{ef}^t \beta^2 \eta_\mathrm{lo}^3 \sum_{u=0}^{U-1} \alpha_u (\kappa_u^t v_u^t \delta_u^t)^2 \left\Vert \nabla f^{t}(\mathbf{w}^t) \right\Vert^2 + \nonumber\\
    & 3 \beta \left(\eta_\mathrm{gl} \eta_\mathrm{lo} \eta_\mathrm{ef}^t \right)^2 \sum_{u=0}^{U-1} v_u^t (1+q-v_u^t) \kappa_u^t \left(\alpha_u \delta_u^t\right)^2 \left\Vert \nabla f^{t}(\mathbf{w}^t) \right\Vert^2 + 72 \left(\eta_\mathrm{gl} \eta_\mathrm{ef}^t \right)^2 \beta^3 \eta_\mathrm{lo}^4 \sum_{u=0}^{U-1} v_u^t(1+q-v_u^t) \left(\alpha_u \delta_u^t\right)^2 \left(\kappa_u^t \right)^3 \left\Vert \nabla f^{t}(\mathbf{w}^t) \right\Vert^2 + \nonumber\\
    &\digamma_\sigma (\sigma, \pmb{v, \kappa^t,\delta^t}) + \digamma_\varpi (\sigma, \pmb{v,\kappa^t,\delta^t}) \nonumber\\
    &=\Big[- \frac{\eta_\mathrm{gl} \eta_\mathrm{lo} \eta_\mathrm{ef}^t}{2} + 2\eta_\mathrm{gl} \eta_\mathrm{lo} \eta_\mathrm{ef}^t \sum_{u=0}^{U-1} \alpha_u ( \delta_u^t (1 - v_u^t) )^2 + 24 \eta_\mathrm{gl} \eta_\mathrm{ef}^t \beta^2 \eta_\mathrm{lo}^3 \sum_{u=0}^{U-1} \alpha_u (\kappa_u^t v_u^t \delta_u^t)^2 + \nonumber\\
    &3 \beta \left(\eta_\mathrm{gl} \eta_\mathrm{lo} \eta_\mathrm{ef}^t \right)^2 \sum_{u=0}^{U-1} v_u^t (1+q-v_u^t) \kappa_u^t \left(\alpha_u \delta_u^t\right)^2 + 72 \left(\eta_\mathrm{gl} \eta_\mathrm{ef}^t \right)^2 \beta^3 \eta_\mathrm{lo}^4 \sum_{u=0}^{U-1} v_u^t(1+q-v_u^t) \left(\alpha_u \delta_u^t\right)^2 \left(\kappa_u^t \right)^3 \Big] \left\Vert \nabla f^{t}(\mathbf{w}^t) \right\Vert^2 + \nonumber\\
    &\digamma_\sigma (\sigma, \pmb{v, \kappa^t,\delta^t}) + \digamma_\varpi (\sigma, \pmb{v,\kappa^t,\delta^t}) \nonumber\\
    &=- \frac{\eta_\mathrm{gl} \eta_\mathrm{lo} \eta_\mathrm{ef}^t}{2} \Big[1 - 4 \sum_{u=0}^{U-1} \alpha_u ( \delta_u^t (1 - v_u^t) )^2 - 48 \beta^2 \eta_\mathrm{lo}^2 \sum_{u=0}^{U-1} \alpha_u (\kappa_u^t v_u^t \delta_u^t)^2 - 6 \beta \eta_\mathrm{gl} \eta_\mathrm{lo} \eta_\mathrm{ef}^t \sum_{u=0}^{U-1} v_u^t (1+q-v_u^t) \kappa_u^t \left(\alpha_u \delta_u^t\right)^2 - \nonumber\\
    & \mquad 144 \eta_\mathrm{gl} \eta_\mathrm{ef}^t \beta^3 \eta_\mathrm{lo}^3 \sum_{u=0}^{U-1} v_u^t(1+q-v_u^t) \left(\alpha_u \delta_u^t\right)^2 \left(\kappa_u^t \right)^3 \Big] \left\Vert \nabla f^{t}(\mathbf{w}^t) \right\Vert^2 + \digamma_\sigma (\sigma, \pmb{v, \kappa^t,\delta^t}) + \digamma_\varpi (\sigma, \pmb{v,\kappa^t,\delta^t}),
\end{align}
where we define $\digamma_\sigma (\sigma, \pmb{v, \kappa^t,\delta^t}) \coloneqq (2+q) \beta \left(\sigma \eta_\mathrm{gl} \eta_\mathrm{lo} \eta_\mathrm{ef}^t \right)^2 \sum_{u=0}^{U-1} v_u^t \alpha_u^2 \frac{ (\delta_u^t)^2}{\kappa_u^t} + 4 \eta_\mathrm{gl} \eta_\mathrm{ef}^t \beta^2 \sigma^2 \eta_\mathrm{lo}^3 \sum_{u=0}^{U-1} \alpha_u \kappa_u^t (v_u^t\delta_u^t)^2 + 12 \sigma^2 \left(\eta_\mathrm{gl} \eta_\mathrm{ef}^t \right)^2 \beta^3 \eta_\mathrm{lo}^4  \sum_{u=0}^{U-1} v_u^t(1+q-v_u^t) \left(\alpha_u \delta_u^t \kappa_u^t \right)^2 $ and $\digamma_\varpi (\sigma, \pmb{v,\kappa^t,\delta^t}) \coloneqq 2\eta_\mathrm{gl} \eta_\mathrm{lo} \eta_\mathrm{ef}^t \varpi^2 \sum_{u=0}^{U-1} \alpha_u ( \delta_u^t (1 - v_u^t) )^2 + 24 \eta_\mathrm{gl} \eta_\mathrm{ef}^t \left(\beta \varpi \right)^2 \eta_\mathrm{lo}^3  \sum_{u=0}^{U-1} \alpha_u (v_u^t\delta_u^t)^2 \left(\kappa_u^t\right)^2 + 3 \beta \left(\eta_\mathrm{gl} \eta_\mathrm{lo} \eta_\mathrm{ef}^t \varpi\right)^2 \sum_{u=0}^{U-1} v_u^t(1+q-v_u^t) \left(\alpha_u \delta_u^t\right)^2 \kappa_u^t + 72 \left(\varpi \eta_\mathrm{gl} \eta_\mathrm{ef}^t \right)^2 \beta^3 \eta_\mathrm{lo}^4 \sum_{u=0}^{U-1} v_u^t(1+q-v_u^t) \left(\alpha_u \delta_u^t\right)^2 \left(\kappa_u^t \right)^3 $ in $(a)$.

\end{proof}

\newpage
\begin{Lemma}
\label{lema:model_drfit_per_sgd}
When the learning rate $\eta_\mathrm{lo}$ satisfies the condition $\eta_\mathrm{lo} \leq \frac{1}{8\beta \kappa}$, we have
\begin{align}
    \mathrm{T}_5 
    &= \mathbb{E}_{\pmb{\zeta}} \left[\left\Vert \mathbf{w}^t - \mathbf{w}_u^{t,\tau} \right\Vert^2\right] 
    \leq 4 \kappa_u^t \eta_\mathrm{lo}^2 \sigma^2 + 24 \left(\kappa_u^t \varpi \eta_\mathrm{lo}\right)^2 + 24 \left(\kappa_u^t \eta_\mathrm{lo}\right)^2 \left\Vert \nabla f^{t}(\mathbf{w}^t) \right\Vert^2.
\end{align}
\end{Lemma}

\begin{proof}
The following proof is similar to \cite[Lemma $3$]{reddi2021adaptive} but shown for a single client as in \cite[Lemma B.3.2]{wang2024lightweight}.
\begin{align}
    \mathrm{T}_5 
    &= \mathbb{E}_{\pmb{\zeta}} \left[\left\Vert \mathbf{w}^t - \mathbf{w}_u^{t,\tau} \right\Vert^2\right] \nonumber\\
    &= \mathbb{E}_{\pmb{\zeta}} \left[\left\Vert \mathbf{w}_u^{t,\tau} - \mathbf{w}^t \right\Vert^2\right] \nonumber\\
    &\overset{(a)}{=} \mathbb{E}_{\pmb{\zeta}} \left[\left\Vert \mathbf{w}_u^{t,\tau-1} - \eta_\mathrm{lo} g_u^{t}(\mathbf{w}_u^{t,\tau-1}) - \mathbf{w}^t \right\Vert^2\right] \nonumber\\
    &= \mathbb{E}_{\pmb{\zeta}} \left[\left\Vert \mathbf{w}_u^{t,\tau-1} - \mathbf{w}^t - \eta_\mathrm{lo} \left[ g_u^{t}(\mathbf{w}_u^{t,\tau-1}) - \nabla f_u^{t}(\mathbf{w}_u^{t,\tau-1}) + \nabla f_u^{t}(\mathbf{w}_u^{t,\tau-1}) - \nabla f_u^{t}(\mathbf{w}^t) + \nabla f_u^{t}(\mathbf{w}^t) - \nabla f^{t}(\mathbf{w}^t) + \nabla f^{t}(\mathbf{w}^t) \right] \right\Vert^2\right] \nonumber\\
    &= \mathbb{E}_{\pmb{\zeta}} \left[\left\Vert \underbrace{\mathbf{w}_u^{t,\tau-1} - \mathbf{w}^t - \eta_\mathrm{lo} \left[\nabla f_u^{t}(\mathbf{w}_u^{t,\tau-1}) - \nabla f_u^{t}(\mathbf{w}^t) + \nabla f_u^{t}(\mathbf{w}^t) - \nabla f^{t}(\mathbf{w}^t) + \nabla f^{t}(\mathbf{w}^t) \right]}_{\coloneqq \mathbf{z}_1} - \underbrace{\eta_\mathrm{lo} \left[ g_u^{t}(\mathbf{w}_u^{t,\tau-1}) - \nabla f_u^{t}(\mathbf{w}_u^{t,\tau-1})\right]}_{\coloneqq \mathbf{z}_2} \right\Vert^2\right] \nonumber\\
    &\overset{(b)}{=} \mathbb{E}_{\pmb{\zeta}} \left[\left\Vert \eta_\mathrm{lo} \left[ g_u^{t}(\mathbf{w}_u^{t,\tau-1}) - \nabla f_u^{t}(\mathbf{w}_u^{t,\tau-1}) \right] \right\Vert^2 \right] + \nonumber\\
    &\squad \mathbb{E}_{\pmb{\zeta}} \left[\left\Vert \mathbf{w}_u^{t,\tau-1} - \mathbf{w}^t - \eta_\mathrm{lo} \left[\nabla f_u^{t}(\mathbf{w}_u^{t,\tau-1}) - \nabla f_u^{t}(\mathbf{w}^t) + \nabla f_u^{t}(\mathbf{w}^t) - \nabla f^{t}(\mathbf{w}^t) + \nabla f^{t}(\mathbf{w}^t) \right] \right\Vert^2\right] - \nonumber\\
    &\squad 2\mathbb{E}_{\pmb{\zeta}} \left[ \left< -\eta_\mathrm{lo} \left[ g_u^{t}(\mathbf{w}_u^{t,\tau-1}) - \nabla f_u^{t}(\mathbf{w}_u^{t,\tau-1}) \right], \mathbf{w}_u^{t,\tau-1} - \mathbf{w}^t - \eta_\mathrm{lo} \left[\nabla f_u^{t}(\mathbf{w}_u^{t,\tau-1}) - \nabla f_u^{t}(\mathbf{w}^t) + \nabla f_u^{t}(\mathbf{w}^t) - \nabla f^{t}(\mathbf{w}^t) + \nabla f^{t}(\mathbf{w}^t) \right] \right> \right] \nonumber\\
    &=\mathbb{E}_{\pmb{\zeta}} \left[\left\Vert \eta_\mathrm{lo} \left[ g_u^{t}(\mathbf{w}_u^{t,\tau-1}) - \nabla f_u^{t}(\mathbf{w}_u^{t,\tau-1}) \right] \right\Vert^2 \right] + \nonumber\\
    &\squad \mathbb{E}_{\pmb{\zeta}} \left[\left\Vert \mathbf{w}_u^{t,\tau-1} - \mathbf{w}^t - \eta_\mathrm{lo} \left[\nabla f_u^{t}(\mathbf{w}_u^{t,\tau-1}) - \nabla f_u^{t}(\mathbf{w}^t) + \nabla f_u^{t}(\mathbf{w}^t) - \nabla f^{t}(\mathbf{w}^t) + \nabla f^{t}(\mathbf{w}^t) \right] \right\Vert^2\right] + \nonumber\\
    &\squad 2 \eta_\mathrm{lo} \left< \mathbb{E}_{\pmb{\zeta}} \left[ g_u^{t}(\mathbf{w}_u^{t,\tau-1}) \right] - \nabla f_u^{t}(\mathbf{w}_u^{t,\tau-1}), \mathbf{w}_u^{t,\tau-1} - \mathbf{w}^t - \eta_\mathrm{lo} \left[\nabla f_u^{t}(\mathbf{w}_u^{t,\tau-1}) - \nabla f_u^{t}(\mathbf{w}^t) + \nabla f_u^{t}(\mathbf{w}^t) - \nabla f^{t}(\mathbf{w}^t) + \nabla f^{t}(\mathbf{w}^t) \right] \right>  \nonumber\\
    &= \eta_\mathrm{lo}^2  \mathbb{E}_{\pmb{\zeta}} \left[\left\Vert g_u^{t}(\mathbf{w}_u^{t,\tau-1}) - \nabla f_u^{t}(\mathbf{w}_u^{t,\tau-1}) \right\Vert^2 \right] + \mathbb{E}_{\pmb{\zeta}} \left[\left\Vert \mathbf{w}_u^{t,\tau-1} - \mathbf{w}^t - \eta_\mathrm{lo} \left[\nabla f_u^{t}(\mathbf{w}_u^{t,\tau-1}) - \nabla f_u^{t}(\mathbf{w}^t) + \nabla f_u^{t}(\mathbf{w}^t) - \nabla f^{t}(\mathbf{w}^t) + \nabla f^{t}(\mathbf{w}^t) \right] \right\Vert^2\right] \nonumber\\
    &\leq \eta_\mathrm{lo}^2 \sigma^2 + \mathbb{E}_{\pmb{\zeta}} \left[\left\Vert \mathbf{w}_u^{t,\tau-1} - \mathbf{w}^t - \eta_\mathrm{lo} \left[\nabla f_u^{t}(\mathbf{w}_u^{t,\tau-1}) - \nabla f_u^{t}(\mathbf{w}^t) + \nabla f_u^{t}(\mathbf{w}^t) - \nabla f^{t}(\mathbf{w}^t) + \nabla f^{t}(\mathbf{w}^t) \right] \right\Vert^2\right] \nonumber\\
    &\overset{(c)}{\leq} \eta_\mathrm{lo}^2 \sigma^2 + \left(1 + \frac{1}{2\kappa_u^t - 1} \right) \mathbb{E}_{\pmb{\zeta}} \left[\left\Vert \mathbf{w}_u^{t,\tau-1} - \mathbf{w}^t \right\Vert^2\right] + \nonumber\\
    &\squad \left(1 + \frac{1}{\frac{1}{2\kappa_u^t - 1} } \right) \mathbb{E}_{\pmb{\zeta}} \left[\left\Vert \eta_\mathrm{lo} \left[\nabla f_u^{t}(\mathbf{w}_u^{t,\tau-1}) - \nabla f_u^{t}(\mathbf{w}^t) + \nabla f_u^{t}(\mathbf{w}^t) - \nabla f^{t}(\mathbf{w}^t) + \nabla f^{t}(\mathbf{w}^t) \right] \right\Vert^2\right]\nonumber\\
    &=\eta_\mathrm{lo}^2 \sigma^2 + \left(1 + \frac{1}{2\kappa_u^t - 1} \right) \mathbb{E}_{\pmb{\zeta}} \left[\left\Vert \mathbf{w}_u^{t,\tau-1} - \mathbf{w}^t \right\Vert^2\right] + 2\kappa_u^t \eta_\mathrm{lo}^2 \mathbb{E}_{\pmb{\zeta}} \left[\left\Vert  \nabla f_u^{t}(\mathbf{w}_u^{t,\tau-1}) - \nabla f_u^{t}(\mathbf{w}^t) + \nabla f_u^{t}(\mathbf{w}^t) - \nabla f^{t}(\mathbf{w}^t) + \nabla f^{t}(\mathbf{w}^t)  \right\Vert^2\right] \nonumber\\
    &\leq \eta_\mathrm{lo}^2 \sigma^2 + \left(1 + \frac{1}{2\kappa_u^t - 1} \right) \mathbb{E}_{\pmb{\zeta}} \left[\left\Vert \mathbf{w}_u^{t,\tau-1} - \mathbf{w}^t \right\Vert^2\right] + 6\kappa_u^t \eta_\mathrm{lo}^2 \mathbb{E}_{\pmb{\zeta}} \left[\left\Vert  \nabla f_u^{t}(\mathbf{w}_u^{t,\tau-1}) - \nabla f_u^{t}(\mathbf{w}^t) \right\Vert^2\right] + \nonumber\\
    &\squad 6\kappa_u^t \eta_\mathrm{lo}^2 \mathbb{E}_{\pmb{\zeta}} \left[\left\Vert \nabla f_u^{t}(\mathbf{w}^t) - \nabla f^{t}(\mathbf{w}^t) \right\Vert^2\right] + 6 \kappa_u^t \eta_\mathrm{lo}^2 \left\Vert \nabla f^{t}(\mathbf{w}^t) \right\Vert^2 \nonumber\\
    &\leq \eta_\mathrm{lo}^2 \sigma^2 + \left(1 + \frac{1}{2\kappa_u^t - 1} \right) \mathbb{E}_{\pmb{\zeta}} \left[\left\Vert \mathbf{w}_u^{t,\tau-1} - \mathbf{w}^t \right\Vert^2\right] + 6 \kappa_u^t \beta^2 \eta_\mathrm{lo}^2 \mathbb{E}_{\pmb{\zeta}} \left[\left\Vert \mathbf{w}_u^{t,\tau-1} - \mathbf{w}^t \right\Vert^2\right] + 6 \kappa_u^t \varpi^2 \eta_\mathrm{lo}^2 + 6 \kappa_u^t \eta_\mathrm{lo}^2 \left\Vert \nabla f^{t}(\mathbf{w}^t) \right\Vert^2 \nonumber\\
    &= \eta_\mathrm{lo}^2 \sigma^2 + \left(1 + \frac{1}{2\kappa_u^t - 1} + 6 \kappa_u^t \beta^2 \eta_\mathrm{lo}^2 \right) \mathbb{E}_{\pmb{\zeta}} \left[\left\Vert \mathbf{w}_u^{t,\tau-1} - \mathbf{w}^t \right\Vert^2\right] + 6 \kappa_u^t \varpi^2 \eta_\mathrm{lo}^2 + 6 \kappa_u^t \eta_\mathrm{lo}^2 \left\Vert \nabla f^{t}(\mathbf{w}^t)  \right\Vert^2,
\end{align}
where $(a)$ follows from one \ac{sgd} step and $(b)$ is true since $\Vert \mathbf{z}_1 - \mathbf{z}_2 \Vert^2 = \Vert \mathbf{z}_1 \Vert^2 + \Vert \mathbf{z}_2 \Vert^2 - 2\left<\mathbf{z}_1, \mathbf{z}_2 \right>$.
We get $(c)$ using Peter-Paul inequality, i.e., $\Vert \mathbf{z}_1 + \mathbf{z}_2 \Vert^2 \leq (1+b)\Vert \mathbf{z}_1 \Vert^2 + \left(1+\frac{1}{b}\right) \Vert \mathbf{z}_2 \Vert^2$, for some $b >0$.

If $\eta_\mathrm{lo} \leq \frac{1}{8\beta \kappa_u^t}$, then we have 
\begin{align}
    &1 + \frac{1}{2\kappa_u^t - 1} + 6 \kappa_u^t \beta^2 \eta_\mathrm{lo}^2 
    \leq 1 + \frac{1}{2\kappa_u^t - 1} + 6 \kappa_u^t \beta^2 \cdot \frac{1}{64\beta^2 (\kappa_u^t)^2}
    = 1 + \frac{1}{2\kappa_u^t - 1} + \frac{3}{32 \kappa_u^t}
\end{align}
Since $\frac{3}{32 \kappa_u^t} \leq \frac{1}{2\kappa_u^t - 1}$ for $\kappa_u^t \geq 1$, we have
\begin{align}
    &1 + \frac{1}{2\kappa_u^t - 1} + 6 \kappa_u^t \beta^2 \eta_\mathrm{lo}^2 
    \leq 1 + \frac{1}{2\kappa_u^t - 1} + \frac{1}{2\kappa_u^t - 1} 
    = 1 + \frac{2}{2\kappa_u^t - 1} 
    = 1 + \frac{1}{\kappa_u^t - \frac{1}{2}}. 
\end{align}
Thus, we have
\begin{align}
\label{eq:t5_bound}
    \mathrm{T}_5 
    &= \mathbb{E}_{\pmb{\zeta}} \left[\left\Vert \mathbf{w}^t - \mathbf{w}_u^{t,\tau} \right\Vert^2\right] \nonumber\\
    &\leq \eta_\mathrm{lo}^2 \sigma^2 + \left(1 + \frac{1}{2\kappa_u^t - 1} + 6 \kappa_u^t \beta^2 \eta_\mathrm{lo}^2 \right) \mathbb{E}_{\pmb{\zeta}} \left[\left\Vert \mathbf{w}_u^{t,\tau-1} - \mathbf{w}^t \right\Vert^2\right] + 6 \kappa_u^t \varpi^2 \eta_\mathrm{lo}^2 + 6 \kappa_u^t \eta_\mathrm{lo}^2 \left\Vert \nabla f^{t}(\mathbf{w}^t) \right\Vert^2 \nonumber\\
    &\leq \left(1 + \frac{1}{\kappa_u^t - \frac{1}{2}} \right) \mathbb{E}_{\pmb{\zeta}} \left[\left\Vert \mathbf{w}_u^{t,\tau-1} - \mathbf{w}^t \right\Vert^2\right] + C_1,
\end{align}
where $C_1 \coloneqq \eta_\mathrm{lo}^2 \sigma^2 + 6 \kappa_u^t \varpi^2 \eta_\mathrm{lo}^2 + 6 \kappa_u^t \eta_\mathrm{lo}^2 \left\Vert \nabla f^{t}(\mathbf{w}^t) \right\Vert^2$.

Now, note that when $\tau=0$, we have
\begin{align}
\label{eq:t5_bound_tau0}
    \mathbb{E}_{\pmb{\zeta}} \left[\left\Vert \mathbf{w}^t - \mathbf{w}_u^{t,0} \right\Vert^2\right] 
    =\mathbb{E}_{\pmb{\zeta}} \left[\left\Vert \mathbf{w}^t - \mathbf{w}^{t} \right\Vert^2\right] = 0, 
\end{align}
because $\mathbf{w}_u^{t,0} \gets \mathbf{w}^t$ during the synchronization.

If $\tau=1$, we have
\begin{align}
\label{eq:t5_bound_tau1}
    \mathbb{E}_{\pmb{\zeta}} \left[\left\Vert \mathbf{w}^t - \mathbf{w}_u^{t,1} \right\Vert^2\right] 
    &\leq \left(1 + \frac{1}{\kappa_u^t - \frac{1}{2}} \right) \mathbb{E}_{\pmb{\zeta}} \left[\left\Vert \mathbf{w}_u^{t,0} - \mathbf{w}^t \right\Vert^2\right] + C_1 \nonumber\\
    &= C_1.
\end{align}
If $\tau=2$, we have
\begin{align}
\label{eq:t5_bound_tau2}
    \mathbb{E}_{\pmb{\zeta}} \left[\left\Vert \mathbf{w}^t - \mathbf{w}_u^{t,2} \right\Vert^2\right] 
    &\leq \left(1 + \frac{1}{\kappa_u^t - \frac{1}{2}} \right) \mathbb{E}_{\pmb{\zeta}} \left[\left\Vert \mathbf{w}_u^{t,1} - \mathbf{w}^t \right\Vert^2\right] + C_1 \nonumber\\
    &= \left(1 + \frac{1}{\kappa_u^t - \frac{1}{2}} \right) C_1 + C_1.
\end{align}
Similarly, if $\tau=3$, we have
\begin{align}
\label{eq:t5_bound_tau3}
    \mathbb{E}_{\pmb{\zeta}} \left[\left\Vert \mathbf{w}^t - \mathbf{w}_u^{t,3} \right\Vert^2\right] 
    &\leq \left(1 + \frac{1}{\kappa_u^t - \frac{1}{2}} \right) \mathbb{E}_{\pmb{\zeta}} \left[\left\Vert \mathbf{w}_u^{t,2} - \mathbf{w}^t \right\Vert^2\right] + C_1 \nonumber\\
    &= \left(1 + \frac{1}{\kappa_u^t - \frac{1}{2}} \right) \left\{\left(1 + \frac{1}{\kappa_u^t - \frac{1}{2}} \right) C_1 + C_1 \right\} + C_1 \nonumber\\
    &= \left(1 + \frac{1}{\kappa_u^t - \frac{1}{2}} \right)^2 C_1 + \left(1 + \frac{1}{\kappa_u^t - \frac{1}{2}} \right) C_1 + C_1.
\end{align}
Therefore, for a generic $\tau$, we have 
\begin{align}
\label{eq:t5_bound_0}
    \mathrm{T}_5 
    &= \mathbb{E}_{\pmb{\zeta}} \left[\left\Vert \mathbf{w}^t - \mathbf{w}_u^{t,\tau} \right\Vert^2\right] \nonumber\\
    &\leq \sum_{\tau''=0}^{\tau-1} \left(1 + \frac{1}{\kappa_u^t - \frac{1}{2}} \right)^{\tau''} C_1 \nonumber\\
    &\overset{(a)}{\leq} \sum_{\tau''=0}^{\kappa_u^t - 1} \left(1 + \frac{1}{\kappa_u^t - \frac{1}{2}} \right)^{\tau''} C_1 \nonumber\\
    &\overset{(b)}{=} \frac{\left(1 + \frac{1}{\kappa_u^t - \frac{1}{2}} \right)^{\kappa_u^t} - 1}{\left(1 + \frac{1}{\kappa_u^t - \frac{1}{2}} \right) - 1} \times C_1 \nonumber\\
    &= \left[\left(1 + \frac{1}{\kappa_u^t - \frac{1}{2}} \right)^{\kappa_u^t} - 1 \right] \left(\kappa_u^t - \frac{1}{2} \right) \times C_1 \nonumber\\
    &= \left[\left(1 + \frac{1}{\kappa_u^t - \frac{1}{2}} \right)^{\kappa_u^t - \frac{1}{2}} \left(1 + \frac{1}{\kappa_u^t - \frac{1}{2}} \right)^{\frac{1}{2}} - 1 \right] \left(\kappa_u^t - \frac{1}{2} \right) \times C_1 \nonumber\\
    &\overset{(c)}{\leq} \left[e \left(1 + \frac{1}{\kappa_u^t - \frac{1}{2}} \right)^{\frac{1}{2}} - 1 \right] \left(\kappa_u^t - \frac{1}{2} \right) \times C_1 \nonumber\\
    &\overset{(d)}{\leq} \left[\sqrt{3}e-1\right] \left(\kappa_u^t - \frac{1}{2} \right) \times C_1 \nonumber\\
    &\overset{(e)}{\leq} 4\kappa_u^t \times C_1 \nonumber\\
    &=4\kappa_u^t \times \left[\eta_\mathrm{lo}^2 \sigma^2 + 6 \kappa_u^t \varpi^2 \eta_\mathrm{lo}^2 + 6 \kappa_u^t \eta_\mathrm{lo}^2 \left\Vert \nabla f^{t}(\mathbf{w}^t) \right\Vert^2 \right] \nonumber\\
    &=4 \kappa_u^t \eta_\mathrm{lo}^2 \sigma^2 + 24 \left(\kappa_u^t \varpi \eta_\mathrm{lo}\right)^2 + 24 \left(\kappa_u^t \eta_\mathrm{lo}\right)^2 \left\Vert \nabla f^{t}(\mathbf{w}^t) \right\Vert^2
\end{align}
where $(a)$ is true since $\tau \leq \kappa_u^t$.
We get $(b)$ using the definition of geometric series.
Besides, $(c)$ is true since for a sufficiently large $n$, $e = \lim_{n\to \infty} \left(1+\frac{1}{n} \right)^n$, whereas $(d)$ is true since $\left(1 + \frac{1}{\kappa_u^t - \frac{1}{2}} \right) \leq 3$ for $\kappa_u^t \geq 1$.
Finally, we use the fact that $\left(\sqrt{3}e - 1\right) \leq 4$ and $\left(\kappa_u^t -\frac{1}{2} \right) \leq \kappa_u^t$ for $\kappa_u^t \geq 1$.
\end{proof}

\newpage
\begin{theorem-box}[Theorem \ref{convgRate}]
Suppose the above assumptions hold. 
When the learning rates satisfy $\eta_\mathrm{gl}\eta_\mathrm{lo} < \frac{1}{\beta \eta_\mathrm{ef}^t}$, and $\eta_\mathrm{lo} < \text{min} \left\{ \frac{1}{8\beta\kappa}, \sqrt{\frac{4 - 16 \sum_{u=0}^{U-1} \alpha_u ( \delta_u^t (1 - v_u^t) )^2 - 33 \kappa \sum_{u=0}^{U-1} v_u^t (1+q-v_u^t) \left(\alpha_u \delta_u^t\right)^2}{192 \kappa^2 \beta^2 \sum_{u=0}^{U-1} \alpha_u ( v_u^t \delta_u^t)^2}} \right\}$, and $4 - 16 \sum_{u=0}^{U-1} \alpha_u ( \delta_u^t (1 - v_u^t) )^2 - 33 \kappa \sum_{u=0}^{U-1} v_u^t (1+q-v_u^t) \left(\alpha_u \delta_u^t\right)^2 \geq 0$, we have
\begin{align}
   &\eta_\mathrm{ef}^t \left\Vert \nabla f^t \left(\mathbf{w}^{t} \right) \right\Vert^2  
    \leq \frac{2 \big[ f^t(\mathbf{w}^t) - \mathbb{E}_{\pmb{\mathrm{1},\zeta},Q} \left[ f^{t+1}(\mathbf{w}^{t+1}) \right]\big]}{\eta_\mathrm{gl} \eta_\mathrm{lo}} + \frac{2\Phi^t}{\eta_\mathrm{gl} \eta_\mathrm{lo}} + \Omega_\sigma (\sigma, \pmb{v,\kappa^t,\delta^t}) + \Omega_\varpi (\sigma, \pmb{v,\kappa^t,\delta^t}).    
\end{align}
where $\Phi^t \triangleq  \mathbb{E}_{\pmb{\mathrm{1},\zeta},Q} \left[f^{t+1}(\mathbf{w}^{t+1}) \right] - \mathbb{E}_{\pmb{\mathrm{1},\zeta},Q} \left[f^{t} (\mathbf{w}^{t+1})\right]$, $\Omega_\sigma (\sigma, \pmb{v,\kappa^t,\delta^t}) \coloneqq \frac{2\digamma_\sigma (\sigma, \pmb{v, \kappa^t,\delta^t})}{\eta_\mathrm{gl}\eta_\mathrm{lo}} = 2(2+q) \beta \eta_\mathrm{gl} \eta_\mathrm{lo} \left(\sigma \eta_\mathrm{ef}^t \right)^2 \sum_{u=0}^{U-1} v_u^t \alpha_u^2 \frac{ (\delta_u^t)^2}{\kappa_u^t} + 8 \eta_\mathrm{ef}^t \beta^2 \sigma^2 \eta_\mathrm{lo}^2 \sum_{u=0}^{U-1} \alpha_u \kappa_u^t (v_u^t\delta_u^t)^2 + 24 \sigma^2 \left(\eta_\mathrm{ef}^t \right)^2 \beta^3 \eta_\mathrm{lo}^3  \sum_{u=0}^{U-1} v_u^t(1+q-v_u^t) \left(\alpha_u \delta_u^t \kappa_u^t \right)^2$, and $\Omega_\varpi (\sigma, \pmb{v,\kappa^t,\delta^t}) \coloneqq  \frac{2\digamma_\varpi (\sigma, \pmb{v,\kappa^t,\delta^t})}{\eta_\mathrm{gl}\eta_\mathrm{lo}} = 4 \eta_\mathrm{ef}^t \varpi^2 \sum_{u=0}^{U-1} \alpha_u ( \delta_u^t (1 - v_u^t) )^2 + 48 \eta_\mathrm{ef}^t \left(\beta \varpi \right)^2 \eta_\mathrm{lo}^2  \sum_{u=0}^{U-1} \alpha_u (v_u^t\delta_u^t)^2 \left(\kappa_u^t\right)^2 + 6 \beta \eta_\mathrm{gl} \eta_\mathrm{lo} \left(\eta_\mathrm{ef}^t \varpi\right)^2 \sum_{u=0}^{U-1} v_u^t(1+q-v_u^t) \left(\alpha_u \delta_u^t\right)^2 \kappa_u^t + 144 \eta_\mathrm{gl} \left(\varpi \eta_\mathrm{ef}^t \right)^2 \beta^3 \eta_\mathrm{lo}^3 \sum_{u=0}^{U-1} v_u^t(1+q-v_u^t) \left(\alpha_u \delta_u^t\right)^2 \left(\kappa_u^t \right)^3$.
An immediate consequence of the above equation is 
\begin{align}
    \frac{1}{\sum_{t=0}^{T-1} \eta_{\mathrm{ef}}^t} \sum_{t=0}^{T-1} \eta_{\mathrm{ef}}^t \mathbb{E}_{\pmb{1,\zeta},Q} \left[\left\Vert \nabla f^t (\mathbf{w}^t) \right\Vert^2 \right] 
    &\leq \frac{2\left(f^0(\mathbf{w}^0) - \mathbb{E}_{\pmb{\mathrm{1},\zeta},Q} \left[ f^{T}(\mathbf{w}^{T}) \right] \right)}{\eta_\mathrm{gl} \eta_\mathrm{lo} \sum_{t=0}^{T-1} \eta_{\mathrm{ef}}^t} + \frac{2}{\eta_\mathrm{gl} \eta_\mathrm{lo} \sum_{t=0}^{T-1} \eta_{\mathrm{ef}}^t} \sum_{t=0}^{T-1} \Phi^t + \nonumber\\
    &\squad \frac{1}{\sum_{t=0}^{T-1} \eta_{\mathrm{ef}}^t} \sum_{t=0}^{T-1} \Omega_\sigma (\pmb{\sigma, v, \kappa^t, \delta^t}) + \frac{1}{\sum_{t=0}^{T-1} \eta_{\mathrm{ef}}^t} \sum_{t=0}^{T-1} \Omega_\varpi (\pmb{\varpi, v, \kappa^t, \delta^t}).
\end{align}
\end{theorem-box}

\begin{proof}
We want to find how the loss function evolves between two consecutive global rounds. 
That is we want to find $f^{t+1}(\mathbf{w}^{t+1}) - f^t(\mathbf{w}^t)$.
Note that we can write
\begin{align}
    f^{t+1}(\mathbf{w}^{t+1}) - f^t(\mathbf{w}^t)& 
    = \underbrace{f^{t+1}(\mathbf{w}^{t+1}) - f^{t} (\mathbf{w}^{t+1})}_{T_1 } + \underbrace{f^{t} (\mathbf{w}^{t+1}) - f^t(\mathbf{w}^t)}_{T_2},
\end{align}
where, by definitions, we have  
\begin{align*}
    &\alpha_u^t \triangleq \alpha_u \delta_u^t = \left( \frac{\mathrm{D}_u}{\sum_{u'=0}^{U-1} \mathrm{D}_{u'}} \right) \delta_u^t. \\
    &f^{t} (\mathbf{w}^t) \triangleq \sum_{u=0}^{U-1} \alpha_u^t f_u^t(\mathbf{w}^t) = \sum_{u=0}^{U-1} \alpha_u^t f_u (\mathbf{w}^t|\mathcal{D}_u^t).\\
    &f^{t+1} (\mathbf{w}^{t+1}) \triangleq \sum_{u=0}^{U-1} \alpha_u^{t+1} (\mathbf{w}^{t+1}) \triangleq \sum_{u=0}^{U-1} \alpha_u^{t+1} f_u (\mathbf{w}^{t+1}|\mathcal{D}_u^{t+1}).\\ 
    &f^{t+1} (\mathbf{w}^t) \triangleq \sum_{u=0}^{U-1} \alpha_u^{t+1} f_u^{t+1} (\mathbf{w}^t) = \sum_{u=0}^{U-1} \alpha_u^{t+1} f_u (\mathbf{w}^t|\mathcal{D}_u^{t+1}).\\ 
    &f^{t} (\mathbf{w}^{t+1}) \triangleq \sum_{u=0}^{U-1} \alpha_u^t f_u^t (\mathbf{w}^{t+1}) = \sum_{u=0}^{U-1} \alpha_u^t f_u (\mathbf{w}^{t+1}|\mathcal{D}_u^t).
\end{align*}
As such, the term $T_1$ bounds the global loss differences, evaluated on datasets $\mathcal{D}^{t+1}$ and $\mathcal{D}^{t}$, calculated with the same model $\mathbf{w}^{t+1}$.  
In other words, $T_1$ is induced due to the data {\em distribution shift}.

Taking expectation over all randomness (i.e., randomness from time-varying wireless link-induced participation uncertainty, random sampling of the mini-batches for stochastic gradient, and randomness in the stochastic quantizer), we have
\begin{align}
    \label{eq:main_conv}
    \mathbb{E}_{\pmb{\mathrm{1},\zeta},Q} \left[ f^{t+1}(\mathbf{w}^{t+1}) \right] - f^t(\mathbf{w}^t) & 
    = \underbrace{\mathbb{E}_{\pmb{\mathrm{1},\zeta},Q} \left[f^{t+1}(\mathbf{w}^{t+1}) \right] - \mathbb{E}_{\pmb{\mathrm{1},\zeta},Q} \left[f^{t} (\mathbf{w}^{t+1})\right]}_{\coloneqq \Phi^t} + \mathbb{E}_{\pmb{\mathrm{1},\zeta},Q} \left[f^{t} (\mathbf{w}^{t+1})\right] - f^t(\mathbf{w}^t),
\end{align}

Plugging (\ref{eq:main_conv_2}) from Lemma \ref{lemma_convrate} into (\ref{eq:main_conv}), we get
\begin{align}
    \label{eq:main_conv_0}
    &\mathbb{E}_{\pmb{\mathrm{1},\zeta},Q} \left[ f^{t+1}(\mathbf{w}^{t+1}) \right] - f^t(\mathbf{w}^t)  
    \leq \Phi^t - \frac{\eta_\mathrm{gl} \eta_\mathrm{lo} \eta_\mathrm{ef}^t}{2} \Big[1 - 4 \sum_{u=0}^{U-1} \alpha_u ( \delta_u^t (1 - v_u^t) )^2 - 48 \beta^2 \eta_\mathrm{lo}^2 \sum_{u=0}^{U-1} \alpha_u (\kappa_u^t v_u^t \delta_u^t)^2 - \nonumber\\
    &\squad 6 \beta \eta_\mathrm{gl} \eta_\mathrm{lo} \eta_\mathrm{ef}^t \sum_{u=0}^{U-1} v_u^t (1+q-v_u^t) \kappa_u^t \left(\alpha_u \delta_u^t\right)^2 - 144 \eta_\mathrm{gl} \eta_\mathrm{ef}^t \beta^3 \eta_\mathrm{lo}^3 \sum_{u=0}^{U-1} v_u^t(1+q-v_u^t) \left(\alpha_u \delta_u^t\right)^2 \left(\kappa_u^t \right)^3 \Big] \left\Vert \nabla f^{t}(\mathbf{w}^t) \right\Vert^2 + \nonumber\\
    &\squad \digamma_\sigma (\sigma, \pmb{v, \kappa^t,\delta^t}) + \digamma_\varpi (\sigma, \pmb{v,\kappa^t,\delta^t}).
\end{align}

Rearranging the terms, we get
\begin{align}
    &\frac{\eta_\mathrm{gl} \eta_\mathrm{lo} \eta_\mathrm{ef}^t}{2} \Big[1 - 4 \sum_{u=0}^{U-1} \alpha_u ( \delta_u^t (1 - v_u^t) )^2 - 48 \beta^2 \eta_\mathrm{lo}^2 \sum_{u=0}^{U-1} \alpha_u (\kappa_u^t v_u^t \delta_u^t)^2 - 6 \beta \eta_\mathrm{gl} \eta_\mathrm{lo} \eta_\mathrm{ef}^t \sum_{u=0}^{U-1} v_u^t (1+q-v_u^t) \kappa_u^t \left(\alpha_u \delta_u^t\right)^2 - \nonumber\\
    &\squad 144 \eta_\mathrm{gl} \eta_\mathrm{ef}^t \beta^3 \eta_\mathrm{lo}^3 \sum_{u=0}^{U-1} v_u^t(1+q-v_u^t) \left(\alpha_u \delta_u^t\right)^2 \left(\kappa_u^t \right)^3 \Big] \left\Vert \nabla f^{t}(\mathbf{w}^t) \right\Vert^2 \nonumber \\
    &\leq f^t(\mathbf{w}^t) - \mathbb{E}_{\pmb{\mathrm{1},\zeta},Q} \left[ f^{t+1}(\mathbf{w}^{t+1}) \right] + \Phi^t +  \digamma_\sigma (\sigma, \pmb{v, \kappa^t,\delta^t}) + \digamma_\varpi (\sigma, \pmb{v,\kappa^t,\delta^t}).
\end{align}

Using the fact that $0 \leq \kappa_u^t \leq \kappa$ in the left-hand side of the above inequality, we have 
\begin{align}
    & \Big[1 - 4 \sum_{u=0}^{U-1} \alpha_u ( \delta_u^t (1 - v_u^t) )^2 - 48 \kappa^2 \beta^2 \eta_\mathrm{lo}^2 \sum_{u=0}^{U-1} \alpha_u ( v_u^t \delta_u^t)^2 - 6 \beta \kappa \eta_\mathrm{gl} \eta_\mathrm{lo} \eta_\mathrm{ef}^t \sum_{u=0}^{U-1} v_u^t (1+q-v_u^t) \left(\alpha_u \delta_u^t\right)^2 - \nonumber\\
    &\squad 144 \eta_\mathrm{gl} \eta_\mathrm{ef}^t \kappa^3 \beta^3 \eta_\mathrm{lo}^3 \sum_{u=0}^{U-1} v_u^t(1+q-v_u^t) \left(\alpha_u \delta_u^t\right)^2 \Big] \times \eta_\mathrm{ef}^t \left\Vert \nabla f^{t}(\mathbf{w}^t) \right\Vert^2 \nonumber \\
    &\leq \frac{2 \left(f^t(\mathbf{w}^t) - \mathbb{E}_{\pmb{\mathrm{1},\zeta},Q} \left[ f^{t+1}(\mathbf{w}^{t+1}) \right] \right)} {\eta_\mathrm{gl} \eta_\mathrm{lo} } + \frac{2 \Phi^t}{\eta_\mathrm{gl} \eta_\mathrm{lo} } +  \frac{2 \digamma_\sigma (\sigma, \pmb{v, \kappa^t,\delta^t})}{\eta_\mathrm{gl} \eta_\mathrm{lo} } + \frac{2 \digamma_\varpi (\sigma, \pmb{v,\kappa^t,\delta^t})}{\eta_\mathrm{gl} \eta_\mathrm{lo} }.
\end{align}

Now, using the fact that $\eta_\mathrm{lo} \leq \frac{1}{8\beta\kappa}$ and $\eta_\mathrm{gl}\eta_\mathrm{lo} \leq \frac{1}{\beta \eta_\mathrm{ef}^t}$, we have 
\begin{align}
    &6 \beta \kappa \eta_\mathrm{gl} \eta_\mathrm{lo} \eta_\mathrm{ef}^t \sum_{u=0}^{U-1} v_u^t (1+q-v_u^t) \left(\alpha_u \delta_u^t\right)^2 \leq 6 \kappa \sum_{u=0}^{U-1} v_u^t (1+q-v_u^t) \left(\alpha_u \delta_u^t\right)^2 \nonumber\\
    &144 \eta_\mathrm{gl} \eta_\mathrm{ef}^t \kappa^3 \beta^3 \eta_\mathrm{lo}^3 \sum_{u=0}^{U-1} v_u^t(1+q-v_u^t) \left(\alpha_u \delta_u^t\right)^2 \leq \frac{9 \kappa}{4} \sum_{u=0}^{U-1} v_u^t(1+q-v_u^t) \left(\alpha_u \delta_u^t\right)^2 
\end{align}

Plugging these values, we get
\begin{align}
    & \underbrace{\Big[1 - 4 \sum_{u=0}^{U-1} \alpha_u ( \delta_u^t (1 - v_u^t) )^2 - 48 \kappa^2 \beta^2 \eta_\mathrm{lo}^2 \sum_{u=0}^{U-1} \alpha_u ( v_u^t \delta_u^t)^2 - \frac{33\kappa}{4} \sum_{u=0}^{U-1} v_u^t (1+q-v_u^t) \left(\alpha_u \delta_u^t\right)^2 \Big]}_{\mathrm{T_6}} \times \eta_\mathrm{ef}^t \left\Vert \nabla f^{t}(\mathbf{w}^t) \right\Vert^2 \nonumber \\
    &\leq \frac{2 \left(f^t(\mathbf{w}^t) - \mathbb{E}_{\pmb{\mathrm{1},\zeta},Q} \left[ f^{t+1}(\mathbf{w}^{t+1}) \right] \right)} {\eta_\mathrm{gl} \eta_\mathrm{lo} } + \frac{2 \Phi^t}{\eta_\mathrm{gl} \eta_\mathrm{lo} } + \Omega_\sigma (\sigma, \pmb{v,\kappa^t,\delta^t}) + \Omega_\varpi (\sigma, \pmb{v,\kappa^t,\delta^t}),
\end{align}
where $\Omega_\sigma (\sigma, \pmb{v,\kappa^t,\delta^t}) \coloneqq \frac{2\digamma_\sigma (\sigma, \pmb{v, \kappa^t,\delta^t})}{\eta_\mathrm{gl}\eta_\mathrm{lo}} = 2(2+q) \beta \eta_\mathrm{gl} \eta_\mathrm{lo} \left(\sigma \eta_\mathrm{ef}^t \right)^2 \sum_{u=0}^{U-1} v_u^t \alpha_u^2 \frac{ (\delta_u^t)^2}{\kappa_u^t} + 8 \eta_\mathrm{ef}^t \beta^2 \sigma^2 \eta_\mathrm{lo}^2 \sum_{u=0}^{U-1} \alpha_u \kappa_u^t (v_u^t\delta_u^t)^2 + 24 \sigma^2 \left(\eta_\mathrm{ef}^t \right)^2 \beta^3 \eta_\mathrm{lo}^3  \sum_{u=0}^{U-1} v_u^t(1+q-v_u^t) \left(\alpha_u \delta_u^t \kappa_u^t \right)^2$, and $\Omega_\varpi (\sigma, \pmb{v,\kappa^t,\delta^t}) \coloneqq  \frac{2\digamma_\varpi (\sigma, \pmb{v,\kappa^t,\delta^t})}{\eta_\mathrm{gl}\eta_\mathrm{lo}} = 4 \eta_\mathrm{ef}^t \varpi^2 \sum_{u=0}^{U-1} \alpha_u ( \delta_u^t (1 - v_u^t) )^2 + 48 \eta_\mathrm{ef}^t \left(\beta \varpi \right)^2 \eta_\mathrm{lo}^2  \sum_{u=0}^{U-1} \alpha_u (v_u^t\delta_u^t)^2 \left(\kappa_u^t\right)^2 + 6 \beta \eta_\mathrm{gl} \eta_\mathrm{lo} \left(\eta_\mathrm{ef}^t \varpi\right)^2 \sum_{u=0}^{U-1} v_u^t(1+q-v_u^t) \left(\alpha_u \delta_u^t\right)^2 \kappa_u^t + 144 \eta_\mathrm{gl} \left(\varpi \eta_\mathrm{ef}^t \right)^2 \beta^3 \eta_\mathrm{lo}^3 \sum_{u=0}^{U-1} v_u^t(1+q-v_u^t) \left(\alpha_u \delta_u^t\right)^2 \left(\kappa_u^t \right)^3$.

Now, we find a sufficient condition to show that $0 < \mathrm{T}_6 \leq 1$.
\begin{align}
\label{eq:nec_cond_on_eta_lo}
    &1 - 4 \sum_{u=0}^{U-1} \alpha_u ( \delta_u^t (1 - v_u^t) )^2 - 48 \kappa^2 \beta^2 \eta_\mathrm{lo}^2 \sum_{u=0}^{U-1} \alpha_u ( v_u^t \delta_u^t)^2 - \frac{33\kappa}{4} \sum_{u=0}^{U-1} v_u^t (1+q-v_u^t) \left(\alpha_u \delta_u^t\right)^2 >0 \nonumber\\
    &48 \kappa^2 \beta^2 \eta_\mathrm{lo}^2 \sum_{u=0}^{U-1} \alpha_u ( v_u^t \delta_u^t)^2 < 1 - 4 \sum_{u=0}^{U-1} \alpha_u ( \delta_u^t (1 - v_u^t) )^2 - \frac{33\kappa}{4} \sum_{u=0}^{U-1} v_u^t (1+q-v_u^t) \left(\alpha_u \delta_u^t\right)^2 \nonumber\\
    &\eta_\mathrm{lo} < \sqrt{\frac{4 - 16 \sum_{u=0}^{U-1} \alpha_u ( \delta_u^t (1 - v_u^t) )^2 - 33 \kappa \sum_{u=0}^{U-1} v_u^t (1+q-v_u^t) \left(\alpha_u \delta_u^t\right)^2}{192 \kappa^2 \beta^2 \sum_{u=0}^{U-1} \alpha_u ( v_u^t \delta_u^t)^2}}
\end{align}

We note that the numerator of (\ref{eq:nec_cond_on_eta_lo}) must be non-negative, which requires $4 - 16 \sum_{u=0}^{U-1} \alpha_u ( \delta_u^t (1 - v_u^t) )^2 - 33 \kappa \sum_{u=0}^{U-1} v_u^t (1+q-v_u^t) \left(\alpha_u \delta_u^t\right)^2 \geq 0$. 
Thus, if $4 - 16 \sum_{u=0}^{U-1} \alpha_u ( \delta_u^t (1 - v_u^t) )^2 - 33 \kappa \sum_{u=0}^{U-1} v_u^t (1+q-v_u^t) \left(\alpha_u \delta_u^t\right)^2 \geq 0$, $\eta_\mathrm{gl}\eta_\mathrm{lo} < \frac{1}{\beta \eta_\mathrm{ef}^t}$, and $\eta_\mathrm{lo} < \text{min} \left\{ \frac{1}{8\beta\kappa},  \sqrt{\frac{4 - 16 \sum_{u=0}^{U-1} \alpha_u ( \delta_u^t (1 - v_u^t) )^2 - 33 \kappa \sum_{u=0}^{U-1} v_u^t (1+q-v_u^t) \left(\alpha_u \delta_u^t\right)^2}{192 \kappa^2 \beta^2 \sum_{u=0}^{U-1} \alpha_u ( v_u^t \delta_u^t)^2}} \right\}$, we have
\begin{align}
\label{eq:globalGradient_Main}
    &\eta_\mathrm{ef}^t \left\Vert \nabla f^t \left(\mathbf{w}^{t} \right) \right\Vert^2  
    \leq \frac{2 \big[f^t(\mathbf{w}^t) - \mathbb{E}_{\pmb{\mathrm{1},\zeta},Q} \left[ f^{t+1}(\mathbf{w}^{t+1}) \right]\big]}{\eta_\mathrm{gl} \eta_\mathrm{lo}} + \frac{2\Phi^t}{\eta_\mathrm{gl} \eta_\mathrm{lo}} + \Omega_\sigma (\sigma, \pmb{v,\kappa^t,\delta^t}) + \Omega_\varpi (\sigma, \pmb{v,\kappa^t,\delta^t}).  
\end{align}

Now, summing over $t=0,1,\dots,T-1$ and dividing both sides by $\frac{1}{\sum_{t=0}^{T-1} \eta_\mathrm{ef}^t}$, we have
\begin{align}
\label{eq:conv_bound}
    \frac{1}{\sum_{t=0}^{T-1} \eta_{\mathrm{ef}}^t} \sum_{t=0}^{T-1} \eta_{\mathrm{ef}}^t \mathbb{E}_{\pmb{1,\zeta},Q} \left[\left\Vert \nabla f^t (\mathbf{w}^t) \right\Vert^2 \right] 
    &\leq \frac{1}{\sum_{t=0}^{T-1} \eta_{\mathrm{ef}}^t} \sum_{t=0}^{T-1} \frac{2\left(\mathbb{E}_{\pmb{\mathrm{1},\zeta},Q} \left[ f^t(\mathbf{w}^t) \right] - \mathbb{E}_{\pmb{\mathrm{1},\zeta},Q} \left[ f^{t+1}(\mathbf{w}^{t+1}) \right] \right)}{\eta_\mathrm{gl} \eta_\mathrm{lo}} + \frac{2}{\eta_\mathrm{gl} \eta_\mathrm{lo} \sum_{t=0}^{T-1} \eta_{\mathrm{ef}}^t} \sum_{t=0}^{T-1} \Phi^t + \nonumber\\
    &\squad \frac{1}{\sum_{t=0}^{T-1} \eta_{\mathrm{ef}}^t} \sum_{t=0}^{T-1} \Omega_\sigma (\pmb{\sigma, v, \kappa^t, \delta^t}) + \frac{1}{\sum_{t=0}^{T-1} \eta_{\mathrm{ef}}^t} \sum_{t=0}^{T-1} \Omega_\varpi (\pmb{\varpi, v, \kappa^t, \delta^t}) \nonumber\\
    &= \frac{2\left(f^0(\mathbf{w}^0) - \mathbb{E}_{\pmb{\mathrm{1},\zeta},Q} \left[ f^{T}(\mathbf{w}^{T}) \right] \right)}{\eta_\mathrm{gl} \eta_\mathrm{lo} \sum_{t=0}^{T-1} \eta_{\mathrm{ef}}^t} + \frac{2}{\eta_\mathrm{gl} \eta_\mathrm{lo} \sum_{t=0}^{T-1} \eta_{\mathrm{ef}}^t} \sum_{t=0}^{T-1} \Phi^t + \nonumber\\
    &\squad \frac{1}{\sum_{t=0}^{T-1} \eta_{\mathrm{ef}}^t} \sum_{t=0}^{T-1} \Omega_\sigma (\pmb{\sigma, v, \kappa^t, \delta^t}) + \frac{1}{\sum_{t=0}^{T-1} \eta_{\mathrm{ef}}^t} \sum_{t=0}^{T-1} \Omega_\varpi (\pmb{\varpi, v, \kappa^t, \delta^t}).
\end{align}
\end{proof}

\newpage \clearpage

\section{Additional Details on Simulation Parameters}

\subsection{OSAFL and Modified Algorithms and Hyperparameters}

\subsubsection{Modified-FedAvg}
We summarized the \ac{mfedavg} baseline in Algorithm \ref{mfedAvg_Algorithm}.
From our ablation study, we find that the learning rates of $0.008$, $0.02$, and $0.01$ respectively, for the CIFAR10, FashionMNIST, and MNIST datasets with the \sqz~ model, $0.02$, $0.03$, and $0.4$, respectively, with the \ac{cnn} model; and $0.01$, $0.02$, and $0.04$, respectively, with the \ac{resnet18} model, works the best in our setting for this \ac{mfedavg} algorithm.

\begin{algorithm}[!h]
\small
\SetAlgoLined 
\DontPrintSemicolon
\KwIn{Initial global model $\mathbf{w}^0$, client set $\mathcal{U}$, total global round $T$, local learning rate $\eta$ }
\For{$t=0$ to $T-1$}{
    \For{$u$ in $\mathcal{U}$ in parallel}{
        Receives the latest global model from the CS \;
        Synchronize the local model: $\mathbf{w}_{u}^{t,0} \gets \mathbf{w}^{t}$ \;
        Client solves resource optimization problem  (\ref{localIterOptim_Orig}) using Algorithm \ref{iterativeAlgorithm_LRO} to determine $\mathrm{1}_u^t$ and local round $\kappa_u^t$ \;
        \uIf{$\mathrm{1}_u^t = 1$}{
            Performs $\kappa_u^t$ SGD steps: $\mathbf{w}_u^{t,\kappa_u^t} = \mathbf{w}_u^{t,0} - \eta \sum\nolimits_{\tau=0}^{\kappa_u^t-1} g_u \left(\mathbf{w}_u^{t,\tau} | \mathcal{D}_u^t \right)$ \;
            Calculates the model differences $\Delta \mathbf{w}_u^t \triangleq \mathbf{w}_u^{t,\kappa_u^t} - \mathbf{w}_u^{t,0} $ \;
            Use stochastic quantizer defined in (\ref{eq:stoQuantizer}) to quantize $\Delta \mathbf{w}_u^t$ and transmit $Q(\Delta \mathbf{w}_u^t)$ to the central server\;
            }
    }
    Perform global aggregation: $\mathbf{w}^{t+1} = \sum\nolimits_{u=0}^{U-1} \alpha_u \cdot \mathrm{1}_u^t \cdot \left(\mathbf{w}^t + Q\left( \Delta \mathbf{w}_u^t \right) \right)$ \;
    }
\KwOut{Trained global model $\mathbf{w}^T$}
\caption{Modified Federated Averaging Baselines}
\label{mfedAvg_Algorithm}
\end{algorithm}
\begin{algorithm}[!h]
\small
\SetAlgoLined 
\DontPrintSemicolon
\KwIn{Initial global model $\mathbf{w}^0$, client set $\mathcal{U}$, total global round $T$, local learning rate $\eta$, proximal penalty parameter $\mu$}
\For{$t=0$ to $T-1$}{
    \For{$u$ in $\mathcal{U}$ in parallel}{
        Receives the latest global model from the CS \;
        Synchronize the local model: $\mathbf{w}_{u}^{t,0} \gets \mathbf{w}^{t}$ \;
        Client solves resource optimization problem  (\ref{localIterOptim_Orig}) using Algorithm \ref{iterativeAlgorithm_LRO} to determine $\mathrm{1}_u^t$ and local round $\kappa_u^t$ \; 
        
        \uIf{$\mathrm{1}_u^t = 1$}{
            Performs $\kappa_u^t$ SGD steps to get $\mathbf{w}_u^{t,\kappa_u^t}$ that minimizes $f_u \left(\mathbf{w}_u^{t, \tau} | \mathcal{D}_u^t \right) + \frac{\mu}{2} \left\Vert \mathbf{w}_u^{t,\tau} - \mathbf{w}^t \right\Vert^2 $ \;
            Calculates the model differences $\Delta \mathbf{w}_u^t \triangleq \mathbf{w}_u^{t,\kappa_u^t} - \mathbf{w}_u^{t,0} $ \;
            Use stochastic quantizer defined in (\ref{eq:stoQuantizer}) to quantize $\Delta \mathbf{w}_u^t$ and transmit $Q(\Delta \mathbf{w}_u^t)$ to the central server\;
            }
    }
    Perform global aggregation: $\mathbf{w}^{t+1} = \sum\nolimits_{u=0}^{U-1} \alpha_u \cdot \mathrm{1}_u^t \cdot \left(\mathbf{w}^t + Q\left( \Delta \mathbf{w}_u^t \right) \right)$ \;
    }
\KwOut{Trained global model $\mathbf{w}^T$}
\caption{Modified-FedProx Baselines}
\label{mfedprox_Algorithm}
\end{algorithm}
\begin{algorithm}[!t]
\small
\SetAlgoLined 
\DontPrintSemicolon
\KwIn{Initial global model $\mathbf{w}^0$, client set $\mathcal{U}$, total global round $T$, learning rate $\eta$}
\For{$t=0$ to $T-1$}{
    \For{$u$ in $\mathcal{U}$ in parallel}{
        Receives the latest global model from the CS \;
        Synchronize the local model: $\mathbf{w}_{u}^{t,0} \gets \mathbf{w}^{t}$ \;
        Client solves resource optimization problem  (\ref{localIterOptim_Orig}) using Algorithm \ref{iterativeAlgorithm_LRO} to determine $\mathrm{1}_u^t$ and local round $\kappa_u^t$ \; 
        \uIf{$\mathrm{1}_u^t = 1$}{
        Performs $\kappa_u^t$ SGD steps to get $\mathbf{w}_u^{t,\kappa_u^t} = \mathbf{w}_u^{t,0} - \eta \sum_{\tau=0}^{\kappa_u^t-1} g_u \left( \mathbf{w}_u^{t,\tau} | \mathcal{D}_u^t \right)$ \;
        Client calculates normalized gradient $\mathbf{d}_u^t = \frac{\mathbf{w}_u^{t, 0} - \mathbf{w}_u^{t, \kappa_u^t}}{\kappa_u^t \cdot \eta}$ \; 
        Use stochastic quantizer defined in (\ref{eq:stoQuantizer}) to quantize $\mathbf{d}_u^t$ and transmit $Q(\mathbf{d}_u^t)$ and $\kappa_u^t$} to the central server \;
        }
    Perform global aggregation: $\mathbf{w}^{t+1} = \mathbf{w}^t - \left[\sum_{i=0}^{U-1} \left(\alpha_i \cdot \mathrm{1}_i^t \cdot \kappa_i^t \right)\right] \sum_{u=0}^{U-1} \eta \left(\frac{\alpha_u \cdot \mathrm{1}_u^t \cdot \kappa_u^t}{\sum_{u'=0}^{U-1}  \alpha_{u'} \cdot \mathrm{1}_{u'}^t} \right) \cdot Q\left( \mathbf{d}_u^t \right)$ \;
    }
\KwOut{Trained global model $\mathbf{w}^T$}
\caption{Modified-FedNova Baselines}
\label{mfednova_Algorithm}
\end{algorithm}
\begin{algorithm}[!t]
\small
\SetAlgoLined 
\DontPrintSemicolon
\KwIn{Initial global model $\mathbf{w}^0$, client set $\mathcal{U}$, total global round $T$, local learning rate $\eta$, hyperparameter $a$ and $b$ }
\For{$t=0$ to $T-1$}{
    \For{$u$ in $\mathcal{U}$ in parallel}{
        Receives the latest global model from the CS \;
        Synchronize the local model: $\mathbf{w}_{u}^{t,0} \gets \mathbf{w}^{t}$ \;
        Client solves resource optimization problem  (\ref{localIterOptim_Orig}) using Algorithm \ref{iterativeAlgorithm_LRO} to determine $\mathrm{1}_u^t$ and local round $\kappa_u^t$ \; 
        
        \uIf{$\mathrm{1}_u^t = 1$}{
            Client performs $\kappa_u^t$ SGD steps: $\mathbf{w}_u^{t,\kappa_u^t} = \mathbf{w}_u^{t,0} - \eta \sum\nolimits_{\tau=0}^{\kappa_u^t-1} g_u \left(\mathbf{w}_u^{t,\tau} | \mathcal{D}_u^t \right)$ \;
            Client calculates model differences $\Delta \mathbf{w}_u^t \triangleq \mathbf{w}_u^{t,\kappa_u^t} - \mathbf{w}_u^{t,0}$ \;
            Client uses stochastic quantizer defined in (\ref{eq:stoQuantizer}) and sends $Q(\Delta \mathbf{w}_u^t)$ and discrepancy value $d_u^t$ \cite[Sec. $5.1$]{ye23fedDisco} to the \ac{cs} \; 
            }
    }
    \Ac{cs} computes aggregation weights following 
    \begin{align}
        \bar{\alpha}_u &= \frac{\mathrm{ReLU}\left(\alpha_u - a \cdot d_u^t + b \right)}{\sum_{u'=0}^{U-1} \mathrm{ReLU}\left(\alpha_{u'} - a \cdot d_{u'}^t + b \right)}, \\
        \tilde{\alpha}_u &= \frac{\bar{\alpha}_u^t \cdot \mathrm{1}_u^t}{\sum_{u=0}^{U-1} \bar{\alpha}_u^t \cdot \mathrm{1}_u^t}.
    \end{align} 
    CS perform global aggregation: $\mathbf{w}^{t+1} = \sum\nolimits_{u=0}^{U-1} \tilde{\alpha}_u \cdot \left[\mathbf{w}^t + Q\left( \Delta \mathbf{w}_u^t \right) \right] $
    }
\KwOut{Trained global model $\mathbf{w}^T$}
\caption{Modified FedDisco  Baselines}
\label{mfeddisco_Algorithm}
\end{algorithm}

\bigskip

\subsubsection{Modified-FedProx Baseline}
The \ac{mfedprox} baseline is summarized in Algorithm \ref{mfedprox_Algorithm}.
From our ablation study, we find that the learning rates of $0.01$, $0.02$, and $0.01$ respectively, for the CIFAR10, FashionMNIST, and MNIST datasets with the \sqz~ model, $0.03$, $0.04$, and $0.4$, respectively, with the \ac{cnn} model; and $0.01$, $0.02$, and $0.05$, respectively, with the \ac{resnet18} model, works the best in our setting for this \ac{mfedprox} algorithm.
Besides, we find $\mu=0.01$ works best in our implementation.

\bigskip
\subsubsection{Modified-FedNova}
Algorithm \ref{mfednova_Algorithm} summarizes the \ac{mfednova} baseline.
Based on our ablation study, use the learning rates of $0.01$, $0.02$, and $0.01$ respectively, for the CIFAR10, FashionMNIST, and MNIST datasets with the \sqz~ model, $0.03$, $0.04$, and $0.04$, respectively, with the \ac{cnn} model; and $0.008$, $0.02$, and $0.08$, respectively, with the \ac{resnet18} model, works the best in our setting for this \ac{mfednova} algorithm.

\bigskip
\subsubsection{Modified-FedDisco Baseline}
Algorithm \ref{mfeddisco_Algorithm} summarizes our \ac{mfeddisco} baseline. 
For this one, we find that the learning rates of $0.01$, $0.015$, and $0.01$ respectively, for the CIFAR10, FashionMNIST, and MNIST datasets with the \sqz~ model, $0.02$, $0.03$, and $0.06$, respectively, with the \ac{cnn} model; and $0.008$, $0.02$, and $0.05$, respectively, with the \ac{resnet18} model, works the best in our setting for this \ac{mfeddisco} algorithm.
Besides, we use $a=0.15$ and $0.1$ in Algorithm \ref{mfeddisco_Algorithm}.

\end{document}